\documentclass{article}          % twocolumn
\usepackage{graphicx,url}
\usepackage[margin = 1in]{geometry}
\usepackage[usenames,dvipsnames]{color}
\usepackage[capposition=top]{floatrow}
\usepackage[ruled]{algorithm2e}
\usepackage{multicol,multirow,graphicx,xcolor,varwidth,float,enumerate,amssymb,amsmath}
\usepackage{subfigure}
\usepackage[authoryear, round]{natbib}

\usepackage[ruled]{algorithm2e}

\usepackage{algorithmic}

\def \Ltheta{L(\pmb{\theta})}
\def \LNtheta{L_{\mathrm{N}}(\pmb{\theta})}
\def \btheta{\pmb{\theta}}

\def\C{\mathcal{C}}
\def\X{\mathcal{X}}
\def\P{\mathcal{P}}
\def\R{\mathbb{R}}

\def\Th{\Theta}
\def\Ln{L_{\mathrm{N}}}

\usepackage{color}

\usepackage{multirow}

\newtheorem{thm}{Theorem}
\newtheorem{lem}[thm]{Lemma}

\newtheorem{defin}[thm]{Definition}

\definecolor{grey}{rgb}{.5, .5, .5}
\definecolor{lblue}{rgb}{0, 1, 1}
\definecolor{breastcol}{HTML}{9AD0F3}
\definecolor{yeastcol}{HTML}{000000}
\definecolor{votescol}{HTML}{E79F00}
\definecolor{dermacol}{HTML}{0072B2}
\definecolor{synthcol}{HTML}{D55E00}

\begin{document}

\title{Minimum Spectral Connectivity Projection Pursuit%\thanks{Grants or other notes
}


\author{David P. Hofmeyr\\Dept. of Statistics and Actuarial Science \\
		Stellenbosch University, South Africa         \and
        Nicos G. Pavlidis\\ Dept. of Management Science\\
        Lancaster University, UK			\and
	Idris A. Eckley\\ Dept. of Mathematics and Statistics\\ Lancaster University, UK
}


\maketitle

\begin{abstract}
We study the problem of determining the optimal low dimensional projection for
maximising the separability of a binary partition of an unlabelled dataset, as
measured by spectral graph theory.  This is achieved by finding projections
which minimise the second eigenvalue of the graph Laplacian of the projected
data, which corresponds to a non-convex, non-smooth optimisation problem. We
show that the optimal univariate projection based on spectral connectivity
converges to the vector normal to the maximum margin hyperplane through the
data, as the scaling parameter is reduced to zero. This establishes a
connection between connectivity as measured by spectral graph theory and
maximal Euclidean separation. 
The computational cost
associated with each eigen-problem is quadratic in the number of data. To
mitigate this issue, we propose an approximation method using microclusters
with provable approximation error bounds.
Combining multiple binary partitions
within a divisive hierarchical model allows us to construct clustering solutions admitting
clusters with varying scales and lying within different subspaces. We evaluate the performance of the
proposed method on a large collection of benchmark datasets and find that it
compares favourably with existing methods for projection pursuit and dimension
reduction for data clustering.
\end{abstract}
\noindent
{\bf keywords:} Spectral clustering \and dimension reduction \and projection pursuit \and maximum margin

\section{Introduction}

Identifying distinct groups, or {\em clusters}\/, in unlabelled data is a
fundamental task in exploratory data analysis, with applications in diverse
disciplines ranging from computer science and biology to sociology and
marketing.
Spectral clustering methods have gained considerable attention because of their simplicity, versatility
and strong performance in numerous
applications~\citep{Shi2000,Weiss1999,Ning2010,Chi2009}.
One of the appealing properties of spectral clustering is its ability to
identify highly non-convex clusters, which may lie on or close to highly non-linear
manifolds. It is, however, sensitive to choices of scaling and to irrelevant or noisy features
which may be present in the data~\citep{BachJ2006,niu2011dimensionality}.

In spectral clustering, clusters are defined as strongly connected components
of a graph whose vertices correspond to data points, and edge weights represent
pairwise similarities between them~\citep{Luxburg2007}.
The minimum-cut problem seeks the partition of the graph that minimises the sum
of edge weights connecting different components of the partition. In other
words, the partition which minimises the total similarity between data assigned
to different clusters.
Although intuitive this formulation frequently produces partitions in which
some components contain very few vertices (data), which may not constitute complete clusters.
To avoid this, normalisations of the minimum-cut problem that favour balanced
partitions are used.
Normalisation, however, renders the problem NP-hard~\citep{Wagner1993}, and so a
continuous relaxation is solved instead. 
The solution of the relaxed problem is given by the eigenvectors of the
\emph{graph Laplacian} matrix.
This spectral decomposition of the graph Laplacian gives rise to the term
spectral clustering.\\

The successful application of any clustering method critically depends on the
extent to which the true group structure in the data is captured by spatial
similarities between points.
However, the presence of irrelevant and noisy features, which abound in modern
applications, can distort this spatial structure. This has been
shown to have particularly adverse effects on the performance of spectral
clustering, even in problems of moderate dimensionality~\citep{BachJ2006,niu2011dimensionality}.
Dimension reduction techniques
attempt to mitigate the effects of noisy and irrelevant features
by identifying low dimensional representations of a
dataset that preserve the maximum amount of relevant
information. Commonly these low dimensional representations
are defined by the projection of the data into a linear
subspace.
Classical techniques, like principal component analysis (PCA), although widely
used in clustering, are not guaranteed to identify subspaces that preserve
cluster structure.
More recently a number of dimension reduction methods that explicitly aim to
reveal cluster structure have been
developed~\citep{krause2005multimodal,Pavlidis2015,HofmeyrP2015,PenaF2001,niu2011dimensionality}.

\citet{PenaF2001} show that under certain conditions the one-dimensional projection of the data with minimum kurtosis
maximises bimodality. Such a projection can thus be used
to separate {\em high-density clusters}\/,
defined as contiguous regions of high probability density around modes
of the (assumed) underlying probability density function. 
For the same purpose,~\citet{krause2005multimodal} propose maximising
the {\em dip statistic}~\citep{Hartigan1985}, a measure of departure from
unimodality of a univariate dataset. 
More recently~\citet{Pavlidis2015} proposed an approach that aims
to identify regions of low probability density that separate
high-density clusters. This is achieved by identifying the univariate subspace normal to the hyperplane
that has the minimum integrated density
along it, called the {\em minimum density hyperplane}\/. \citet{HofmeyrP2015}
proposed a method to identify projections that maximise the variance-ratio
clusterability measure~\citep{Zhang2001}.
This measure is a normalisation of the $K$-means objective,
which is invariant to changes in scale and is thus less susceptible to projections
which exhibit high variance but little cluster structure. 
The problem of dimensionality reduction for spectral clustering was first
considered by~\citet{niu2011dimensionality}. 
A detailed description of this method and its relation to our work is
provided in Section~\ref{sec:background} after the presentation of
necessary background material.\\

The main problem we consider in this paper is the identification of the optimal projection to
bi-partition a dataset through spectral clustering.
This is achieved by minimising the second smallest eigenvalue of the graph Laplacian, which measures the
spectral connectivity between the two clusters. We consider
the graph Laplacians arising from the two most widely used normalisations of the minimum-cut objective, namely Ratio Cut~\citep{HagenK1992} 
and Normalised Cut~\citep{Shi2000}.
Although both formulations can lead to high quality clustering models, our
experience suggests that for our purposes the Normalised Cut formulation yields overall superior
performance.
Applying this bi-partitioning approach recursively produces a divisive spectral
clustering algorithm capable of identifying clusters with varying scales and defined in different
subspaces.
The minimisation of the sum of the $K$~smallest eigenvalues of the normalised
graph Laplacian with respect to a projection of the data was first proposed
by~\citet{niu2011dimensionality} to perform dimension reduction for spectral
clustering. 

In this paper we develop an improved methodology for
finding optimal projections based on the spectral clustering
objective, and provide new theoretical perspectives on
the problem. We perform a rigorous investigation into the
continuity and differentiability properties of eigenvalues of
graph Laplacians as functions of the projection, and find
that they are Lipschitz continuous (and hence differentiable
almost everywhere), and everywhere directionally differentiable.
We derive expressions for the derivative of an eigenvalue
with respect to the projection when the eigenvalue
is simple, thereby allowing us to minimise the objective
directly using generalised gradient descent methods. This
approach is guaranteed to converge to a local minimum,
whereas existing methodology for this problem does not
directly minimise the overall objective and may fail to find
an optimal projection. In addition, we provide a formulation
of the directional derivative which allows us to easily
derive optimality conditions for the proposed method.
Although our focus is on minimising the second smallest
eigenvalue our analysis applies to an arbitrary eigenvalue of
the Laplacian, and so the proposed methodology can easily
be extended to minimising sums of eigenvalues of graph
Laplacians.

Each eigenvalue computation requires~$\mathcal{O}(N^2)$ operations, where~$N$
is the size of the dataset. This can be prohibitive for large datasets.  We
show how preprocessing the dataset using microclusters provides an
approximation of the optimisation surface which enables a speed-up of up to two
orders of magnitude without an appreciable degradation in empirical clustering
accuracy. We also derive theoretical worst case error bounds for this
approximation.

We establish an asymptotic connection between optimal univariate
projections for spectral bi-partitioning and maximum margin hyperplanes.
Formally, we show that as the scaling parameter defining pairwise similarities
is reduced to zero, the optimal one-dimensional
projection for spectral bi-partitioning converges to the vector normal to the
largest margin hyperplane through the data. This establishes a theoretical
connection between connectivity as measured by spectral graph theory and
Euclidean separation, which underlies maximum margin clustering~\citep{xu2004maximum,zhang2009maximum},
an increasingly popular and effective approach to clustering.\\

The remainder of the paper is organised as follows. In~Section~\ref{sec:background}
we provide a brief introduction to spectral clustering, and existing dimension reduction
based on the spectral clustering objective.
Section~\ref{sec:method} presents our methodology for finding optimal
projections based on spectral connectivity. Section~\ref{sec:maxmargin} describes
the theoretical connection between the optimal one-dimensional projection for spectral
bi-partitioning and maximum margin hyperplanes. In
Section~\ref{sec:microclust} we discuss an approximation technique which allows
for a substantial improvement in computation time of the method, and derive theoretical
worst case error bounds. Experimental
results and sensitivity analyses are presented in
Section~\ref{sec:experiments}.

\section{Background} \label{sec:background}

In this section we provide a brief introduction to spectral clustering, with
particular attention to binary partitioning, and discuss existing methodology for
dimension reduction based on the spectral clustering objective.  Let $\X =
\{x_1, \ldots, x_N\}$ denote a dataset in $\R^d$. Then define the graph
$\mathcal{G}=(\mathcal{V},\mathcal{E})$, where vertices correspond to elements in $\X$, and the
\emph{undirected} edges assume weights equal to the pairwise \emph{similarities}
between data. 
The information in $\mathcal{G}$ can be represented by the {\em adjacency}, or {\em affinity} matrix,
$A \in \R^{N\times N}$, with $A_{ij} = \mathcal{E}_{ij}:= \mbox{similarity}(x_i, x_j)$.
The \emph{degree} of the $i$-th vertex is defined as $d_i = \sum_{j=1}^N
A_{ij}$, and the degree matrix is defined as $D = \mathrm{diag}(d_1,\ldots,d_N)$.
%
%The \emph{degree matrix}, $D$, is then
%defined as the diagonal matrix with $i$-th diagonal element equal to $d_i$.
%
For a subset $\C \subset \X$, the size of $\C$ can be measured either by its
cardinality, $\vert \C\vert$, or by its {\em volume},
$\textrm{vol}(\C) := \sum_{i: x_i \in \C} d_i$. 

\begin{defin}
The \emph{normalised minimum-cut of a graph} is the solution to the optimisation problem
\begin{equation}\label{eq:mincut}
\min_{\C \subset \X} \sum_{i, j: x_i \in \C, x_j \in \X\setminus \C} A_{ij} \left(\frac{1}{\mbox{size}(\C)} + \frac{1}{\mbox{size}(\X\setminus \C)}\right).
\end{equation}
\end{defin}

When size$(\C) = \vert \C\vert$ the above objective is referred to as {\em Ratio Cut}~\citep{HagenK1992},
whereas when size$(\C) = \mathrm{vol}(\C)$ it is known as {\em Normalised Cut}~\citep{Shi2000}.
\citet{HagenK1992} and~\citet{Shi2000} have shown that the normalised minimum-cut
problems arising from these two definitions of size can be formulated
in terms of the \emph{graph Laplacian} matrices,
\begin{align}
(standard) \ L &= D - A, \label{eq:standard}\\
(normalised) \ \Ln &= D^{-1/2}LD^{-1/2},\label{eq:normal}
\end{align}
as follows. For $\C \subset \X$ define $u^{\C} \in\mathbb{R}^N$ to be the vector
with $i$-th entry,
\begin{equation}\label{eq:discretef}
u^{\C}_i = \left\{ \begin{array}{rl}
			\sqrt{\mbox{size}\left(\X \setminus \C\right)/\mbox{size}(\C)}, & \textrm{if } x_i \in \C \\
			-\sqrt{\mbox{size}\left(\C \right)/\mbox{size}(\X \setminus \C)}, & \textrm{if } x_i \in \X\setminus \C. \end{array} \right.
\end{equation}
%
%$\sqrt{\mbox{size}(\bar \C)/\mbox{size}(\C)}$ if $x_i \in \C$
%and $-\sqrt{\mbox{size}(\C)/\mbox{size}(\bar \C)}$ otherwise.
%
For size$(\C) =\vert \C \vert$, %Eq.~
the optimisation problem in (\ref{eq:mincut}) can be written as,
\begin{equation}\label{eq:ratiocut}
\min_{\C \subset X} (u^{\C})^\top L u^{\C} \; \textrm{ s.t. } \; u^{\C} \perp \mathbf{1}, \
\|u^{\C}\| = \sqrt{N}.
\end{equation}
If instead size$(\C)=$ vol$(\C)$ then~(\ref{eq:mincut}) is equivalent to,
\begin{eqnarray}\label{eq:Ncut}
\min_{\C \subset X} (u^{\C})^\top L u^{\C} \; \mbox{ s.t. } Du^{\C} \perp \mathbf{1}, \
(u^{\C})^\top D u^{\C}= \mbox{vol}(\X).
\end{eqnarray}

Both problems in~(\ref{eq:ratiocut}) and~(\ref{eq:Ncut}) are
NP-hard~\citep{Wagner1993}. However continuous relaxations, in which the
discreteness condition on $u^{\C}$, Eq.~(\ref{eq:discretef}), is removed,
can be solved in quadratic time~\citep{HagenK1992,Shi2000}. The solutions to the relaxed
problems are given by the second eigenvector of~$L$, and the second eigenvector
of the generalised eigen-equation $Lu = \lambda Du$ respectively. The latter
is thus equivalently solved by $D^{-1/2}u$, where~$u$ is the second eigenvector of
$\Ln$. 
The above approach readily extends to the problem of obtaining a $K$-partition
of the dataset. In this case the
solution is obtained from the eigenvectors corresponding to the $K$ smallest eigenvalues of~$L$ or~$\Ln$~\citep{Luxburg2007}, respectively. \\

Dimension reduction based on the spectral clustering objective using the
normalised graph Laplacian was first considered
by~\citet{niu2011dimensionality}. The objective considered by the authors is
equivalent to the objective we consider, and can be formulated as follows,
\begin{subequations}\label{eq:DRSC}
\begin{eqnarray}
\max_{U, V} &  \mbox{trace}(U^\top D^{-1/2}AD^{-1/2}U) \label{eq:drsc1}\\
s.t. & U^\top U = I\\
& A_{ij} = k(\|V^\top x_i - V^\top x_j \|) \label{eq:Aij}\\
&V^\top V = I.
\end{eqnarray}
\end{subequations}
Note that since $\Ln=I- D^{-1/2}AD^{-1/2}$, the trace maximisation in
(\ref{eq:drsc1}) is equivalent to $\min_{U, V}  \mbox{trace}(U^\top \Ln U)$.
The elements of the affinity matrix, $A$, are determined by a function,
$k(\cdot)$, of the pairwise distances of the points projected into the subspace
defined by the projection matrix~$V$; and $D$ is the corresponding degree matrix. It is clear that
for a given $V$ the matrix $U$ that maximises the trace in (\ref{eq:drsc1}) has
columns given by the~$K$ eigenvectors associated with the $K$~largest
eigenvalues of $D^{-1/2}AD^{-1/2}$ (or equivalently the $K$ smallest
eigenvalues of $\Ln$).
To solve the problem in~(\ref{eq:DRSC}),~\citet{niu2011dimensionality} propose
an algorithm that alternates between two stages: (i) for a fixed $V$ a spectral
decomposition of $\Ln$ determines the optimal~$U$; and (ii) fixing $U$ and $D$
a gradient ascent method is used to maximise $\mbox{trace}(U^\top
D^{-1/2}AD^{-1/2}U)$ with respect to $V$, where the dependence of this
objective on the projection matrix $V$ is through Eq.~(\ref{eq:Aij}). 
%
%Within the gradient ascent stage the matrices~$U$ and~$D$ are kept fixed,
This process is then iterated.
However, this approach does not account for the fact that the degree matrix $D$
is a function of $A$ and therefore it is itself a function of $V$. An ascent
direction for the objective assuming a fixed~$D$ is thus not necessarily an
ascent direction for the overall objective. 
We have further observed that in practice this algorithm is not guaranteed to
lead to an increase in the overall objective across iterations and may thus
fail to converge. In the following section we derive expressions for the
gradient of the overall objective as a function of the projection, allowing us
to optimise it directly.

\section{Projection Pursuit for Spectral Connectivity} \label{sec:method}

In this section we study the problem of minimising the second eigenvalue of the
graph Laplacian of the projected data.
If the projected data are bi-partitioned through spectral clustering, then the
projection
that minimises the second eigenvalue of the graph Laplacian minimises the
connectivity between the two clusters, as measured by spectral graph theory.
%
%Although the discussion in this section focuses on the minimisation of the
%second eigenvalue, the methodology presented applies to an arbitrary eigenvalue
%of a graph Laplacian.
%
%The proposed approach can
%therefore be extended to the problem of determining a $K$-partition
%by minimising the sum of the $K$ smallest eigenvalues of a graph Laplacian.

Let $\X = \{x_1, \ldots x_N\}$ be
a dataset in $\R^d$.
We define the {\em projection matrix} $V$ as a $d\times l$ matrix, with $l < d$,
whose columns $\{v_1, \ldots, v_l\}$, have unit norm.
With this formulation it is convenient to express~$V$ in polar coordinates.
Let $\Th = [0, \pi)^{(d-1)\times l}$, then 
for $\btheta \in \Th$,
the projection matrix $V(\btheta)$ is given by,
\begin{equation} \label{angles}
V(\pmb{\theta})_{ij} = \left\{ \begin{array}{ll}\cos(\btheta_{ij}) \prod_{k = 1}^{i-1}\sin(\btheta_{kj}), & i = 1, ..., d-1\\
\prod_{k=1}^{d-1} \sin(\btheta_{kj}), & i = d. \end{array}\right.
\end{equation}
The $l$-dimensional {\em projected data set} is denoted by $\mathcal{P}(\btheta)
= \{p(\btheta)_1, \ldots, p(\btheta)_N\} = \{V(\btheta)^\top x_1, \ldots,V(\btheta)^\top x_N\}$.
We also define the data matrix, $X \in \R^{d \times N}$, and the projected
data matrix $P \in \R^{l \times N}$, as matrices whose columns contain the original and projected
data, respectively.

We define $\Ltheta$ (resp. $\LNtheta$) as the Laplacian (resp. normalised
Laplacian) of the graph constructed from the projected data set
$\mathcal{P}(\btheta)$. Throughout we use $\lambda_i(\cdot)$ to denote the
$i$-th smallest eigenvalue of its real symmetric matrix argument, and we assume
that all eigenvectors are normalised. Edge weights in the graph of
$\P(\btheta)$ are determined by a Lipschitz continuous and continuously differentiable {\em similarity function} $s:\R^{l\times N} \times
\{1\dots N\}^2 \to \R^+$, in that the affinity matrix is given by,
\begin{equation}\label{eq:sim}
A(\btheta)_{ij}:= s(P(\btheta), i, j) = k\left(d(p(\btheta)_i, p(\btheta)_j)/\sigma\right),
\end{equation}
where $k:\R^+\to\R^+$ is a smooth decreasing function,
$d(\cdot, \cdot)$ is a metric and $\sigma >0$ is the {\em scaling parameter}.
It is common to use the Euclidean metric, however our
experience has shown that projection pursuit for spectral clustering can be
sensitive to outliers when this metric is used. This is especially the case
when using the standard Laplacian. To mitigate against this we define a metric which encourages cluster boundaries to intersect a chosen convex set, $\pmb{\Delta}(\btheta)$, which depends on the projection $\btheta$. This is achieved by defining $d(\cdot, \cdot)$ so that the resulting similarities between points outside $\pmb{\Delta}(\btheta)$, which may be outliers, and other points, are increased. A detailed discussion is provided in 
Appendix~\ref{sec:sim}.\\

A common requirement in linear dimension reduction methods is that the
projection matrix~$V$ is orthonormal, that is $V^\top V = I$. 
\citet{niu2011dimensionality} directly enforce this constraint by
generating the columns of~$V$ sequentially and optimising each column over the
null space of previously determined columns. 
By restricting the domain of the optimisation problem to the manifold of $d
\times l$ orthonormal matrices, known as the Stiefel manifold, it is possible
to optimise over the entire matrix $V$~\citep{Edelman1998,Boumal2014}. 
However, optimisation algorithms operating over the Stiefel manifold have
only been shown to have guaranteed convergence when the objective function is everywhere continuously
differentiable. As we discuss in the next section this requirement is not
necessarily met by the eigenvalues of graph Laplacians.
We instead introduce a penalty term into the objective function which leads to
approximately orthogonal projection matrices. Specifically, we consider the
objective,
\begin{equation}\label{eq:orthog}
\min_{\btheta \in \Th} \lambda_2(\Ltheta) + \omega \sum_{i \neq j} \left(V(\btheta)_i^\top V(\btheta)_j \right)^2,
\end{equation}
or replacing $\lambda_2(\Ltheta)$ with $\lambda_2(\LNtheta)$ in the normalised
case. As in the case of optimising over the Stiefel manifold, this formulation
enables us to update the entire matrix~$V$ at each iteration.
This is an important advantage because the expensive computation of the
eigenvalue of the graph Laplacian is performed once rather than~$l$ times for
each complete update of~$V$.

\subsection{Continuity and Differentiability}\label{sec:differentiability}

In this subsection we investigate the continuity and differentiability properties
of~$\lambda_2(\Ltheta)$ and~$\lambda_2(\LNtheta)$, which are required to
establish global convergence of the optimisation algorithm discussed in
Section~\ref{sec:optimisation}.

To begin with, simple applications of the inequalities of~\citet{weyl1912}
and~\citet{Schur1911} give us,
\begin{equation*}
\vert \lambda_i(L(\pmb{\theta})) - \lambda_i(L(\pmb{\theta}^\prime))\vert  
 \leq N\sqrt{\max_{ij}\vert L(\pmb{\theta})-L(\pmb{\theta}^\prime)\vert_{ij}}.
\end{equation*} 
By assumption
the similarity function, $s$, is Lipschitz continuous
in $P \in \R^{l\times N}$ for fixed $i, j$. The elements of $\Ltheta$
are therefore Lipschitz continuous as compositions of Lipschitz functions ($V(\pmb{\theta})$ is Lipschitz in
$\pmb{\theta}$ as a collection of finite products of Lipschitz functions). 
Thus the objective $\lambda_2(\Ltheta)$ is Lipschitz continuous in~$\btheta$. An analogous
argument can be used to show that $\lambda_2(\LNtheta)$ is Lipschitz continuous.
\noindent
Rademacher's theorem therefore
tells us that $\lambda_2(\Ltheta)$ and $\lambda_2(\LNtheta)$
are almost everywhere differentiable~\citep{Polak1987}.
Generalised gradient descent methods therefore provide a natural framework
for finding locally optimal projections for spectral bi-partitioning~\citep{Polak1987}. 

Eigenvalue optimisation is made challenging by the fact that eigenvalues are only
guaranteed to be differentiable when they are {\em simple}, i.e., are not repeated.
However, minimising the smallest eigenvalue tends to
separate it from other eigenvalues, and therefore
the issue of non-differentiability becomes less of a concern~\citep{lewis1996eigenvalue}.
%\red{In contrast minimising the sum of the $K$ smallest eigenvalues, 
%tends to make these eigenvalues equal exacerbating the problem of non-differentiability.}
%
A basic property of graph Laplacian
matrices is that  both $\lambda_1(L)$ and $\lambda_1(L_{\mathrm{N}})$ are always equal to zero~\citep{Luxburg2007}.
If the similarity function, $s$, is
strictly positive, then $\lambda_2(\Ltheta)$ and $\lambda_2(\LNtheta)$ are
bounded away from zero.
%
%and hence $\lambda_2(\cdot) > \lambda_1(\cdot)$ for all $\pmb{\theta}$
%
Therefore minimising $\lambda_2(\cdot)$ 
%
%enjoys the same advantages as minimising the smallest eigenvalue in the
%general case, in that the optimisation procedure 
%
tends to separate it from other eigenvalues, guiding the search to regions of
the domain where the objective function is differentiable.
Nonetheless, we cannot guarantee
that $\lambda_2(\Ltheta)$ and $\lambda_2(\LNtheta)$ are simple
throughout the optimisation procedure.
%
%In Section~\ref{sec:optimisation} we discuss how to identify descent directions
%when the eigenvalue objectives $\lambda_2(\Ltheta)$ and $\lambda_2(\LNtheta)$ are not simple.
%
We next provide expressions for the derivatives of $\lambda_2(\Ltheta)$ and $\lambda_2(\LNtheta)$
as functions of $\btheta$, when they are simple. 
Using these we then establish that these eigenvalue objectives are in fact
{\em continuously} differentiable when they are simple.

A useful formulation of eigenvalue derivatives is found in~\citep[Th.~1]{Magnus1985};
if $\lambda$ is a simple eigenvalue of a real symmetric matrix $M$, 
then
$\lambda$ is infinitely differentiable on a neighbourhood of $M$, and the
differential at $M$ is given by,
\begin{equation}\label{eigdev}
d \lambda = u^\top d(M) u,
\end{equation}
where $u$ is the corresponding eigenvector.
As previously mentioned $s(P, i, j)$ is assumed to be continuously differentiable in $P \in \R^{l\times
N}$ for fixed $i, j \in \{1\dots N\}$.
%
%For brevity we temporarily drop the notational dependence on $\btheta$ and
%denote the second eigenvalue of the Laplacian by $\lambda$, and the
%corresponding eigenvector by $u$. 
%
The derivative $D_{\btheta} \lambda_2(\cdot)$ is given by the $(d-1)\times l$
matrix with $i$-th column $D_{\btheta_i} \lambda_2(\cdot)$, which can be obtained
through the chain rule decomposition,
\[ D_{\pmb{\theta_i}} \lambda_2(\cdot) = D_{P} \lambda_2 \, D_{V}P \,
D_{\pmb{\theta_i}}V, \]
where $D_{\cdot}\cdot$ is the differential operator. Since only the $i$-th
column of $V$ depends on $\btheta_i$, and only the $i$-th row of $P$ depends on
$V_i$, this product can be simplified as 
$$
D_{\pmb{\theta_i}} \lambda_2(\cdot)
= D_{P_i} \lambda_2 \, D_{V_i}P_i \, D_{\pmb{\theta_i}}V_i,
$$
where $P_i$ is
used to denote the $i$-th row of $P$, while $V_i$ and $\btheta_i$ are, as
usual, the $i$-th columns of $V$ and $\btheta$ respectively. 
By definition $D_{V_i} P_i = X^\top$, while
$D_{\btheta_i} V_i \in \R^{d\times (d-1)}$ is obtained by differentiating Eq.~(\ref{angles}),
and is given by,
\begin{equation}\label{eq:difftheta}
\frac{\partial V(\btheta)_{ji}}{\partial \btheta_{ki}} = \left\{\begin{array}{ll}
0, & j<k\\
-\sin(\btheta_{ki})\prod \limits_{m=1}^{k-1}\sin(\btheta_{mi}), & j=k<d\\
\cos(\btheta_{ki})\cos(\btheta_{ji})\prod \limits_{m<j, m\not=k}\sin(\btheta_{mi}), & k<j<d\\
\cos(\btheta_{ki}) \prod \limits_{m \not = k}\sin(\btheta_{mi}), & j = d.
\end{array}\right.
\end{equation}
Finally, in the case of the standard Laplacian, we find,
\begin{equation}\label{eq:derivlam}
\frac{\partial{\lambda_2(L)}}{\partial P_{ij}} = \frac{1}{2}\sum_{m, n}(u_m-u_n)^2\frac{\partial s(P, m, n)}{\partial P_{ij}},
\end{equation}
and for the normalised Laplacian we instead have,
\begin{align}\label{eq:derivlamnorm}
\nonumber
\frac{\partial \lambda_2(L_{\mathrm{N}})}{\partial P_{ij}} = & \frac{1}{2}\sum_{m, n} \left(\frac{u_m}{\sqrt{d_m}} - \frac{u_n}{\sqrt{d_n}}\right)^2\frac{\partial s(P, m, n)}{\partial P_{ij}}\\
& - \lambda \sum_{m, n} \frac{u_m^2}{d_m}\frac{\partial s(P, m, n)}{\partial P_{ij}}.
\end{align}
Complete derivations of Eqs.~(\ref{eq:derivlam}) and ~(\ref{eq:derivlamnorm})
can be found in Appendix~\ref{sec:deriv}.
The elements of the eigenvector, $u$, are continuous since we have
assumed the corresponding eigenvalue $\lambda_2(\cdot)$ to be simple~\citep{Magnus1985}.
In addition we have assumed that $s$ is continuously differentiable.
Therefore, the
product $D_{P} \lambda_2 \, D_{V}P \,
D_{\pmb{\theta_i}}V$ is continuous in $\btheta$, as desired.

If $\lambda_2(\cdot)$ is not simple at $\btheta$ the derivative
$D_{\pmb{\theta}} \lambda_2(\cdot)$ may not be defined. 
Gradient sampling~\citep{BurkeLO2006} can be applied
to minimising objectives which are not
differentiable everywhere.
The method works by sampling points within a
shrinking radius, $\epsilon$,
of the current iterate. The convex hull of the gradients at these sampled
points acts as an approximation for the Clarke $\epsilon$-subdifferential,
and the minimum norm element of this convex hull
provides an approximate steepest descent direction.
This approach
is appealing for its broad applicability and almost sure
convergence to a local minimum on objectives which are
locally Lipschitz and almost everywhere continuously differentiable.
However to obtain a search direction at each iteration
a quadratic program has to be solved, the
formulation of which requires $\mathcal{O}(d)$ gradient
computations. This makes the
method computationally expensive for large problems.
We consider a simple modification which exploits the
properties of eigenvalues of graph Laplacians,
and uses directional derivatives to derive optimality conditions.

%
%At such points a descent direction can be identified using directional
%derivatives.
%

The eigenvalues of a real symmetric matrix can be expressed as the difference
between two convex matrix functions~\citep{fan1949theorem}. 
Therefore $\lambda_2(\Ltheta)$ and $\lambda_2(\LNtheta)$ are directionally differentiable everywhere.
%
%(we have assumed $s$ is differentiable).
%provided
%the similarity function, $s$, is (Lipschitz continuous and) differentiable.
%
%$\lambda_2(\Ltheta)$ and $\lambda_2(\LNtheta)$ are directionally differentiable
%everywhere.
%
\citet{overton1993optimality} provide an expression for the directional
derivative of the sum of the $K$~largest eigenvalues of a matrix whose elements
are continuous functions of a parameter, at a point of non-simplicity of the
$K$-th largest eigenvalue. We discuss the case of $\lambda_2(\Ltheta)$, where
$\lambda_2(\LNtheta)$ is analogous. Consider the function $F^K:\R^{N\times N} \to \R$ which takes
as input a square matrix and returns
the sum of its $K$ largest eigenvalues. Then,
\begin{equation*}
\lambda_2(\Ltheta) = F^{N-1}(\Ltheta) - F^{N-2}(\Ltheta).
\end{equation*}
Now consider a $\btheta$ such that,
\begin{align*} 
&\lambda_N(\Ltheta) \geqslant \dots \geqslant \lambda_{N-r+1}(\Ltheta) > \\
&\lambda_{N - r}(\Ltheta)  = \dots = \lambda_{N-K+1}(\Ltheta) =\\
& \dots =
	\lambda_{N - r - t + 1}(\Ltheta) \\ &> \lambda_{N - r - t}(\Ltheta)
	\geqslant \cdots > \lambda_1(\Ltheta) =0.
\end{align*}
That is, the $K$-th largest eigenvalue has multiplicity $t$ and $K-r$ of the
repeated eigenvalues are included in the sum $F^K(\Ltheta)$.
\citet{overton1993optimality} have shown that the directional derivative of $F^K(\Ltheta)$ in the direction
$\pmb{\psi}$, $d F^{K}(\Ltheta; \pmb{\psi})$, is
equal to,
\begin{align*}
%
%d F^{K}(\Ltheta; \theta) = & 
F^r \left( \sum_{i=1}^{d-1}\sum_{j=1}^l\pmb{\psi}_{ij} R^\top L_{ij}R \right) + 
	 F^{K-r} \left( \sum_{i=1}^{d-1}\sum_{j=1}^l\pmb{\psi}_{ij} Q^\top L_{ij} Q \right),
\end{align*}
where $L_{ij} = \partial \Ltheta/\partial \btheta_{ij}$,
the $j$-th column of the
matrix $R \in \R^{N\times r}$ is equal to the eigenvector
associated with the $j$-th largest eigenvalue of $\Ltheta$,
and the $j$-th column of the matrix $Q \in \R^{N\times t}$
is equal to the eigenvector associated with the $(r+j)$-th largest eigenvalue of $\Ltheta$.
The
directional derivative of $\lambda_2(\Ltheta)$ in the direction $\pmb{\psi}$ is thus,
\begin{align}
\nonumber
d \lambda_2(\Ltheta; \pmb{\psi}) = & d F^{N-1}(\Ltheta; \pmb{\psi}) - d F^{N-2}(\Ltheta; \pmb{\psi}) \\
	= & \lambda_1 \left( \sum_{i=1}^{d-1}\sum_{j=1}^l\pmb{\psi}_{ij} Q^\top L_{ij} Q   \right)\label{eq:dderiv},
\end{align}
where the columns of~$Q$ are given by the complete set of eigenvectors
for the eigenvalue $\lambda = \lambda_2(\Ltheta)$.

\subsection{Minimising $\lambda_2(\Ltheta)$ and $\lambda_2(\LNtheta)$.}\label{sec:optimisation}
Applying standard gradient descent methods to functions which are almost
everywhere differentiable can result in convergence to sub-optimal
points~\citep{wolfe1972convergence}. 
This occurs when the method for determining
the gradient is applied at a point of non-differentiability and produces a
non-descent direction. In this case the algorithm cannot reduce
the objective function value and terminates at a point that is not necessarily a local minimum.
The second eigenvalues of the graph Laplacian matrices,
while not necessarily differentiable everywhere, benefit
from the fact that their minimisation tends to
separate them from other eigenvalues. Thus a standard gradient
descent algorithm performs well on these objectives,
very often converging to locally optimal solutions.
Our approach for minimising $\lambda_2(\Ltheta)$ and
$\lambda_2(\LNtheta)$, therefore assumes them to be continuously
differentiable until there is evidence that this assumption fails.
Only then is it necessary to use the computationally more expensive gradient sampling
algorithm to identify a descent direction.

Our approach is summarised in Algorithm~\ref{alg:optimisation}.
Once again we discuss only $\lambda_2(\Ltheta)$ explicitly, noting that
the methodology for minimising $\lambda_2(\LNtheta)$ is equivalent, with the
only difference being in the computation of the gradients and directional
derivatives.

At each iteration a standard gradient-based algorithm with inexact line-search
is used to minimise the objective function
%
%As at this stage the objective function is assumed to be differentiable, the gradient is computed through
%
using the formulation for the gradient presented in Section~\ref{sec:differentiability}.
When this algorithm terminates, say with solution $\btheta^\star$, either the magnitude of the
computed gradient is below a threshold, or a sufficient
decrease in the objective function value was not feasible.
We then need to verify whether $\btheta^\star$ is a local minimum. 
%
%To this end we check the simplicity of the eigenvalue. 
%
If $\lambda_2(L(\btheta^\star))$ is simple then $\lambda_2(\cdot)$ is continuously
differentiable at $\btheta^\star$, and therefore $\btheta^\star$ is close to a local
minimiser. In this case the algorithm terminates. On the other
hand, if $\lambda_2(L(\btheta^\star))$ is not simple,
then $\btheta^\star$ may or may not be a local minimiser. 
%
%A computationally inexpensive way to check whether
%a descent direction from $\btheta^\star$ exists is through
The directional derivative formulation
in Eq.~(\ref{eq:dderiv}) provides a computationally efficient
way to determine if a descent direction from $\btheta^\star$ exists.
In particular, if at $\btheta^\star$, $Q^\top L_{ij}Q  \approx \mathbf{0}$ for all pairs, $i, j$, then 
the directional derivative $d\lambda_2(L(\btheta^\star); \pmb{\psi})$
is approximately zero
%
%\begin{equation}
%%
%d\lambda_2(\Ltheta; \theta) = \lambda_1\left(\sum_{i=1}^{d-1}\sum_{j=1}^l \theta_{ij} Q^\top L_{ij}Q\right) \approx \mathbf{0},
%%
%\end{equation}
%%
for all directions~$\pmb{\psi}$. In this case the algorithm terminates as $\btheta^\star$ is sufficiently
close to a local minimiser.
If this condition is not met a descent directions exists, that is $\exists \pmb{\psi} \in \Theta$ s.t.
$\lambda_1\left(\sum_{i=1}^{d-1}\sum_{j=1}^l \pmb{\psi}_{ij} Q^\top L_{ij}Q\right)<0$.
At this point a single step of the gradient sampling algorithm is performed.
As in the standard gradient sampling algorithm~\citep{BurkeLO2006} the magnitude of the sampling radius $\epsilon$ is
progressively reduced until a valid descent direction is identified, or the radius
is reduced beyond a user-specified threshold $\epsilon_f$. In the latter case the
current solution is considered sufficiently close to a local minimiser and the algorithm terminates.
In the former case, once a valid descent direction is identified
$\btheta^\star$ is updated using an inexact line-search algorithm.

Termination under any of the above criteria indicates the identification of a
local minimiser. Moreover, the convergence of the method is guaranteed
under the same analyses as for gradient descent on smooth
functions~\citep{nocedal2006numerical} and gradient sampling~\citep{BurkeLO2006}.

\begin{algorithm}
\caption{Minimising $\lambda_2(\Ltheta)$}\label{alg:optimisation}
\begin{algorithmic}
\STATE Input: Initial projection $\btheta_0$, optimality tolerance $\tau$,
\STATE initial sampling radius for gradient sampling $\epsilon_0$,
\STATE minimum sampling radius $\epsilon_f$, radius reduction
\STATE factor $\eta$, number of sampled gradients $n_g$
\STATE Output: Optimal projection $\btheta^\star$
\STATE
\STATE $\btheta^\star \gets \btheta_0$
\STATE $\epsilon \gets \epsilon_0$
\WHILE{$\epsilon > \epsilon_f$}
	\STATE \# {\em apply standard gradient descent to convergence}
	\STATE $\btheta^\star \gets \mathrm{Gradient Descent Solution}(\btheta^\star)$
	\STATE \# {\em check for optimality of the solution}
	\IF{$\lambda_2(L(\btheta^\star))$ is simple {\bf or} $\max_{i,j} \vert Q^\top L_{ij} Q \vert < \tau$}
	\STATE {\bf return} $\btheta^\star$
	\ELSE 
	\STATE \# {\em obtain gradients at points sampled uniformly in a}
	\STATE \# {\em ball of radius $\epsilon$ around the current solution}
	\FOR {$i = 1 \ldots n_g$}
	\STATE $\btheta_i \sim U(\mathcal{B}_{\epsilon}(\btheta^\star))$
	\STATE $\pmb{\Gamma}_i \gets D_{\btheta} \lambda_2(\Ltheta) \vert_{\btheta = \btheta_i}$
	\ENDFOR
	\STATE \# {\em obtain the search direction}
	\STATE $\pmb{\Gamma}_s \gets \mathrm{argmin}_{\pmb{\Gamma} \in \mathbf{conv}(\{\pmb{\Gamma}_1, \ldots, \pmb{\Gamma}_{n_g}\})} \|\pmb{\Gamma} \|_F$
	\STATE \# {\em if the magnitude of the search direction is below}
	\STATE \# {\em the optimality threshold decrease sampling radius}
	\IF{$\|\pmb{\Gamma}_s\|_F < \tau $}
		\STATE $\epsilon \gets \eta \epsilon$
	\ELSE
	\STATE \# {\em update solution using inexact line-search}
	\STATE $\nu^\star \gets \ \approx \mathrm{argmin}_{\nu > 0} \lambda_2(L(\btheta^\star - \nu \pmb{\Gamma}_s))$
	\STATE $\btheta^\star \gets \btheta^\star - \nu^\star \pmb{\Gamma}_s$
	\ENDIF
	\ENDIF
	\ENDWHILE
\STATE {\bf return} $\btheta^\star$
\end{algorithmic}
\end{algorithm}

A brief derivation of the computational complexity of each iteration of the method is provided in Appendix~\ref{sec:complexity}.
Each step in the standard gradient descent algorithm requires $\mathcal{O}(lN(N+d(d-1)))$ operations. The gradient sampling step
requires $\mathcal{O}(d)$ gradient computations, therefore having complexity $\mathcal{O}(dlN(N+d(d-1)))$. The complexity of
computing the optimality conditions using directional derivatives is similar, requiring $\mathcal{O}(t^2lN(n+d(d-1)))$ operations, where
$t$ is the multiplicity of the eigenvalue $\lambda = \lambda_2(\Ltheta)$. Our experience with this method indicates that the algorithm
almost always terminates with $\lambda_2(\cdot)$ being simple, without the need for any gradient sampling or directional derivative computations.

%\subsection{The Similarity Function}\label{sec:sim}

%\subsection{Computing Similarities}\label{sec:sim}

Figure~\ref{fig:2dplots} shows two dimensional plots of a subset of the datasets used in our experiments in Section~\ref{sec:experiments}. The left plots show projections of the data onto the first two principal components.
The right plots show the optimal projections of the data obtained by minimising the objective in~(\ref{eq:orthog}) by applying Algorithm~\ref{alg:optimisation}, and using the normalised Laplacian. Figures~\ref{fig:2dyale} and~\ref{fig:2diso}
show examples where the principal components do not show a clear identification of any of the clusters, whereas the
optimal projections for spectral clustering clearly admit a strong separation of clusters. In Figure~\ref{fig:2dhdd} the principal component projection does show some separation of clusters. In this case optimisation of the spectral connectivity serves to enhance this separation, and make the individual clusters more compact. 

\begin{figure}[t!]
\subfigure[Yale Faces B]{\includegraphics[width = 4cm]{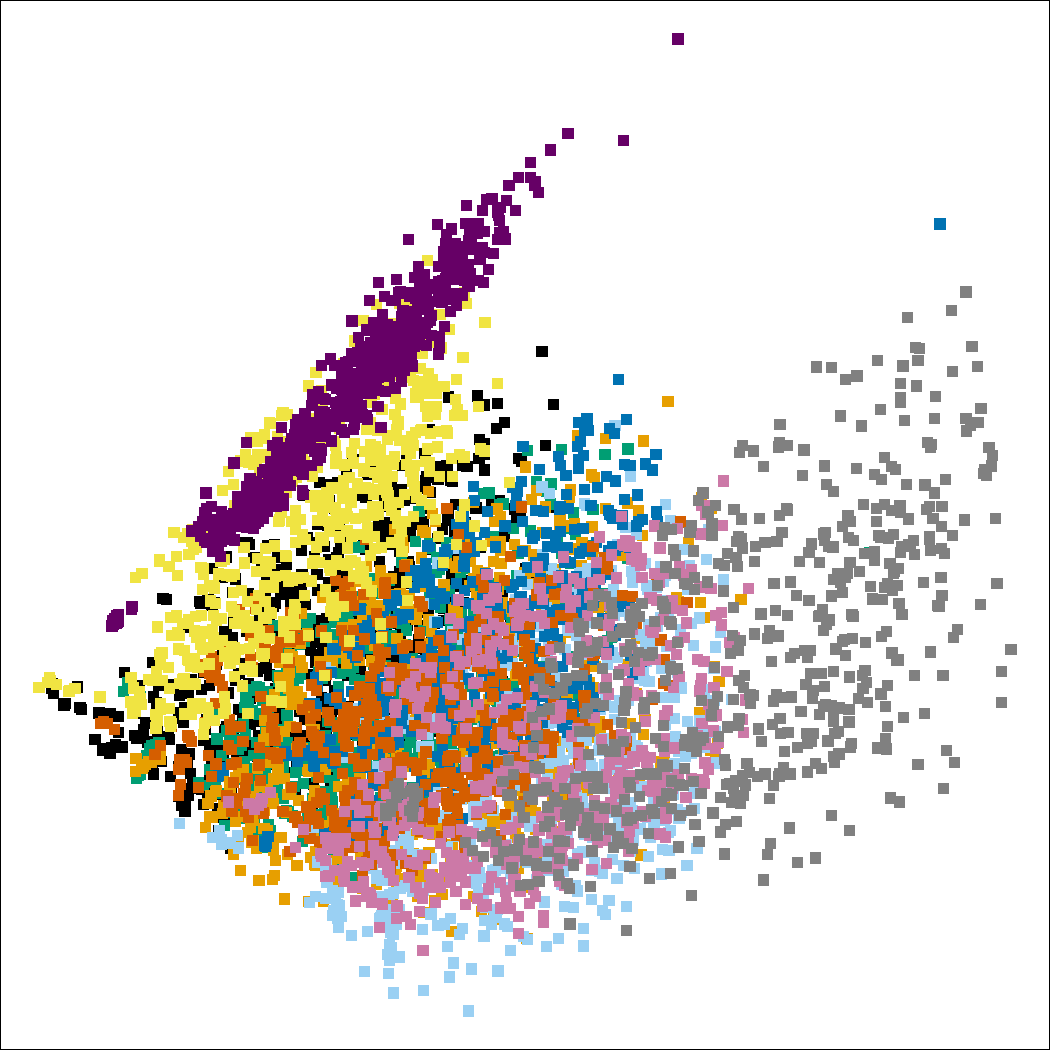} \includegraphics[width = 4cm]{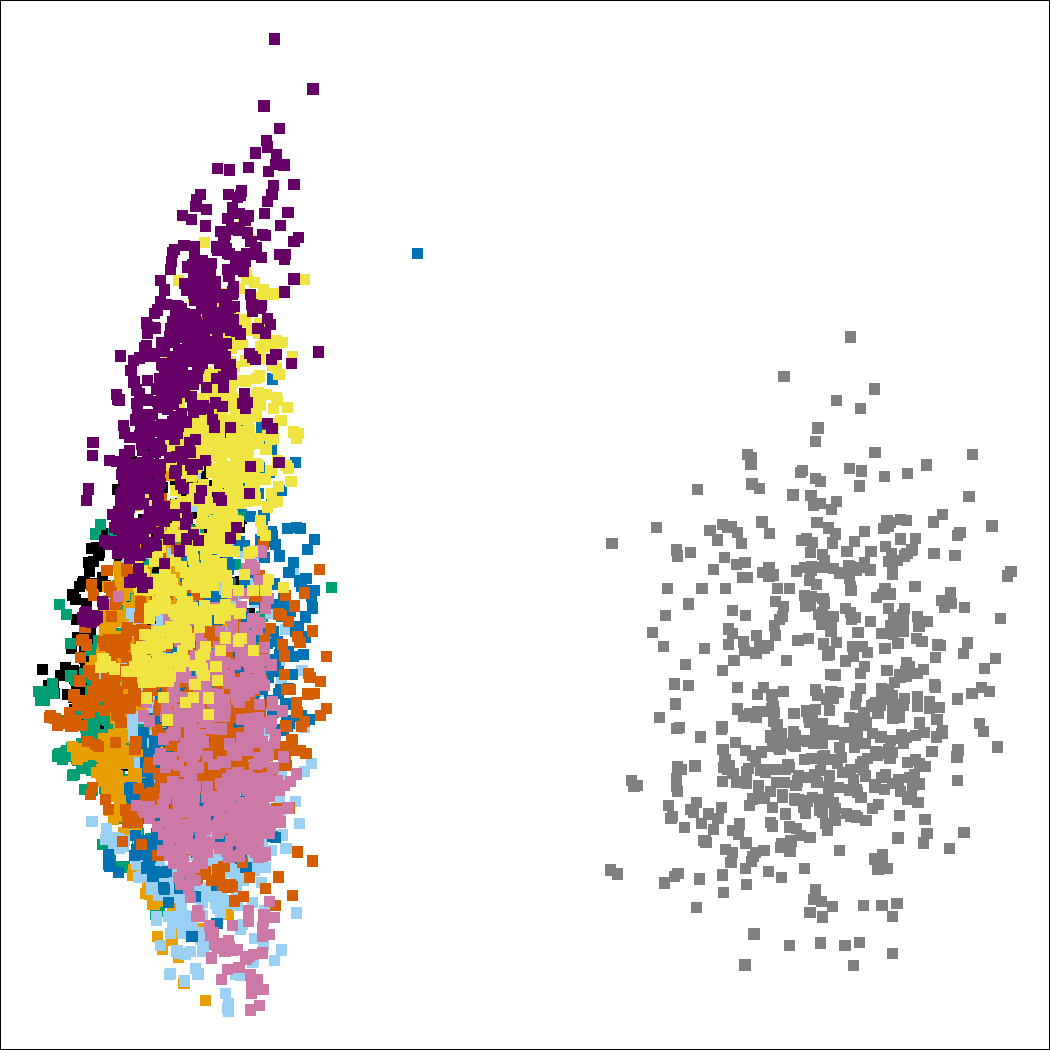}\label{fig:2dyale}}
\subfigure[Isolet]{\includegraphics[width = 4cm]{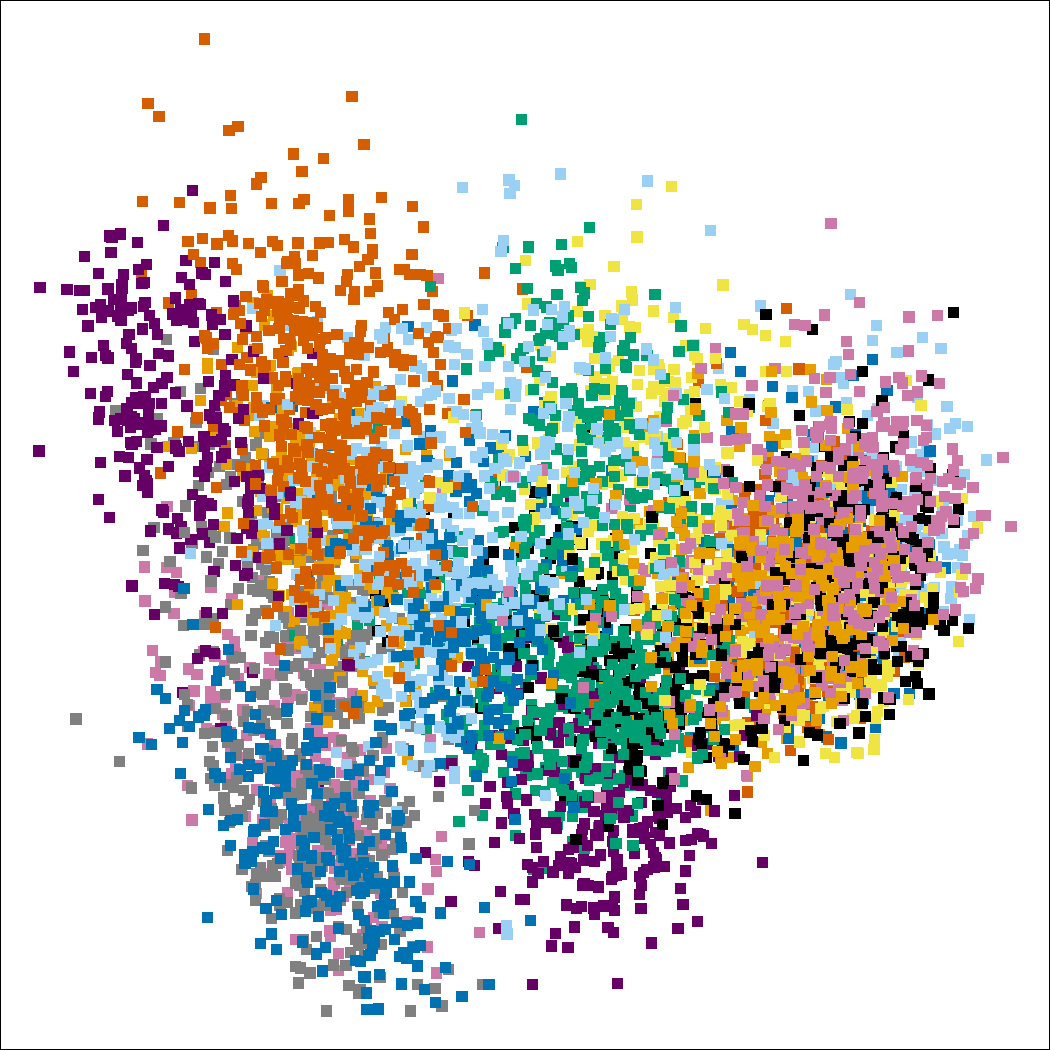} \includegraphics[width = 4cm]{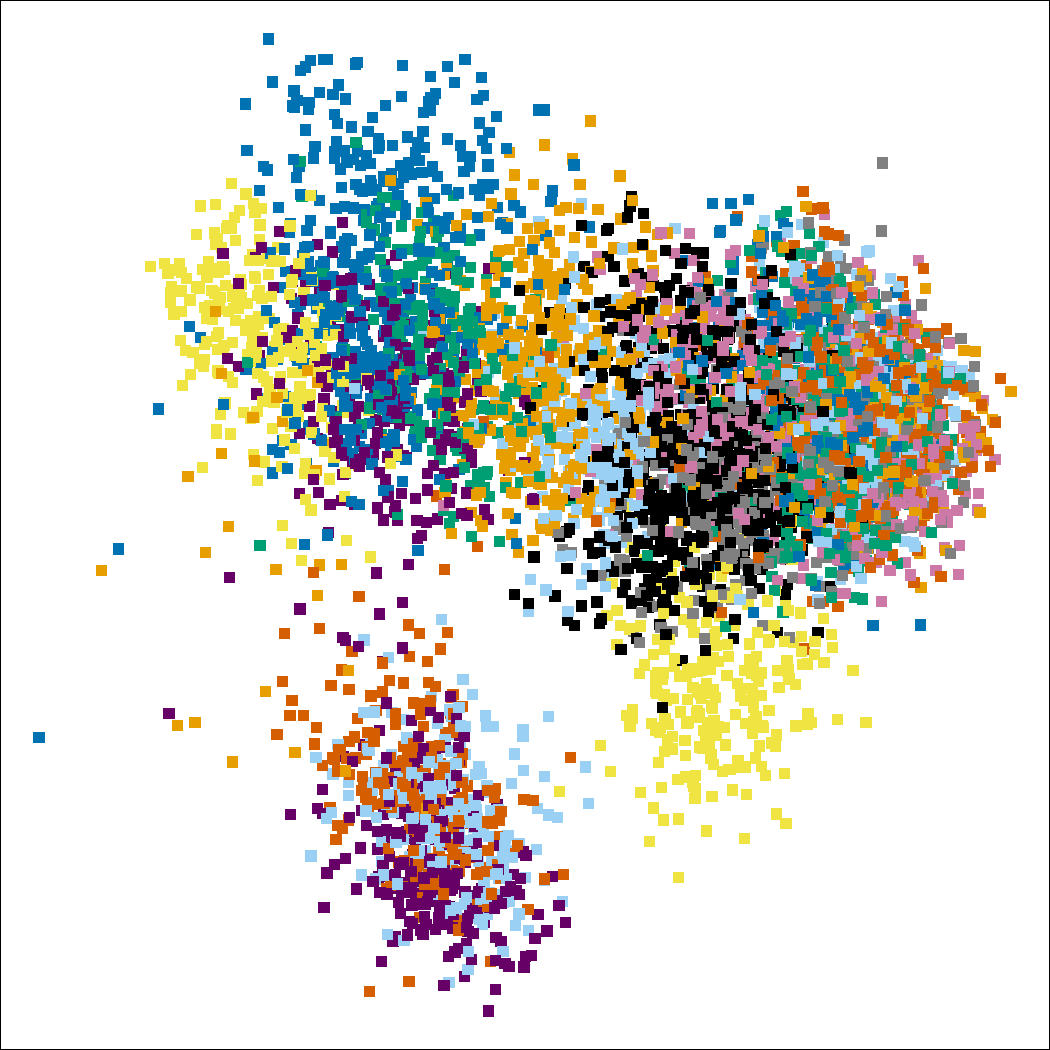}\label{fig:2diso}}
%\subfigure[Smartphone]{\includegraphics[width = 4cm]{PhonePCA.pdf} \includegraphics[width = 4cm]{PhoneOpt.pdf}}\\
%\subfigure[Pen Digits]{\includegraphics[width = 4cm]{PenPCA.pdf} \includegraphics[width = 4cm]{PenOpt.pdf}}\\
\subfigure[Multiple Feature Digits]{\includegraphics[width = 4cm]{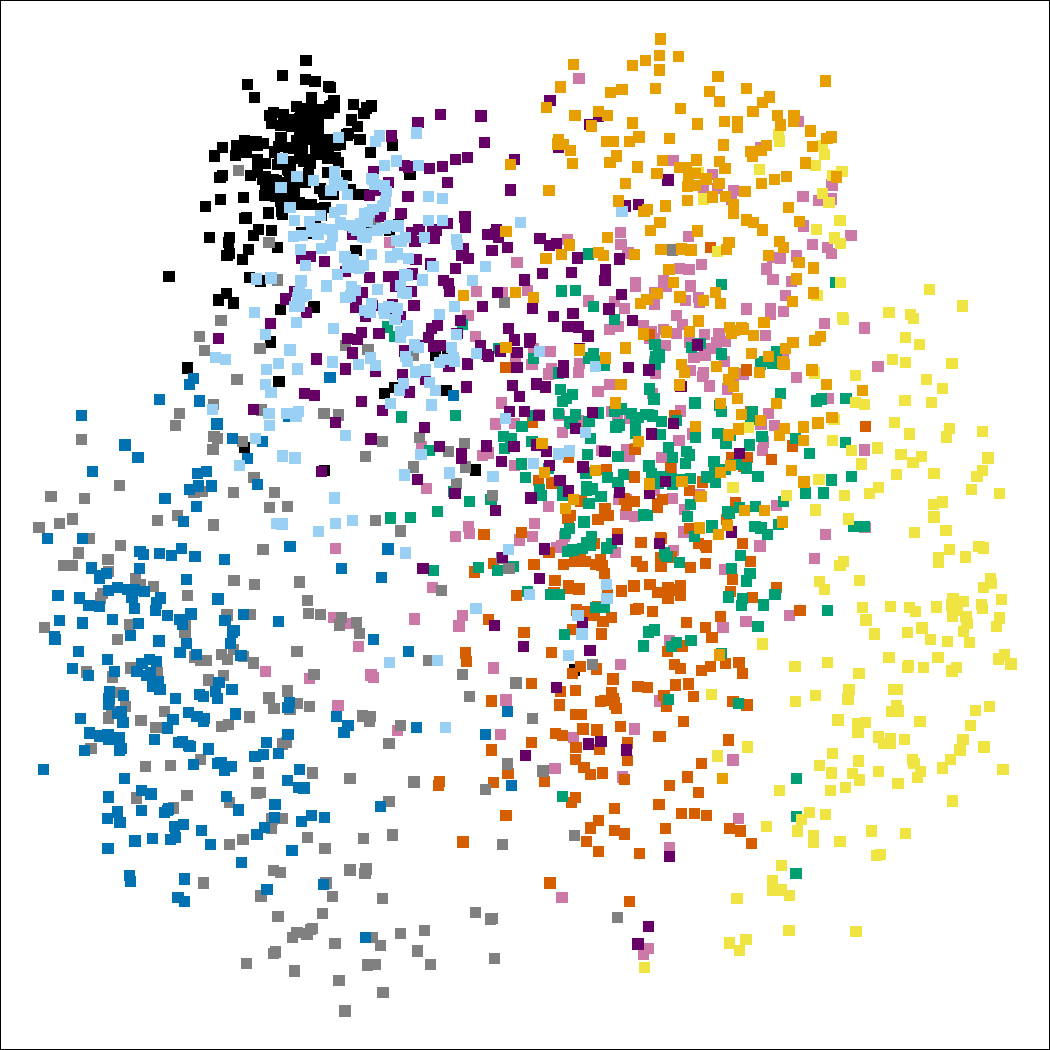} \includegraphics[width = 4cm]{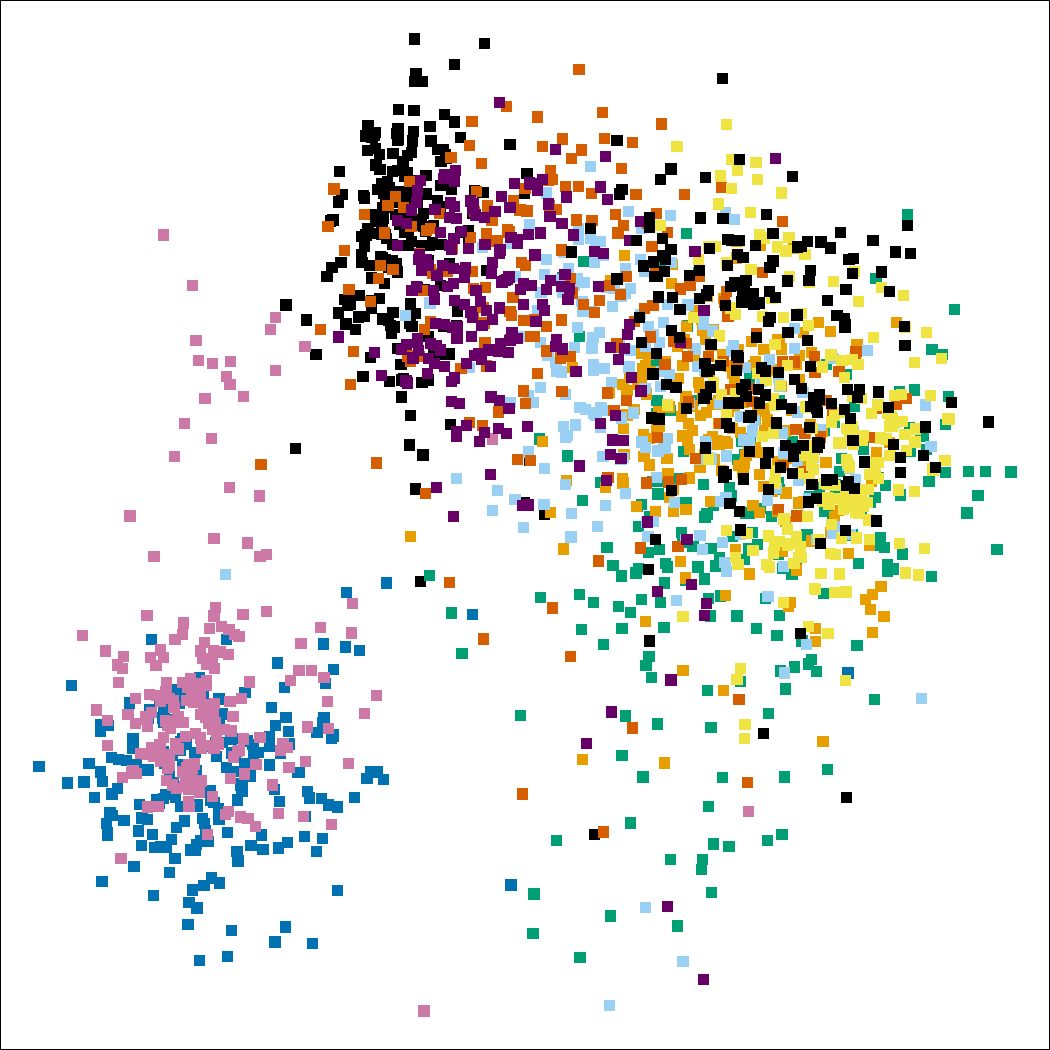}\label{fig:2dhdd}}
\caption{Two dimensional projections of publicly available datasets. PCA (left) and optimal projection for spectral clustering (right). \label{fig:2dplots}}
\end{figure}

%%%%%%%%%%%%%%%%%%%%%%%%%%%%%%%%%%%%%%%%%%%%%%%%%%%%%%%%%%%%%%%%%%%%%%%%%%%%%%%%

\section{Connection to Maximum Margin Hyperplanes} \label{sec:maxmargin}

Maximum margin hyperplanes have become a unifying principle
in data classification tasks. Starting with the fully
supervised problem using support vector machines~\citep{vapnik1982estimation},
the methodology has been extended to semi-supervised classification~\citep{Joachims1999},
and more recently to the problem of maximum margin clustering~\citep{xu2004maximum,zhang2009maximum}.

In this section, we establish a connection between the optimal univariate
projection for spectral clustering
and maximum margin hyperplanes for clustering. In particular, we show that under suitable
conditions, as the scaling parameter, $\sigma$, tends to zero, the optimal
univariate projection for spectral bi-partitioning converges to the vector
normal to the largest margin hyperplane through the data. This establishes
a theoretical connection between separability measured by spectral graph
theory, and standard notions of separation in terms of the Euclidean metric.
Connections between maximum margin hyperplanes and Bayes optimal hyperplanes~\citep{Tong2000} as
well as minimum density hyperplanes~\citep{Pavlidis2015} have previously been established.
The result we discuss herein therefore connects spectral connectivity to these
objectives as well.

In this section we
use the notation $v(\pmb{\theta})$ instead of $V(\btheta)$ to stress that
the we are concerned with univariate projections.
A hyperplane is a translated subspace of co-dimension 1, and can be
parameterised by a vector $v \in \mathbb{R}^d \setminus \{0\}$
and a scalar $b$ as the set $H(v, b) = \{x \in \mathbb{R}^d \big \vert v^\top x =
b\}$. 
No generality is lost if $v$ is assumed to have unit norm,
thus the same
parameterisation by $\pmb{\theta}$ can be used.
%
%
%Clearly, for any $c \in \mathbb{R}\setminus \{0\}$, one has $H(v, b) =
%H(cv, cb)$, and so we can assume that $v$ has unit norm, thus the same
%parameterisation by $\pmb{\theta}$ can be used. 
%
For a finite set of points $\mathcal{X}$ in $\R^d$,
the \emph{margin} of hyperplane $H(v(\pmb{\theta}), b)$
w.r.t. $\mathcal{X}$ is the minimal Euclidean distance between $H(v(\pmb{\theta}), b)$
and $\mathcal{X}$,
\begin{equation}
\mbox{margin}(v(\pmb{\theta}), b) = \min_{x \in \mathcal{X}}\vert v(\pmb{\theta})^\top x
- b\vert.
\end{equation}

The set $\pmb{\Delta}(\btheta)$ again plays an important role as in many cases the largest margin hyperplane through a set of data separates only a few points from the rest, making it meaningless for the purpose of clustering. For the theory presented herein we consider an arbitrary convex and compact set $\pmb{\Delta} \subset \R^d$ and define $\pmb{\Delta}(\btheta)$ to be the projection of $\pmb{\Delta}$ onto $v(\btheta)$.
% We therefore prefer to restrict the hyperplane to intersect the set $\pmb{\Delta}$. 
What we in fact show in this section is that there exists a set $\pmb{\Delta}^\prime \subset \pmb{\Delta}$ satisfying
$\pmb{\Delta}^\prime \cap \mathcal{X} = \pmb{\Delta} \cap \mathcal{X}$, such that, as the scaling parameter tends to zero, the optimal projections for $\lambda_2(\Ltheta)$ and $\lambda_2(\LNtheta)$ converge to the vector admitting the largest margin hyperplane that intersects $\pmb{\Delta}^\prime$. The distinction between the largest margin hyperplane intersecting $\pmb{\Delta}^\prime$ and that intersecting $\pmb{\Delta}$ is scarcely of practical relevance, but plays an important role theoretically. It accounts for situations when the largest margin hyperplane intersecting $\pmb{\Delta}$ lies close to its boundary and the distance between the hyperplane and the nearest point outside $\pmb{\Delta}$ is larger than to the nearest point inside $\pmb{\Delta}$. Aside from this very specific case, the two solutions in fact coincide.

The following theorem is the main result of this section. The proof and supporting results are provided in Appendix~\ref{sec:proofs}.
The result holds for all similarities in which the function $k$, in Eq.~(\ref{eq:sim}), satisfies the tail condition $\lim_{x\to\infty} k((1 + \epsilon)x)/k(x) = 0$ for all $\epsilon > 0$. This condition is satisfied by functions with exponentially decaying tails, including the popular Gaussian and Laplace kernels, but not those with polynomially decaying tails.

The proof of the result relies on obtaining upper and lower bounds on the magnitude of
$\lambda_2(\Ltheta)$ and $\lambda_2(\LNtheta)$ which depend essentially
on $k(M/\sigma)$, where $M$ is the largest gap between consecutive points
in $\mathcal{P}(\btheta)$. Notice that $M$ is equal to twice the maximum
margin of all hyperplanes orthogonal to $v(\btheta)$. These bounds show
immediately that as $\sigma$ approaches zero, if $\lambda_2(L(\btheta_1))<
\lambda_2(L(\btheta_2))$ (or $\lambda_2(L_{\mathrm{N}}(\btheta_1))<
\lambda_2(L_{\mathrm{N}}(\btheta_2))$) then the maximum margin of all
hyperplanes orthogonal to $v(\btheta_1)$ is greater than the maximum margin
of all hyperplanes orthogonal to $v(\btheta_2)$. The convergence
of the optimal projection itself to the vector normal to the maximum margin
hyperplane uses a property of the maximum margin hyperplane established
by~\citet{Pavlidis2015}. 

\begin{thm} \label{thm:convergence}
Let $\mathcal{X} = \{x_1, ..., x_N\}$ be a finite set of points in $\R^d$ and suppose that there is a unique hyperplane, which can be parameterised by
$(v(\btheta^\star), b^\star)$, intersecting $\pmb{\Delta}^\prime$ and attaining maximal margin on $\mathcal{X}$.
Let $k : \mathbb{R}_+ \to \mathbb{R}_+$ be decreasing, positive and
satisfy $\lim_{x \to \infty} k((1+\epsilon)x)/k(x) = 0$ for all $\epsilon > 0$. For $\sigma>0$
define 
$$
\btheta_{\sigma}:=\mbox{argmin}_{\btheta \in \Th}\lambda_2(L(\btheta, \sigma)),$$
$$
\btheta^N_{\sigma}:=\mbox{argmin}_{\btheta \in \Th}\lambda_2(L_{\mathrm{N}}(\btheta, \sigma)),$$
where 
%$L(\btheta, \sigma, \delta)$ (resp. $L_{\mathrm{N}}(\btheta, \sigma, \delta)$) is as $\Ltheta$ (resp. $\LNtheta$) from before but with 
there is now an explicit dependence on the
scaling parameter, $\sigma$. Then,
\begin{align*}
\lim_{\sigma \to 0^+} v(\btheta_{\sigma}) \ =
\lim_{\sigma \to 0^+} v(\btheta^N_{\sigma}) \ \ = \ \ v(\btheta^\star).
\end{align*}
\end{thm}

We note that the same result holds when using the Euclidean metric. In this case the optimal
projection based on spectral connectivity converges to the vector normal to the maximum margin
hyperplane through the data. The importance
of constraining the maximum margin hyperplane to avoid separating only outliers was also observed by~\citet{xu2004maximum} and~\citet{ zhang2009maximum}.

While the above result is only established
for univariate projections, we have
observed empirically that if a decreasing sequence of scaling parameters is
employed for a multivariate projection, then 
the projected data, $\mathcal{P}(\btheta)$, tend to exhibit
large Euclidean separation. This is illustrated in Figure~\ref{fig:LMSCyeast}
which shows two dimensional plots of the 72 dimensional yeast cell cycle
analysis dataset~\citep{BacheL2013}.
The left plots show the true clusters,
while the right plots show the cluster assignments made by the algorithm.
In Figure~\ref{fig:1dmm} the horizontal axis corresponds to the optimal projection
obtained by minimising $\lambda_2(\LNtheta)$ for a decreasing sequence of
scaling parameters, while the vertical axis is the direction of maximum variance
orthogonal to this vector. Figure~\ref{fig:2dmm} instead shows the result of
two dimensional projection pursuit for a decreasing sequence of scaling parameters.

\begin{figure}[!t]
\subfigure[One dimensional projection pursuit]{\includegraphics[width = 3.8cm]{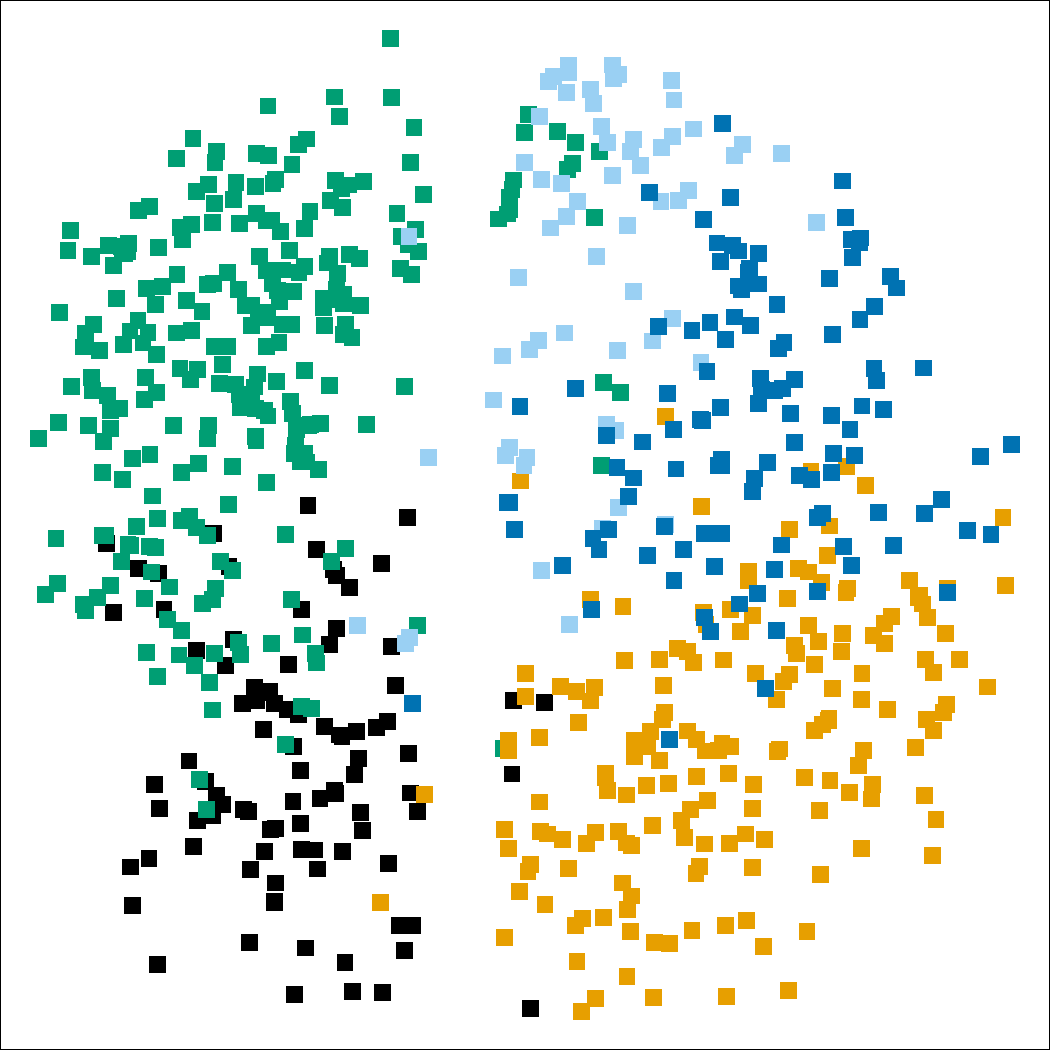}
\includegraphics[width = 3.8cm]{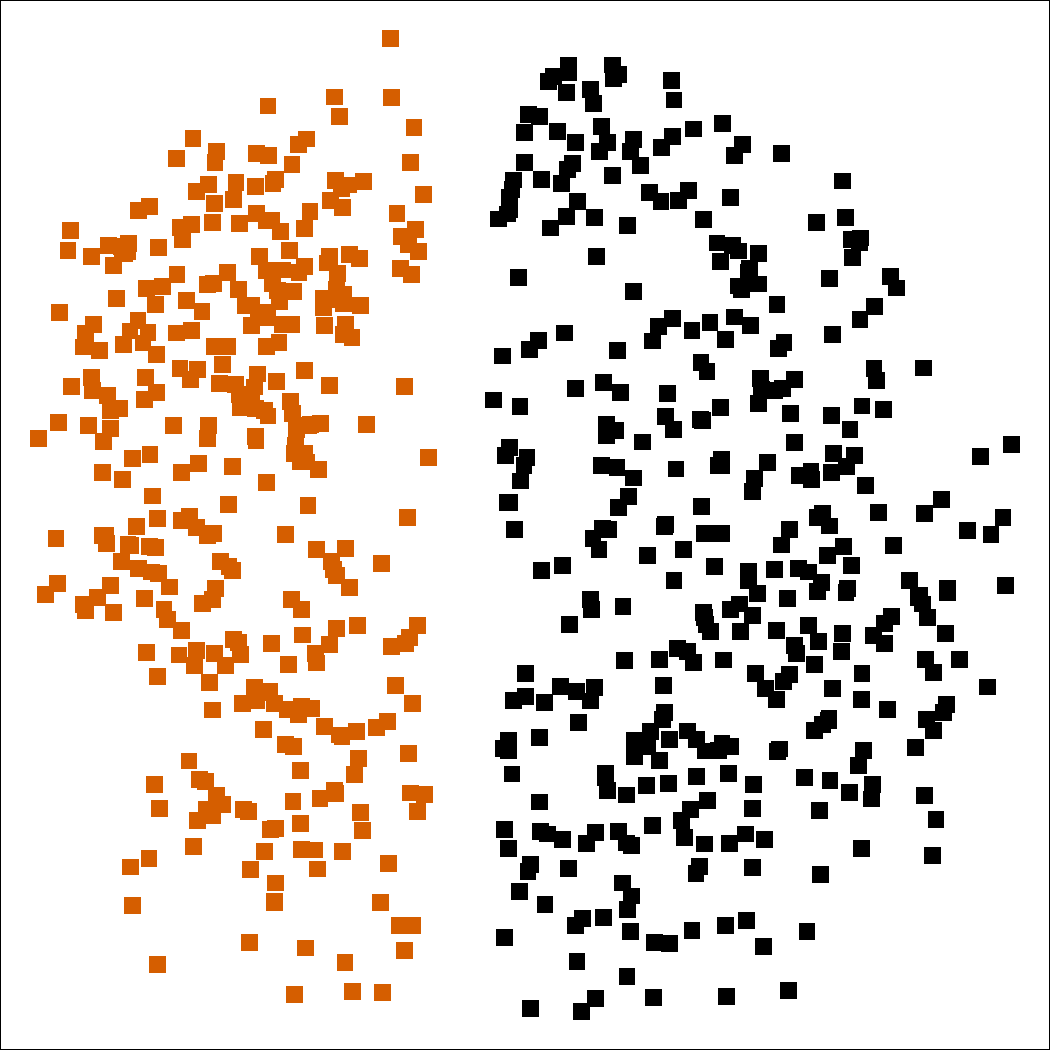} \label{fig:1dmm}}
\subfigure[Two dimensional projection pursuit]{\includegraphics[width = 3.8cm]{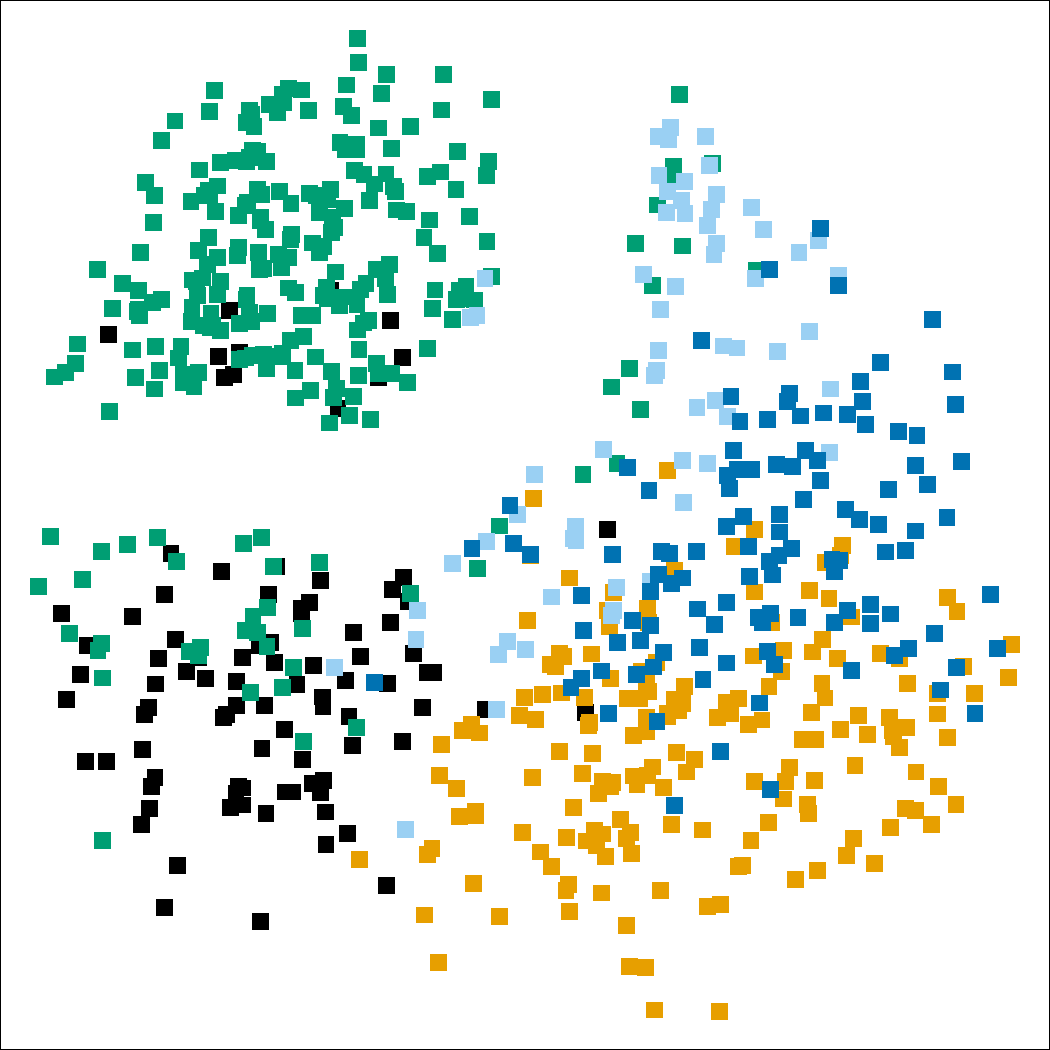}
\includegraphics[width = 3.8cm]{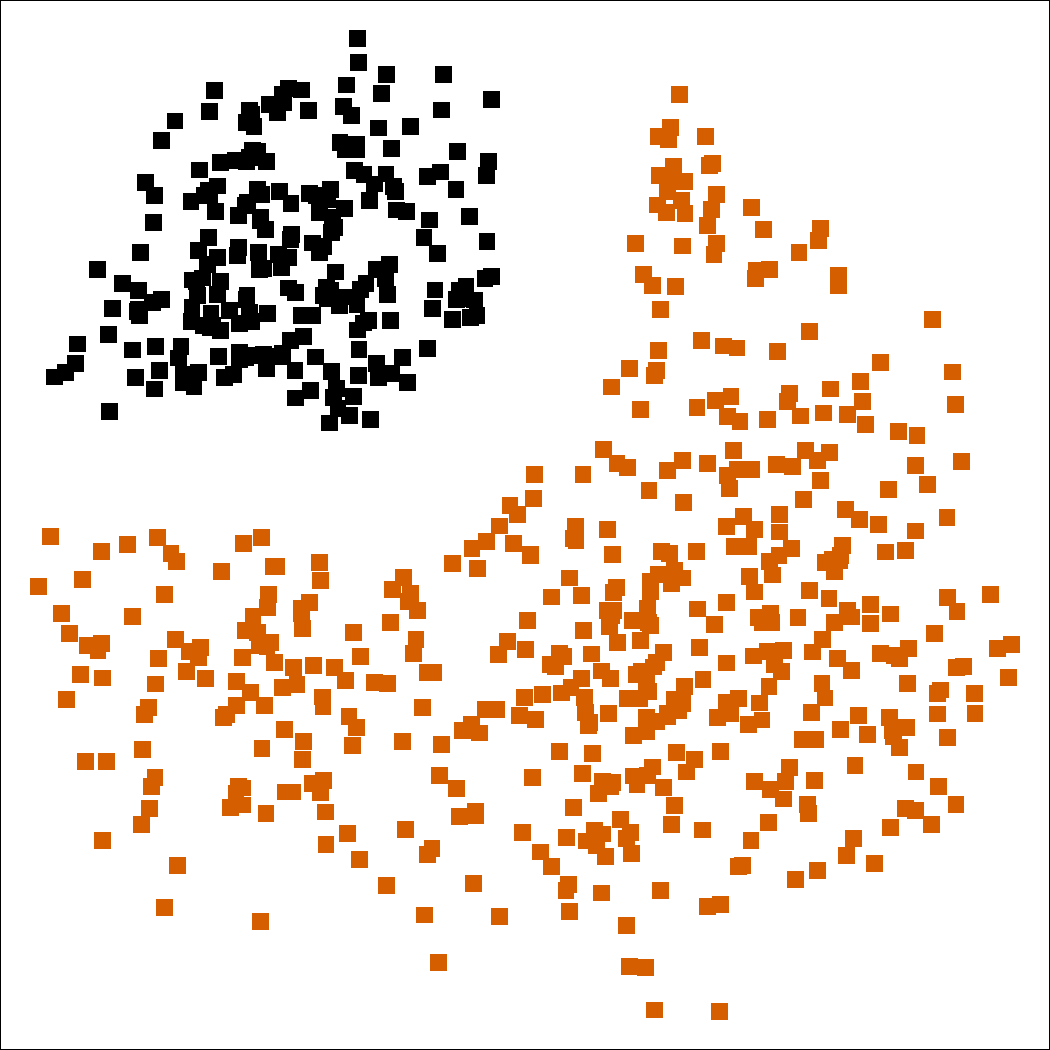} \label{fig:2dmm}}
%
%\subfigure[Max margin clustering]{\includegraphics[scale=0.2]{lmscyeast3.pdf}}
%
\caption{Large Euclidean separation of yeast cell cycle dataset
by decreasing the scaling parameter during one and two dimensional
projection pursuit.}\label{fig:LMSCyeast}
\end{figure}

%A similar intuition which underlies the theoretical results of this section can be used to reason why this will occur in multivariate projections, in that as the scaling parameter reduces to zero the value of $\lambda_2(\Ltheta)$ is controlled by the smallest distance between observations in different elements of the induced partition. It is however elusive how to formulate this rigorously in the presence of the constraint set $\pmb{\Delta}$ in more than one dimension.
%\blue{I am not sure we can have this final paragraph without more of the supporting theory for the theorem. At least as it stands now it doesn't really make sense in the context.}

%%%%%%%%%%%%%%%%%%%%%%%%%%%%%%%%%%%%%%%%%%%%%%%%%%%%%%%%%%%%%%%%%%%%%%%%%%%%%%%%

\section{Speeding up Computation} \label{sec:microclust}

Each step in the projection pursuit algorithm involves the solution of an eigen
problem which requires $\mathcal{O}(N^2)$ operations.
In this section we discuss how preprocessing a dataset using
\emph{microclusters}~\citep{Zhang1996} can reduce this cost significantly, and
derive theoretical bounds on the approximation error. 
Microclusters are small clusters of data which can in turn be clustered to obtain
a complete clustering of a data set. A microcluster based approach 
to reduce the computational cost of the standard spectral clustering algorithm has been previously proposed
by~\citet{yan2009fast}. 
In this work
we use microclusters to obtain an approximation of the optimisation surface for projection pursuit
which is significantly less expensive to explore.

In the microcluster approach, the data set $\mathcal{X} = \{x_1, \ldots, x_N\}$ is
replaced by $m$ points $\{c_1, \ldots, c_m\}$ which represent the centres of a
$m$-way clustering of $\mathcal{X}$. By projecting these microcluster centres during
projection pursuit rather than the data the computational cost
associated with each eigen problem is reduced to
$\mathcal{O}(m^2)$. If we define the radius, $\rho$, of a cluster $C$ to be the
largest distance between any one of its members and its centre,
\begin{equation}
\rho(C) = \max_{x \in C}\left\|x - \frac{1}{\vert C \vert}\sum_{x \in C} x
\right\|,
\end{equation}
then we expect the approximation error to be small whenever the microcluster
radii are small. This relationship is shown in the following lemma. The proof of the
lemma, which is given in Appendix~\ref{sec:proofs}, relies on a result from matrix perturbation theory for diagonally dominant
matrices~\citep[Th. 3.3]{ye2009}

\begin{lem}\label{thm:approxbound1}
Let $\C = C_1, \ldots, C_m$ be a $m$-way clustering of $\X$ with centres $c_1, \ldots, c_m$,
radii $\rho_1, ..., \rho_m$ and counts $n_1, ..., n_m$.
For $\btheta \in \Theta$ define $N(\btheta), B(\btheta) \in \mathbb{R}^{m \times m}$
where $N(\btheta)$ is the diagonal matrix with,
\begin{equation*}
N(\btheta)_{i,i} = \sum_{j=1}^m n_j s(P^c(\btheta), i, j),
\end{equation*}
and
\begin{equation*}
B(\btheta)_{i,j} = \sqrt{n_in_j} s(P^c(\btheta), i, j),
\end{equation*}
where $P^c(\btheta) = \{V(\btheta)^\top c_1, ..., V(\btheta)^\top c_m\}$ are the projected microcluster
centres and the similarities are given by $s(P^c(\btheta), i, j) = k(d(V(\btheta)^\top c_i, V(\btheta)^\top c_j)/\sigma)$, and $k(x)$ is positive and non-increasing for $x\geq0$. Then,
\begin{align*}
\frac{\vert \lambda_2(L(\pmb{\theta})) - \lambda_2(N(\btheta)-B(\btheta)) \vert}{\lambda_2(L(\pmb{\theta}))}& \\
	\leqslant \max_{i \not = j} 
	\max \Bigg\{1-&\frac{k(D_{ij}/\sigma)}{k((D_{ij}-\rho_i-\rho_j)^+/\sigma)}, \\
      &\hspace{-10pt} \frac{k(D_{ij}/\sigma)}{k((D_{ij}+\rho_i+\rho_j)/\sigma)} - 1\Bigg\},
\end{align*}
where $D_{ij} = d(V(\btheta)^\top c_i,V(\btheta)^\top c_j)$ and $(x)^+ = \max\{0, x\}$.
\end{lem}

The bound in the above lemma depends on $\btheta$ via the quantity $D_{ij}$. Uniform bounds can be derived for specific functions, $k$. For example, if using the Gaussian kernel, $k = \exp(-x^2/2)$, then we can show that
\begin{align*}
&\frac{\vert \lambda_2(L(\pmb{\theta})) - \lambda_2(N(\btheta)-B(\btheta)) \vert}{\lambda_2(L(\pmb{\theta}))} \\
	&\hspace{10pt}\leqslant \max_{i \not = j} \exp\left(\frac{(\rho_i+\rho_j)^2+2(\rho_i+\rho_j)\mathrm{Diam}(\X)}{2\sigma^2}\right) - 1.
\end{align*}
If $k$ is the Laplace kernel, $k(x) = \exp(-\vert x\vert)$, then we instead have
\begin{align*}
&\frac{\vert \lambda_2(L(\pmb{\theta})) - \lambda_2(N(\btheta)-B(\btheta)) \vert}{\lambda_2(L(\pmb{\theta}))} \\
	&\hspace{120pt}\leqslant \max_{i \not = j}  \exp\left(\frac{\rho_i+\rho_j}{\sigma}\right) - 1.
\end{align*}
Clearly if the radii of the microclusters are small relative to the
scale parameter, $\sigma$, then these bounds are close to zero. However the uniform bounds are pessimistic,
and to obtain a reasonable bound on the approximation surface, as many
as $m \approx 0.6N$ might be needed, leading to only a threefold speed up. We
have observed empirically, however, that even for $m = 0.1N$ (and sometimes
lower) one still obtains a close approximation of the optimisation surface.
This renders the projection pursuit of the order of 100 times faster.

While bounds of the above type are not verifiable for $\LNtheta$ since this
matrix is not diagonally dominant, a similar degree of agreement
between the true and approximate eigenvalues has been observed.

Once an optimal projection has been determined, the corresponding bi-partition needs to be established. We again use the microclusters to determine this partition. Let $\mathcal{P}(\btheta)^\prime = \{V(\btheta)^\top c_1, V(\btheta)^\top c_1, \ldots, V(\btheta)^\top c_m, V(\btheta)^\top c_m\}$,
where each $V(\btheta)^\top c_i$ is repeated $n_i$ times. $\mathcal{P}(\btheta)^\prime$ therefore
represents an approximation of the projected data set, where each datum is replaced by its assigned microcluster.
It is straightforward to verify that if $u^C$ is the second eigenvector
of $N(\btheta) - B(\btheta)$, then the vector $u\in \R^N$, with
$u_i = u^C_j/\sqrt{n_j}$ for all $i$ s.t. $x_i$ is in microcluster $j$, is the second eigenvector
of the Laplacian of $\mathcal{P}(\btheta)^\prime$. The vector $u$ therefore represents
an approximation of the second eigenvector of $\Ltheta$.
In case of the normalised Laplacian the $m\times m$ matrix is given by the normalised
Laplacian of the graph of $\P^c(\btheta)$ with similarities given by
$n_i n_j s(P^c(\btheta), i, j)$. This matrix has the same structure as the
original normalised Laplacian, the only difference being the introduction of
the factors $n_i, n_j$.
The approximation of the second eigenvector of $\LNtheta$ is again given
by $u_i = u_j/\sqrt{n_j}$ whenever $x_i$ is in microcluster $j$.
This approximate eigenvector is then used to determined the partition of the data.

%%%%%%%%%%%%%%%%%%%%%%%%%%%%%%%%%%%%%%%%%%%%%%

%%%%%%%%%%%%%%%%%%%%%%%%%%%%%%%%%%%%%%%%%%%%%%%%%%%%%%%%%%%%%%%%%%%%%%%%%%%%%%%%

\section{Practical Implementation and Experimental Results} \label{sec:experiments}

We have found that projection pursuit based on both $\lambda_2(\Ltheta)$ and $\lambda_2(\LNtheta)$
leads to high quality clustering results. However, we have observed empirically
%our experience suggests
that the
minimisation of $\lambda_2(\LNtheta)$ is more robust to varying parameter settings,
and we recommend using this objective. Our complete clustering
algorithm, which we will refer to as Spectral Clustering Projection Pursuit (SCPP),
is summarised in Algorithm~\ref{alg:SCPP}\footnote{An {\tt R} implementation of the SCPP algorithm is available at \url{https://github.com/DavidHofmeyr/SCPP}}. Starting with all the data in a single cluster,
we recursively bi-partition the data until we have the desired number of clusters. At each iteration we simply split the largest cluster in the current partition.
To split a cluster, we first obtain $m$ microclusters from it, for which we use the $K$-means algorithm. We then apply Algorithm~\ref{alg:optimisation} to
obtain the optimal projection, $\btheta^\star$, based on Eq.~(\ref{eq:orthog}). Recall that the normalised Laplacian
based on (weighted) projected microcluster centers $\P^c(\btheta) = \{V(\btheta)^\top c_1, ..., V(\btheta)^\top c_m\}$ is given by $L_{\mathrm{N}}(\btheta) = D(\btheta)^{-1/2}L(\btheta)D(\btheta)^{-1/2} = I - D(\btheta)^{-1/2}A(\btheta)D(\btheta)^{-1/2}$, where
$A(\btheta)_{ij} = n_i n_j s(P^C(\btheta), i, j)$ and $D_{ii} = \sum_{j=1}^m A(\btheta)_{ij}$.
To obtain a bi-partition of the cluster we use the method recommended by~\citet{NgJW2002}.
For this we obtain the first two eigenvectors of $\Ln(\btheta^\star)$ as the matrix $U^c \in \R^{m \times 2}$.
From these we obtain the approximate eigenvectors of the Laplacian of the complete set of projected points
as the matrix $U \in \R^{N\times 2}$, with $i$-th row equal to the $j$-th row of $U^c$ divided by $\sqrt{n_j}$ for each $x_i$ in
microcluster $j$. We then normalise the rows of $U$ and apply $K$-means for $K=2$.
%We obtain the approximate second eigenvector of the Laplacian of the full set of
%projected points, $u$, from the second eigenvector of $L_{\mathrm{N}}(\btheta^\star)$, $u^C$. We then,
%normalise by the square root of the degree to obtain the approximate Normalised Cut solution.
%This is done by defining the vector $u$ to satisfy $u_i = u_j^C/\sqrt{D(\btheta^\star)_{jj}}$
%for all $i$ s.t. $x_i$ is in microcluster $j$. We then bi-partition the cluster by applying $k$-means to~$u$.
For the sake of easier interpretability we make our algorithm completely deterministic by initialising
all implementations of $K$-means as follows. We select the first center to be the point furthest from the mean of the
data. We then iteratively add to the set of initial centroids the furthest point from the current set.\\
\\
\begin{algorithm}[t!]
\caption{SCPP}
\label{alg:SCPP}
\begin{algorithmic}
\STATE Input: Dataset $\X$, number of clusters $K$
\STATE Output: Partition $\Pi$ of $\X$ into $K$ clusters
\STATE \# {\em Initialise $\Pi$ as the set containing $\X$}
\STATE $\Pi \gets \{\X\}$
\WHILE{$\vert \Pi \vert < K$}
	\STATE \# {\em Select the next cluster to split, $\C^\prime$}
	\STATE $\C^\prime \gets \mbox{argmax}_{\C \in \Pi}$ $\vert \C \vert$
	\STATE \# {\em Obtain centers and counts from microclustering of $\C^\prime$}
	\STATE $[\{c_1...c_m\}, \{n_1... n_m\}] \gets Microcluster(\C^\prime)$
	\STATE \# {\em Optimise projection for spectral clustering of $\P^c(\btheta)$}
	\STATE $\btheta^\star \gets \mathrm{argmin}_{\btheta} \lambda_2(\LNtheta) + \omega\sum_{i\not = j} (V(\btheta)_i^\top V(\btheta)_j )^2$
	\STATE \# {\em Find the first two eigenvectors of $\Ln(\btheta^\star)$}
	\STATE $U^c \gets \mathrm{argmin}_U \mathrm{trace}(U^\top L_{\mathrm{N}}(\btheta^\star) U)$ s.t. $U^\top U = I$
	\STATE \# {\em Get approximate eigenvectors of Laplacian of $V(\btheta^\star)^\top \C^\prime$}
	\STATE $U \gets U_i = U_j^c/\sqrt{n_{j}} \iff x_i \in \mathrm{microcluster} \ j$
	\STATE \# {\em Normalise the rows of $U$}
	\STATE $U_i \gets U_i/\|U_i\|, \ \forall \ i = 1, \ldots, N$
	\STATE \# {\em Bi-partition rows of $U$ using $k$-means}
	\STATE $[\mathcal{U}_1, \mathcal{U}_2] \gets K\mathrm{-means}(U, 2)$
	\STATE \# {\em Obtain corresponding split of $\C^\prime$}
	\STATE $\C_1 \gets \cup_{i: U_i \in \mathcal{U}_1} \{x_i\}, \C_2 \gets \cup_{i: U_i \in \mathcal{U}_2} \{x_i\}$
	\STATE \# {\em Update overall partition $\Pi$}
	\STATE $\Pi\gets (\Pi\setminus \{\C^\prime\})\cup \{\C_1, \C_2\}$
\ENDWHILE
\STATE {\bf return} $\Pi$
\end{algorithmic}
\end{algorithm}
The clustering model obtained by the SCPP algorithm has a binary tree structure, as illustrated in Figure~\ref{fig:bi-tree}. The figure shows a divisive hierarchical clustering of the 256 dimensional phoneme dataset~\citep{HastieTF2009}.  Each scatter plot shows the data assigned to the corresponding node in the model projected into the optimal subspace based on the minimisation of the second eigenvalue of the Laplacian matrix. In Figure~\ref{fig:bi-tree-nolabel} the colours indicate the binary partitions made by the SCPP algorithm, while in Figure~\ref{fig:bi-tree-label} the colours show the true cluster labels of the data. The model has accurately partitioned the clusters; indicated by the fact that the leaf nodes each contain primarily data of a single cluster, and aside from the two clusters arising in the bottom most level in the hierarchy no cluster is split among multiple leaves.
\begin{figure}
\centering
\subfigure[Without cluster labels]{\includegraphics[width = 7cm]{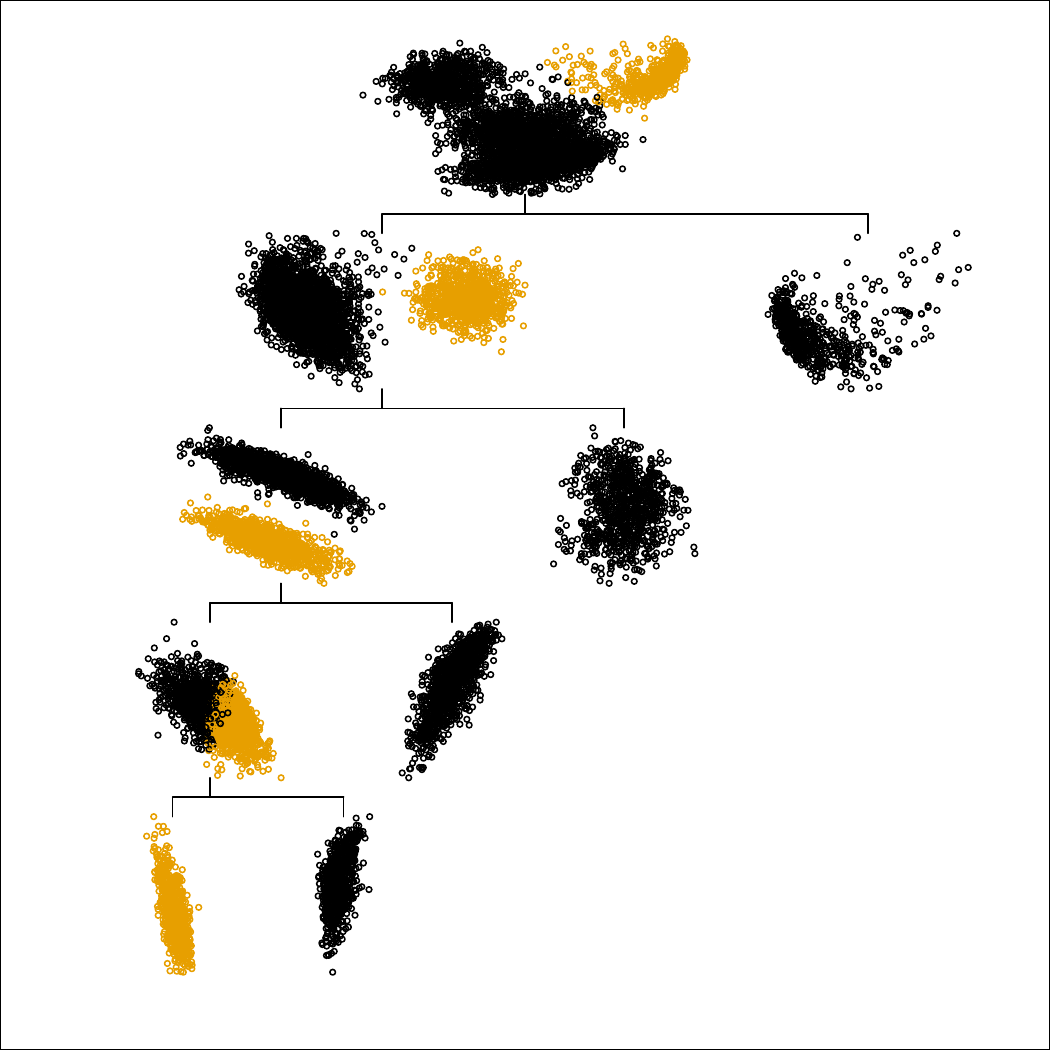}\label{fig:bi-tree-nolabel}}
\subfigure[With true cluster labels]{\includegraphics[width = 7cm]{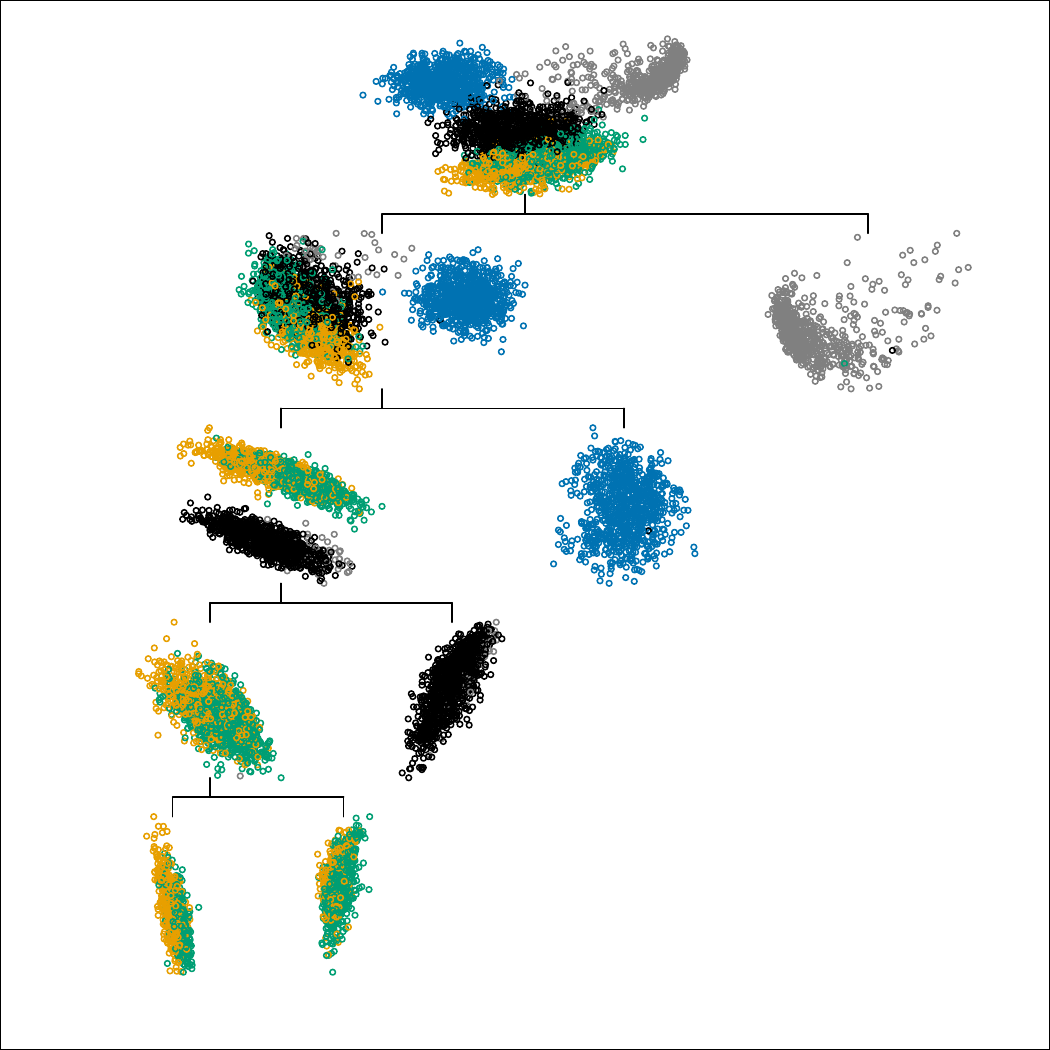}\label{fig:bi-tree-label}}
\caption{Hierarchical clustering model obtained by SCPP on phoneme dataset \label{fig:bi-tree}}
\end{figure}
\subsection{Parameter Settings for SCPP}
For the experiments herein, we use the following settings. In all cases the data dependent settings are determined for each partition using the subset of the data being split. We set $l$, the dimension of the projection to 2 as this is the lowest number of dimensions which admits non-linear separation of clusters. We initialise the projection pursuit using the first two principal components. We have found that this often leads to higher quality solutions compared to random initialisations. Experiments with higher dimensional projections have not shown substantially improved performance.
Similarities between projected points are determined using the Gaussian kernel. The scale parameter, $\sigma$, is set as follows. We approximate $d^*$, the intrinsic dimensionality of the data, using Kaiser's criterion~\citep{kaiser1960application}. We then set $\sigma = \sqrt{\bar \lambda}\left(\frac{4}{3N}\right)^{\frac{1}{4+d^*}}$, where
$\bar \lambda$ is the average of the first $d^*$ eigenvalues of the covariance matrix of the data. The factor $\sqrt{\bar\lambda}$ captures the scale of the data, while $\left(\frac{4}{3N}\right)^{\frac{1}{4+d^*}}$ is borrowed from kernel density bandwidth estimation, and we have found it to work well for our problem as well.

Recall that we use $\pmb{\Delta}(\btheta)$ to mitigate the influence of outliers. We define $\pmb{\Delta}(\btheta) = \Delta_1 \times \dots \times \Delta_l$,
 where $\Delta_i = [\mu_i-\beta
\sigma_{i}, \mu_i + \beta \sigma_{i}]$; $\mu_i$ and $\sigma_{i}$ are the mean
and standard deviation of the $i$-th component of the projected data respectively; and
$\beta \geqslant 0$ controls the size of
$\pmb{\Delta}(\btheta)$. Rather than attempting to define a single value of $\beta$ which
is appropriate for all datasets, we initialise $\beta$ to a large value,
$\beta=5$, and decrease $\beta$ until the induced bi-partition is sufficiently balanced. For this we define a minimum
cluster size, the average cluster size in the complete clustering solution divided by 5. That is, we decrease $\beta$ until the smaller of the two clusters contains at least $\frac{N}{5K}$ points, where $N$ is the number of data in the complete dataset being clustered.
Note that in general we do not have to execute the optimisation of $\btheta$ to convergence for each value of $\beta$, since a few
iterations generally suffice to determine if the optimisation is focusing on outliers. We therefore terminate the optimisation
as soon as the induced partition does not meet the desired balance, reduce $\beta$, and reinitialise.

The setting of the parameter $\omega$, which controls the penalisation of non-orthogonal projections, does not affect the result substantially provided it is relatively larger than the eigenvalues being optimised. Since $\lambda_2(\LNtheta)$ is bounded above by 1, we set $\omega = 1$.

Finally, for our experiments we use a small number of microclusters, $m = 200$. A sensitivity study presented in Section~\ref{sec:sensitivity} using simulated data shows that even for data sets of up to 10 000
points and in 50 dimensions, 200 microclusters are sufficient to obtain high quality clustering results.

\subsection{Competing Approaches}

We compare our approach against existing dimension reduction
methods for clustering, where the final clustering result is determined using
spectral clustering. We use SC to refer to spectral clustering
applied to the original data, and SC$_{\mathrm{PC}}$ and SC$_{\mathrm{IC}}$ to refer
to spectral clustering applied to Principal and Independent Component
projections of the data respectively. DRSC refers to dimensionality reduction for spectral clustering, proposed by~\citet{niu2011dimensionality}. For SC$_{\mathrm{PC}}$, SC$_{\mathrm{IC}}$ and DRSC we consider $K-1$ dimensional
projections, as suggested by~\citet{niu2011dimensionality}. These approaches all directly seek a $K$ way partition of the data.

For these competing approaches we compute clustering results for all values of $\sigma$ in $\{0.1, 0.2, 0.5, 1, 2, 5, 10, 20, 50,$ $100, 200\}$, and select the solution which gives the lowest cluster distortion measure. This selection criterion is recommended by~\citet{NgJW2002} and~\citet{niu2011dimensionality}. We also compute the clustering result for the local scaling approach of~\citet{Zelnik2004}. We report the highest performance of these two in each case. We also provide DRSC with a warm start via PCA as this improved performance over a random initialisation, and provides a fair comparison.

The connection between optimal projections for spectral clustering
and maximum margin clustering, established in Section~\ref{sec:maxmargin}, also
leads us to compare our method with the iterative support vector regression approach
of~\citet{zhang2009maximum},
%\footnote{We are grateful to Dr. Kai Zhang for supplying us with code to
%implement this method.},
%
a state-of-the-art maximum margin clustering algorithm. We use iSVR$_{\mathrm{G}}$ to refer to
this method, where the subscript G indicates that we use the Gaussian kernel. We set the balancing parameter equal to $0.3$ as suggested by~\citet{zhang2009maximum} when the cluster sizes are not balanced. 
The unbalanced setting led to superior performance compared with the balanced setting in the examples considered. The iSVR approach generates only a bi-partition, and to generate multiple clusters we apply the same divisive approach as in our method.

\subsection{Clustering Results}

We compare the different methods based on two popular evaluation metrics for clustering, namely Purity~\citep{zhao2004empirical},  and Normalised Mutual Information (NMI)~\citep{strehl2002cluster}. These metrics compare the cluster assignments with the true labels of the data. Both take values in $[0, 1]$, with larger values indicating better performance.

The following benchmark datasets were used for comparison.
Optical recognition of handwritten digits (Opt. Digits)\footnote{\tt https://archive.ics.uci.edu/ml/datasets.html}, Pen based recognition of handwritten digits (Pen Digits)$^1$, Multiple feature digits (M.F. Digits)$^1$, Satellite$^1$, Statlog image segmentation (Image Seg.)$^1$, Breast cancer Wisconsin (Br. Cancer)$^1$, Synthetic control chart (Chart)$^1$, Isolet$^1$, Dermatology$^1$, Yeast cell cycle analysis (Yeast)\footnote{\tt http://genome-www.stanford.edu/cellcycle/}, Smartphone based activity recognition (Smartphone)$^1$,  Yale faces dataset B 30 $\times$ 40 (Faces)\footnote{\tt https://cervisia.org/machine\_learning\_data.php/}, Phoneme\footnote{\tt http://statweb.stanford.edu/$\sim$tibs/ElemStatLearn/}.
\noindent
Before applying the clustering algorithms, data were rescaled so that every feature had unit variance. %This is a standard approach to handle situations where different features are captured on different scales and an appropriate rescaling is not obviously apparent from the context. For consistency we used this same preprocessing approach for all datasets.
\setlength{\tabcolsep}{.15cm}

Clustering results for all methods considered are given in Table~\ref{tb:clustresults}. SCPP achieves the highest performance in more than half the cases considered, and very importantly is competitive with
the best performing method in every case. All other methods achieve substantially lower performance than SCPP in
multiple examples.

The vastly different natures of the datasets considered means that the associated clustering tasks differ in difficulty. This is evidenced by the range of performance values achieved by the clustering algorithms on different datasets. To combine the results from the different datasets we standardise them as follows. For each dataset $\X$ we compute for each method the relative deviation from the average performance of all methods when applied to $\X$. That is, for each method, $M_i$, we compute the relative purity,
\begin{equation}
\frac{\mathrm{Purity}(M_i, \X) - \frac{1}{\# \mathrm{Methods}}\sum_{j=1}^{\# \mathrm{Methods}}\mathrm{Purity}(M_j, \X)}{\frac{1}{\# \mathrm{Methods}}\sum_{j=1}^{\# \mathrm{Methods}}\mathrm{Purity}(M_j, \X)},
\end{equation}
and similarly for NMI. We can then compare the distributions of the relative performance measures from all datasets and for all methods. It is clear from Table~\ref{tb:clustresults} that the DRSC method is not competitive with other methods
in the examples considered, due to its substantially inferior performance on multiple datasets. Moreover, the performance of DRSC is sufficiently low to obscure the comparisons between other methods. We therefore remove DRSC from this comparison and in computing the relative performance measures.
Figure~\ref{fig:relmeasure} shows boxplots of the relative performance measures. These plots show clearly that SCPP achieves substantially higher performance overall than all other methods considered.

\begin{table*}
\begin{center}
\caption{Clustering performance. Highest performance in each case is highlighted in bold. Details of datasets in terms of number of data (N), number of dimensions (d), and number of clusters (K) are provided. \label{tb:clustresults}}
\scalebox{1}{
\begin{tabular}{ll|cccccccc}
&& SCPP & DRSC & SC$_{\mathrm{PC}}$  &SC$_{\mathrm{IC}}$ & SC & iSVR$_{\mathrm{G}}$\\
\hline\hline
Opt. Digits 
							& Purity 		& {\bf 0.89}  & 0.10 & 0.66 &  0.69  &  0.66   & 0.73   \\
(N = 5620, d = 64, K = 10)	& NMI 		& {\bf 0.83} & 0.03     & 0.63 &   0.67     & 0.63 & 0.65\\
\hline
Pen Digits 
							& Purity 		& 0.81  & 0.44  & 0.77 & 0.77  & {\bf 0.87} & 0.74\\
(N = 10992, d = 16, K = 10)	& NMI 		& 0.79  & 0.41 & 0.76 & 0.75  & {\bf 0.82} & 0.68\\
\hline
M.F. Digits 
							& Purity 		& 0.76  & 0.66  & 0.75 & 0.72  & 0.77 & {\bf 0.78}\\
(N = 2000, d = 216, K = 10)	& NMI 		& {\bf 0.73}  & 0.67  & 0.70 & 0.68 & 0.72 & 0.65\\
\hline
Satellite 
							& Purity 		& {\bf 0.80}  & 0.53 & 0.73 & 0.74  &  {\bf 0.76} & 0.61\\
(N = 6435, d = 36, K = 6)		& NMI 		& {\bf 0.67}  & 0.22 & 0.61 & 0.62  &  0.62 & 0.48\\
\hline
Image Seg. 
							& Purity 		& 0.56  & 0.38 & 0.56 & {\bf 0.76}  & 0.50 & 0.64\\
(N = 2310, d = 19, K = 7)		& NMI 		& 0.56  & 0.40 & 0.55 & {\bf 0.69}   & 0.48 & 0.59\\
\hline
Br. Cancer
							& Purity 		& {\bf 0.97}  &  0.89  & {\bf 0.97}  & {\bf 0.97} &   0.96 & 0.95\\
(N = 699, d = 9, K = 2)		& NMI 		& 0.78  & 0.51 & 0.81 & {\bf 0.82}  & 0.76 & 0.72\\
\hline
Chart 
							& Purity 		& {\bf 0.89}  & 0.24  & 0.67 & 0.73  & 0.67 & 0.80\\
(N = 600, d = 60, K = 6)		& NMI 		& {\bf 0.87}  & 0.01 & 0.81 & 0.76  & 0.74 & 0.72\\
\hline
Isolet 
							& Purity 		& 0.58  & - & 0.59 & {\bf 0.60} &  {\bf 0.60} & 0.50\\
(N = 6238, d = 617, K = 26)	& NMI 		& {\bf 0.72}  & - & 0.69 & 0.67  & 0.69 & 0.61\\
\hline
Dermatology 
							& Purity 		& 0.87 & 0.59 & 0.92 & 0.91 &  {\bf 0.95} & 0.82   \\
(N = 366, d = 34, K = 6)		& NMI 		& 0.90  & 0.40 & 0.87 & 0.83 & {\bf 0.91} & 0.78\\
\hline
Yeast 
							& Purity		& 0.73  & 0.42 & 0.68 & 0.60  & {\bf 0.78} & 0.76\\
(N = 698, d = 72, K = 5)		& NMI 		& 0.53  & 0.05 & 0.51 & 0.34  & {\bf 0.57} & {\bf 0.57}\\
\hline
Smartphone 
							& Purity 		& {\bf 0.70}  & - & 0.61 & {\bf 0.70} &  0.67 & 0.65\\
(N = 10929, d = 561, K = 12)	& NMI 		& {\bf 0.61} & - & 0.52 & 0.58 & 0.55 & 0.52 \\
\hline
Faces 
							& Purity 		& 0.71   & - & 0.68 & 0.69 &  0.73 & 0.63 \\
(N = 5850, d = 1200, K = 10)	& NMI 		& 0.76   & - & 0.77 & {\bf 0.82} & 0.76 & 0.64 \\
\hline
Phoneme 
							& Purity 		& {\bf 0.85}   & 0.56 & 0.83  & 0.84 & 0.80 & 0.82 \\
(N = 4509, d = 256, K = 5)	& NMI 		& 0.82   & 0.45 & {\bf 0.84} & 0.76 & 0.71 & 0.70 \\
\end{tabular}
}
\end{center}
{\footnotesize `-' indicates that a clustering solution could not be obtained in a reasonable amount of time.}
\end{table*}

Among the competing methods, it is evident that spectral clustering tends to outperform maximum margin clustering in general. Among competing spectral clustering variants, we see that both principal and independent component projections are capable of improving the performance of spectral clustering, but across multiple datasets the overall
performance is not appreciably higher.

Overall the proposed approach for projection pursuit based on spectral connectivity is highly competitive with existing dimension reduction methods. Furthermore, a simple data driven heuristic can be used to select the important scaling parameter without tuning it for each dataset.

\begin{figure}
\centering
\caption{Box plots of relative performance measures with additional red dots to indicate means. \label{fig:relmeasure}}
\subfigure[Relative Purity]{\includegraphics[width = 7cm, height = 7cm]{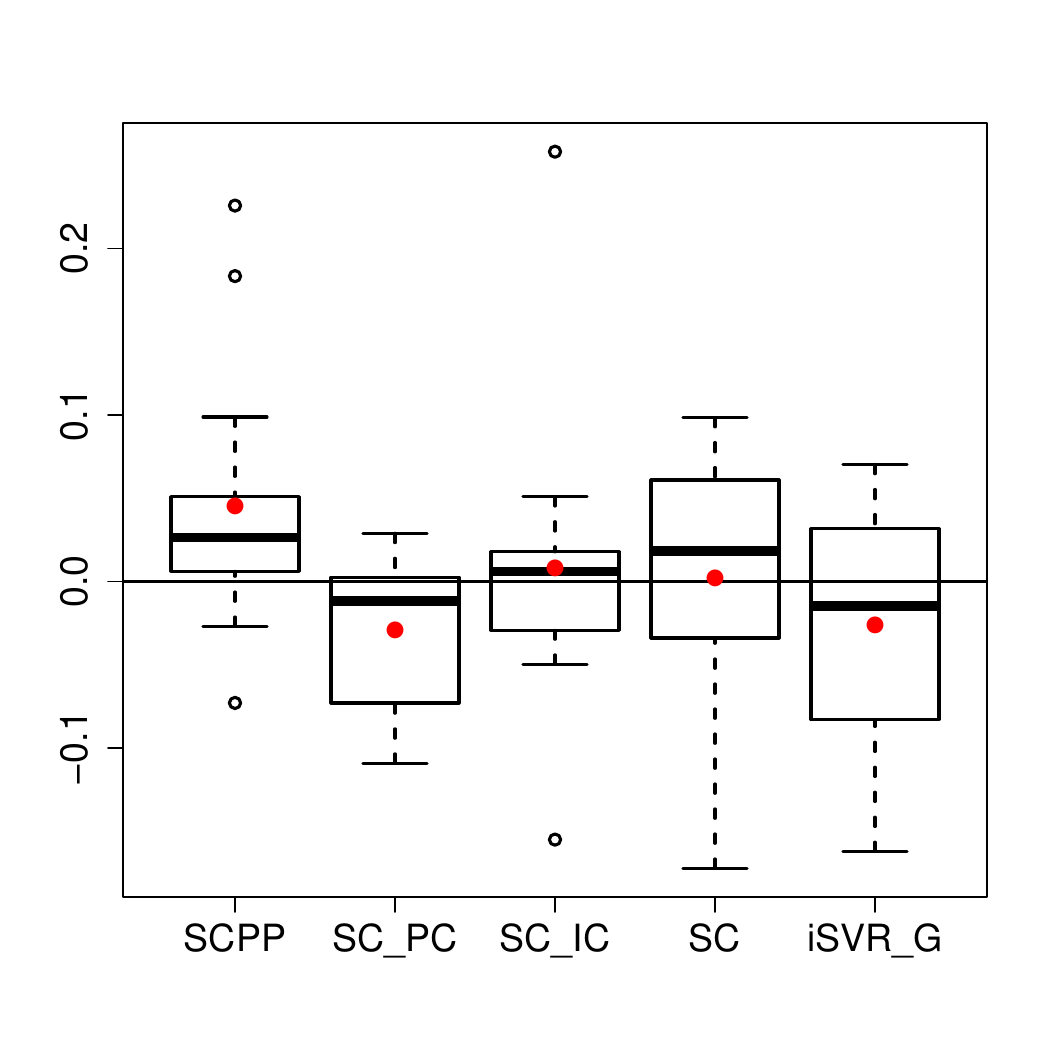}}
\subfigure[Relative NMI]{\includegraphics[width = 7cm, height = 7cm]{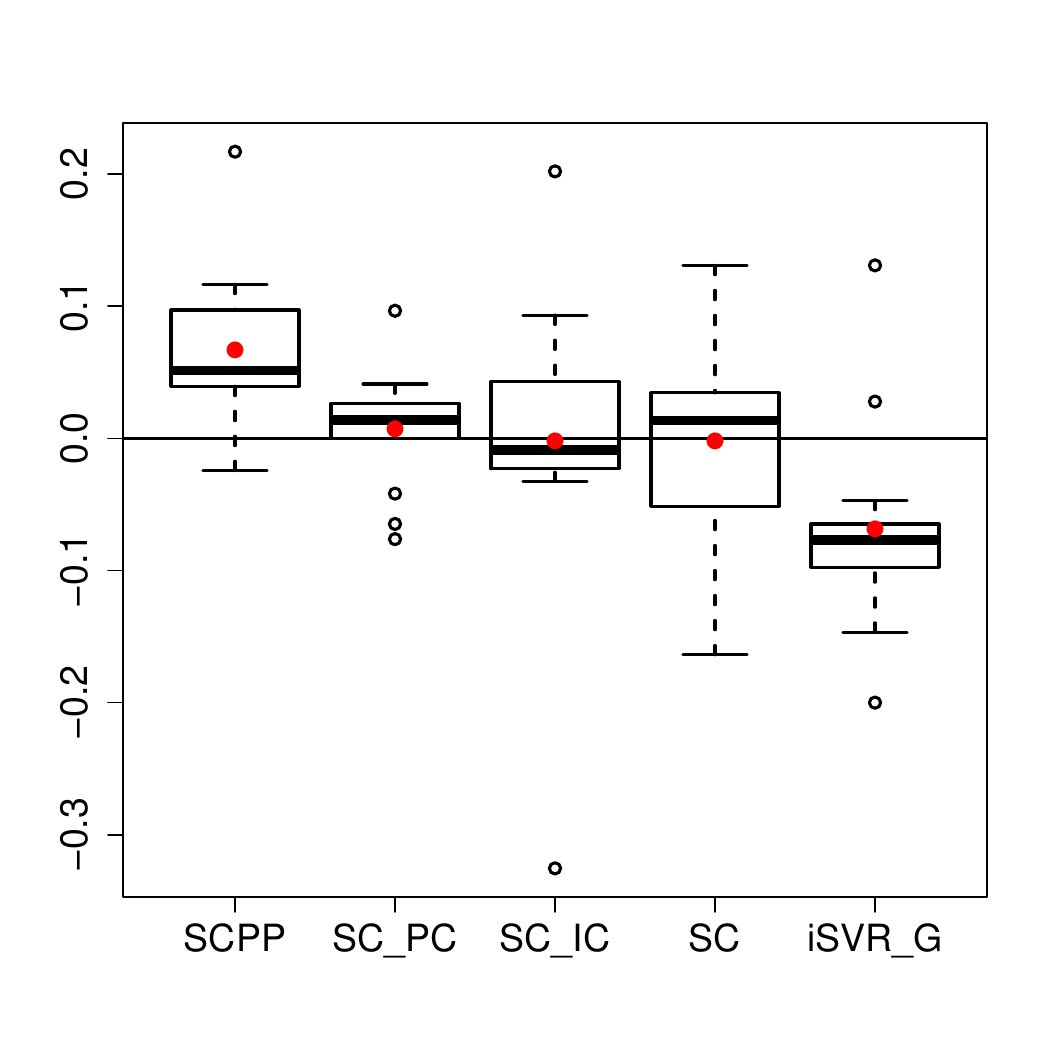}}
\end{figure}

\subsection{The Effect of Microclusters on Performance} \label{sec:sensitivity}

To investigate the effect of microclusters on clustering accuracy we simulated datasets from Gaussian mixtures containing 5 components (clusters) in 50 dimensions. This allows us to generate datasets of any desired size.
For these experiments 30 sets of parameters for the Gaussian mixtures were generated randomly. In the first case a single dataset of size 1000 was simulated from each set of parameters, and clustering solutions obtained for a number of microclusters, $m$, ranging from 100 to 1000, the final value therefore applying no approximation. Figure~\ref{fig:microcluster1} shows the median and interquartile range of both performance measures for 10 values of $m$. It is evident that aside from $m=100$, performance is similar for all other values, and so using a small value, say $m = 200$, should be sufficient to obtain a good approximation of the underlying optimisation surface.

In the second case, we fix the number of microclusters, $m = 200$, and for each set of parameters simulate datasets with between 1000 and 10 000 observations. In the most extreme case, therefore, the number of microclusters is only 2\% of the total number of data. Figure~\ref{fig:microcluster2} shows the corresponding performance plots, again containing the medians and interquartile ranges. Even for datasets of size 10 000, the coarse approximation of the dataset through 200 microclusters is sufficient to obtain a high quality projection using the proposed approach.

\begin{figure}
\caption{Effect of microclusters on performance. Plots show median and interquartile ranges of performance measures from 30 datasets simulated from 50 dimensional Gaussian mixtures with 5 clusters.}
\subfigure[Fixed number of data (1000) and varying number of microclusters, $m$]{\includegraphics[width = 8cm]{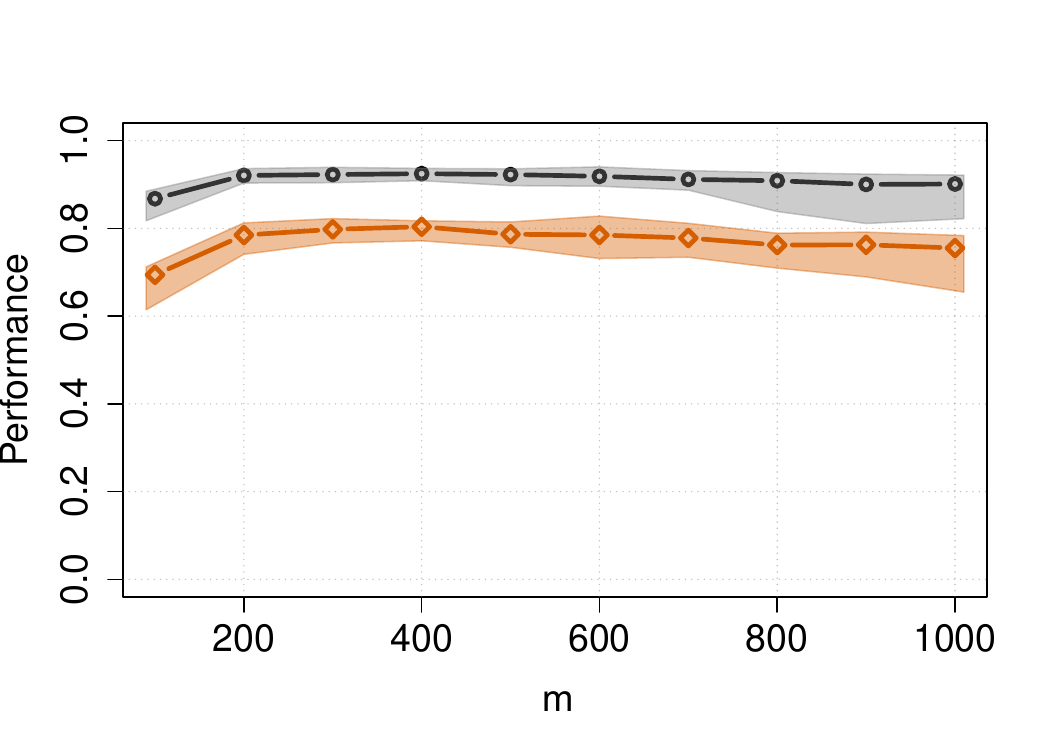}\label{fig:microcluster1}}
\subfigure[Fixed number of microclusters (200) and varying number of data, $N$]{\includegraphics[width = 8cm]{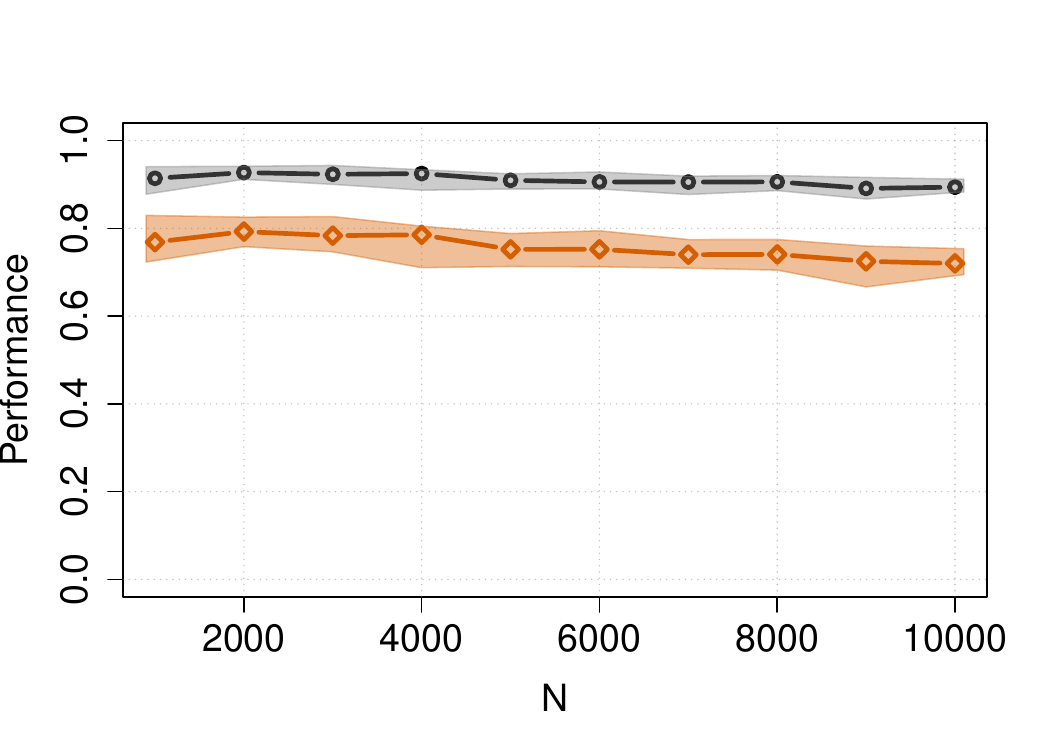}\label{fig:microcluster2}}\\
\centering
{\footnotesize
Purity (\textcolor{yeastcol}{\bf --$\circ$--}), NMI (\textcolor{synthcol}{\bf --$\diamond$--})
}
\end{figure}

%%%%%%%%%%%%%%%%%%%%%%%%%%%%%%%%%%%%%%%%%%%%%%%%%%%%%%%%%

%%%%%%%%%%%%%%%%%%%%%%%%%%%%%%%%%%%%%%%%%%%%%%%%%%%%%%%%%%%%%%%%

%%%%%%%%%%%%%%%%%%%%%%%%%%%%%%%%%%%%%%%%%%%%%%%%%%%%%%%%%%%%%%%%%%%%%%%%%%%%%%%%

\section{Conclusions}\label{sec:conclusion}

We proposed an approach to identify optimal projections to bi-partition a
dataset through spectral clustering, based on the minimisation of the second smallest
eigenvalue of the graph Laplacian (which measures the connectivity of the two clusters) with respect to the projection.
We provided a rigorous analysis of this optimisation problem
and proposed a globally convergent
algorithm, which directly minimises the overall
objective.
Using this approach to perform binary partitioning recursively gives rise to a
divisive clustering algorithm capable of identifying clusters defined in
different subspaces. 

The computational cost of the proposed projection pursuit method per iteration
is~$\mathcal{O}(N^2)$, where~$N$ is the number of observations, which can become prohibitive for large datasets. To mitigate
this an approximation method using microclusters, with provable error bounds is
proposed. This reduces the complexity to $\mathcal{O}(m^2)$, where $m$ is the number
of microclusters. We found that in practice using even a small number of microclusters, $m = 200$,
our method is capable of generating high quality clustering models. This results in
a speed up of up to two orders of magnitude for the examples considered in this paper.

Finally, we established an asymptotic connection between optimal univariate
projections for spectral bi-partitioning and maximum margin hyperplanes. In
particular we showed that as the scaling parameter of the similarity function
is reduced towards zero, the optimal vector to bi-partition the data using spectral
clustering also achieves the maximum Euclidean distance between the two
clusters. In other words, the optimal projection vector for spectral
bi-partitioning converges to the normal vector to the maximum margin separating
hyperplane.

Experimental results on a large collection of datasets indicate that the
proposed approach is highly competitive with spectral clustering
applied on the full dimensional data, and with existing dimension reduction
methods for spectral clustering.

It is interesting to note that while we discuss only the linear projection of
Euclidean embedded data, the methodology we present
can be
generalised to apply to any differentiable transformation of a collection of data objects
admitting a similarity measure. Extensions to structured data such as time series, graphical
and image data represent interesting future directions for this work.

\appendix

\section{Avoiding Outliers}\label{sec:sim}

It has been documented that spectral clustering can be sensitive to
outliers~\citep{rahimi2004clustering}. Our experience has shown that this problem
becomes more pronounced when performing dimension reduction based on the spectral clustering
objective, especially in high dimensional applications. Consider the extreme
case where $d>N$: since the linear system $V^\top X = P$ is underdetermined,
for any $P$ there exists $\btheta \in \Theta, c \in
\R\setminus\{0\}$ s.t. $V(\btheta)^\top X = cP$. The projected
data can therefore be made to have {\em any} distribution (up to a scaling constant).
In other words there will always be projections that contain outliers. 
We have found that even in problems of moderate dimensionality, there
often exist projections which induce large separation of a
small group of points from the remainder of the data.
These projections frequently achieve the minimum spectral connectivity
for both Ratio Cut and Normalised Cut.

We have found that by defining a metric
which encourages the induced cluster boundaries to intersect
a compact set, $\pmb{\Delta}(\btheta)$, around the mean of the projected data,
the problem of outliers can be mitigated.
This is achieved by reducing the distance, relative to the usual Euclidean metric,
to points lying outside $\pmb{\Delta}(\btheta)$. Points
lying outside $\pmb{\Delta}(\btheta)$, which may be outliers, therefore
have increased similarity to all others.
We define
$\pmb{\Delta}(\btheta) = \Delta_1 \times \ldots \times \Delta_l$,
 where $\Delta_i = [\mu_i-\beta
\sigma_{i}, \mu_i + \beta \sigma_{i}]$; $\mu_i$ and $\sigma_{i}$ are the
mean and standard deviation of the $i$-th component of the projected data; and
$\beta \geqslant 0$ controls the size of
$\pmb{\Delta}(\btheta)$.
The modified distance metric, $d(\cdot, \cdot)$,
is defined with respect to a continuously differentiable transformation, $T_{\Delta}$, of the projected data,
\medmuskip = 1.5mu \thickmuskip = 1.5mu
\begin{align}
d(p_i, p_j) & = \|T_{\Delta}(p_i) - T_{\Delta}(p_j)\|_2,\\
T_{\Delta}(y) &= \left(t_{\Delta_1}(y_1), \ldots, t_{\Delta_l}(y_l)\right),\label{eq:Ttransform}\\
t_{\Delta_i}(z)&:= \left\{\begin{array}{ll}
c_2 -\beta \sigma_i -\delta \left(c_1 -\beta\sigma_i - z \right)^{1-\delta}, & z < -\beta \sigma_i\\ 
z, & z \in \Delta_i\\
\beta \sigma_i  + \delta \left(z - \beta \sigma_i + c_1 \right)^{1-\delta} - c_2, & z > \beta \sigma_i,
\end{array}\right.  \end{align}
where $\delta \in (0, 0.5]$ is the distance reducing parameter,
and $c_1$ and $c_2$ are equalt to $\left(\delta\left(1-\delta\right)\right)^{1/\delta}$ and $\delta c_1^{1-\delta}$ respectively.
By construction $\|T_{\Delta}(p_i) - T_{\Delta}(p_j)\|_2 \leq \|p_i - p_j\|_2$
for any $p_i,p_j \in \R^l$, with strict inequality when either or both $p_i,p_j \notin \pmb{\Delta}(\btheta)$.

Figure~\ref{fig:retardeddistance} illustrates
the impact of $T_{\Delta}$ on pairwise distances in the univariate case.
As shown, distance increases linearly in the interval $\Delta$, but
outside $\Delta$ it increases
much more slowly, with the rate being determined by $\delta$.
In the limit as $\delta$ approaches zero, all points outside $\Delta$
are mapped to the boundary of~$\Delta$.
As a result distances between points outside $\Delta$ and all other
points are much smaller after being transformed through $T_\Delta$, and points which can
be characterised as outliers
in terms of the original projections,~$\P$, do not appear as such in
terms of $T_\Delta(\P)$.

\begin{figure}[h]
\includegraphics[width = 6cm]{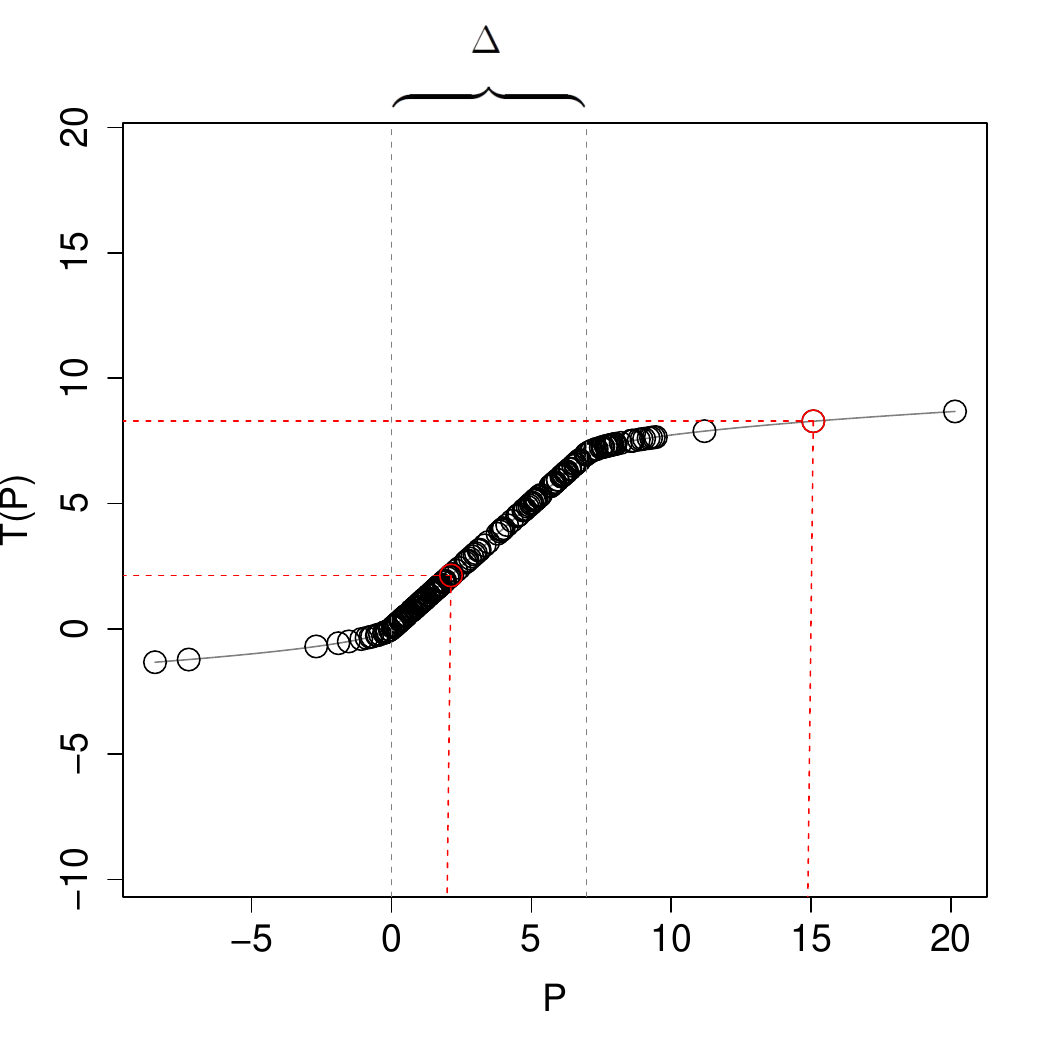}
\caption{Pairwise distances of points outside $\Delta$ are decreased through the transformation $T_{\Delta}$}
\label{fig:retardeddistance}
\end{figure}

An illustration of the usefulness of this modified metric is provided in Figure~\ref{fig:optidigitsprojections}.
The figure shows two dimensional projections of the
64~dimensional optical recognition of handwritten digits dataset~\citep{BacheL2013}.
The left plots show the true clusters while the right plots show the clustering assignments based on
spectral clustering using the normalised Laplacian~\citep{Shi2000}.
Figure~\ref{fig:optidigitsPCA} shows the projection onto the first two principal components, which are also used as initialisation for
our method. There are clearly a few points outlying from the remainder of the data, which are
separated by the spectral clustering algorithm.
Figure~\ref{fig:optidigitsBetamax} shows the optimal projection from
minimising $\lambda_2(\LNtheta)$ using the Euclidean metric. The
result is that the outlying points have been further separated from the
remainder of the data, thereby exacerbating the outlier problem.
Finally, Figure~\ref{fig:optidigitsBeta3} shows the same result but using
the modified metric discussed above, and with $\beta = 3$. In this case
the projection pursuit is able to find a projection which separates two of
the true clusters
clearly from the remainder.

\begin{figure}
\caption{Two dimensional projections of optical recognition of handwritten digits dataset. The left
plots show the true clusters while the right plots show the partitions made by spectral clustering. \label{fig:optidigitsprojections}}
\subfigure[PCA projection used for initialisation]{\includegraphics[width = 3.8cm]{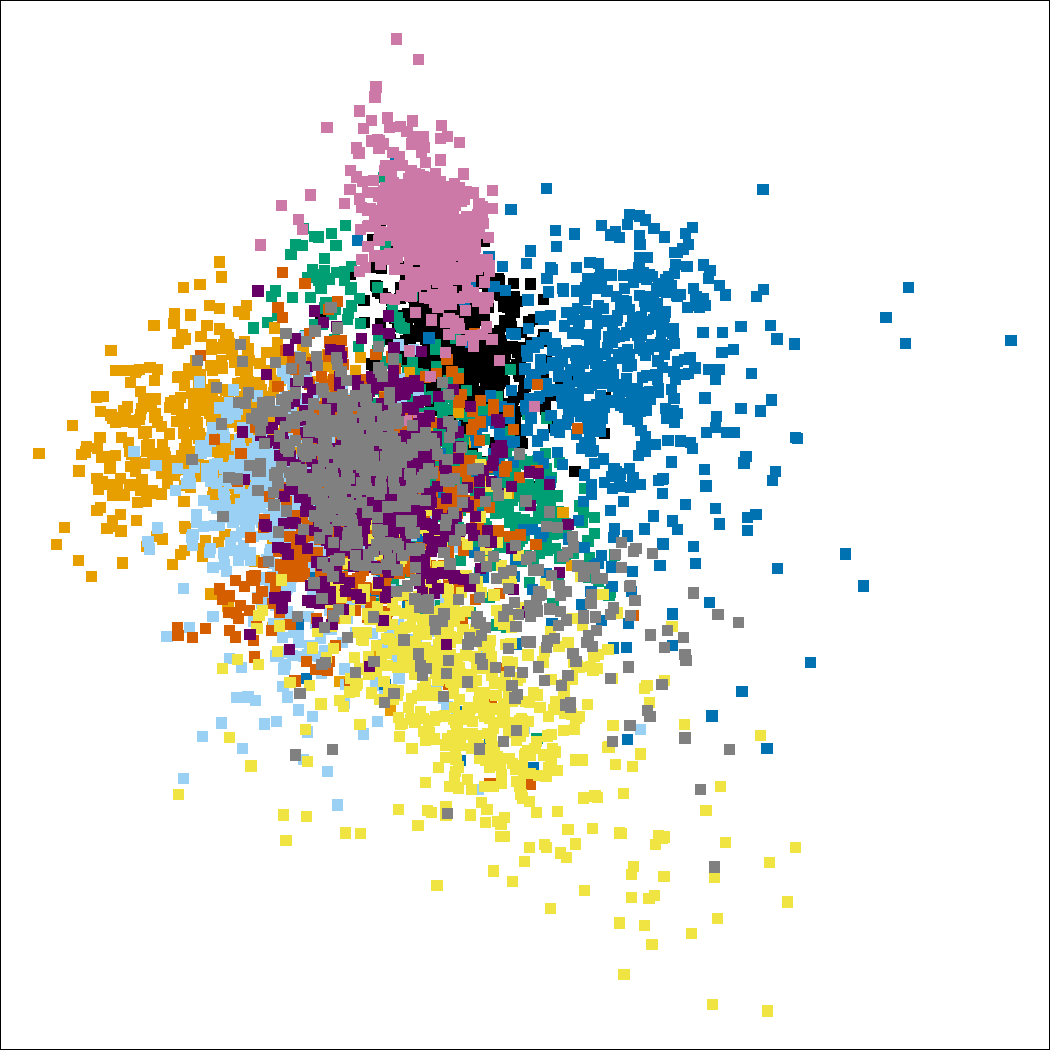}
 \includegraphics[width = 3.8cm]{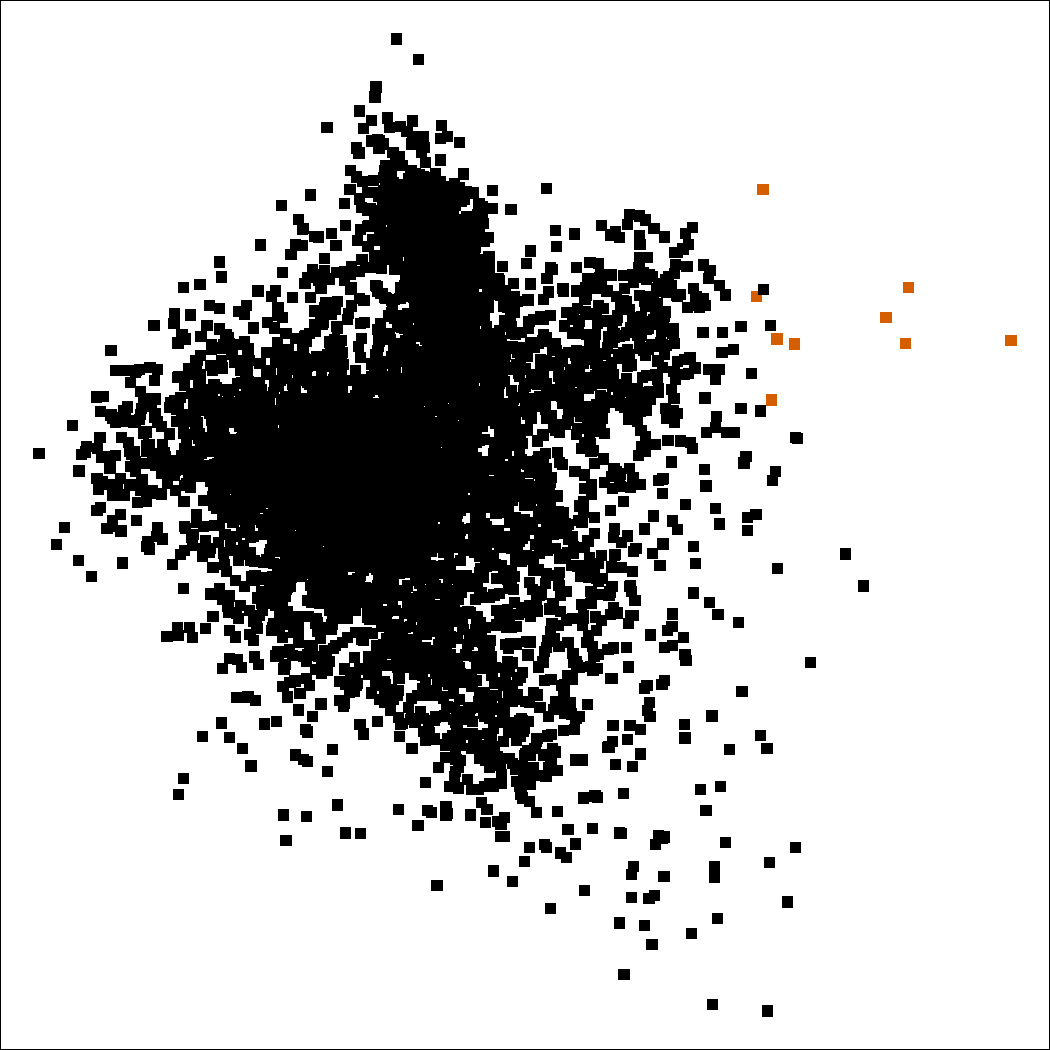} \label{fig:optidigitsPCA}}
\subfigure[Optimal projection from minimising $\lambda_2(\LNtheta)$ with the Euclidean metric]{\includegraphics[width = 3.8cm]{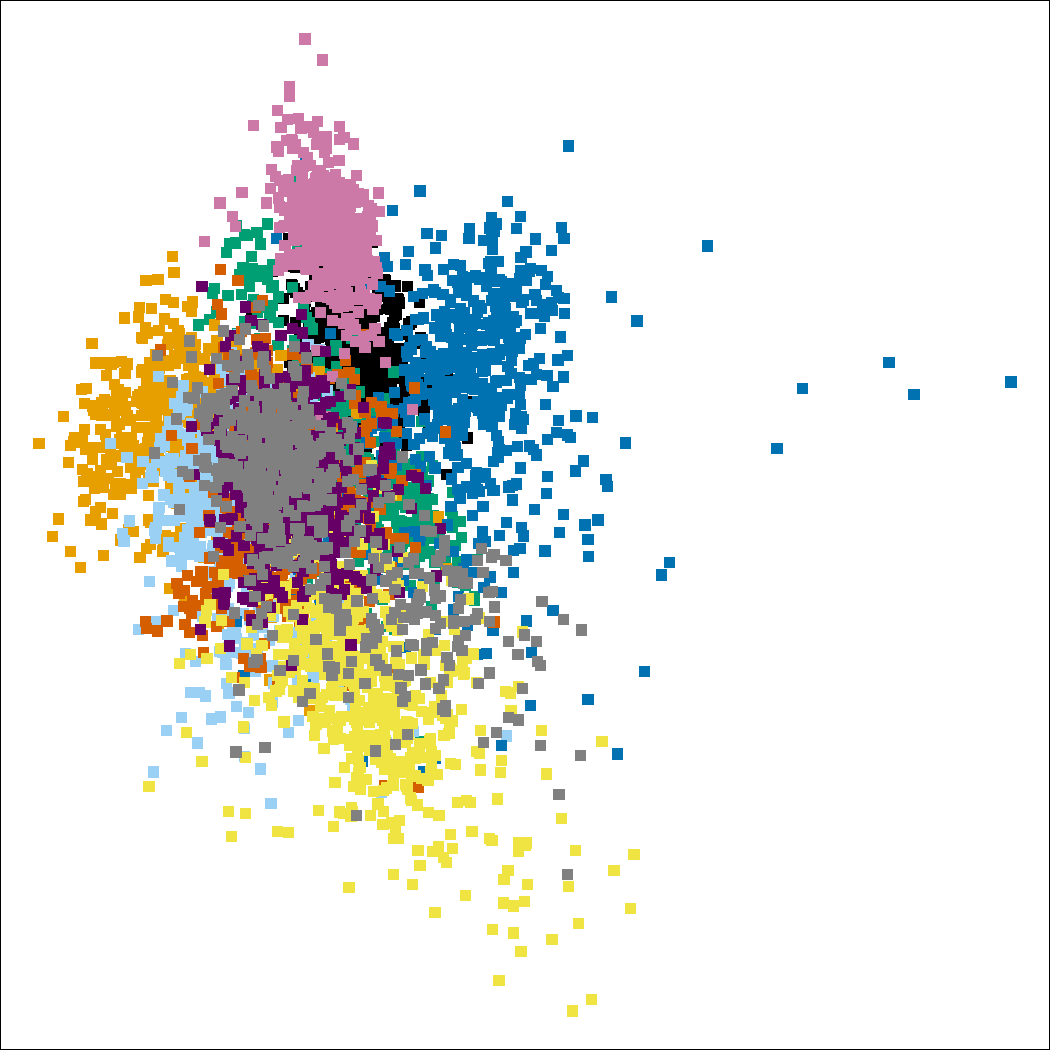}
\includegraphics[width = 3.8cm]{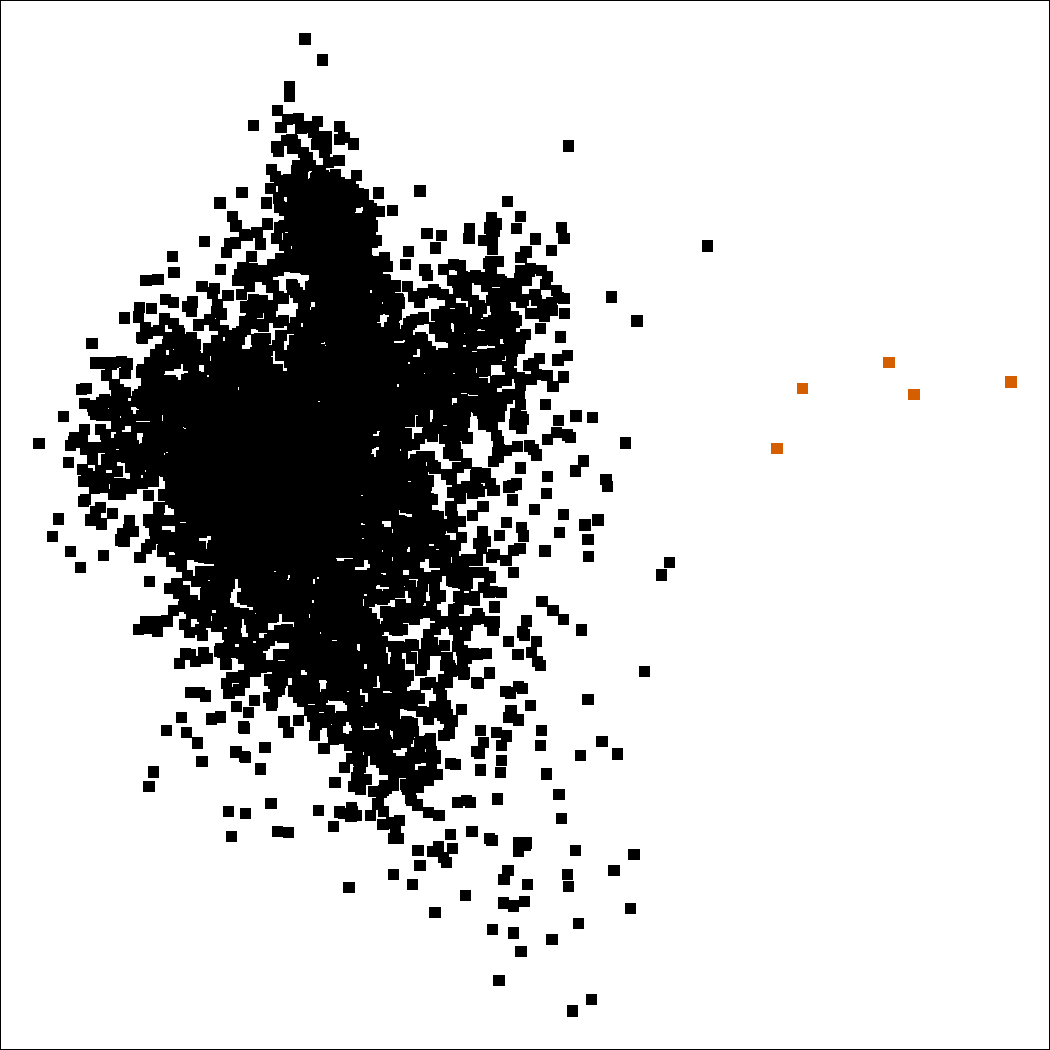} \label{fig:optidigitsBetamax}}
\subfigure[Optimal projection from minimising $\lambda_2(\LNtheta)$ with the modified metric ($\beta = 3$)]{\includegraphics[width = 3.8cm]{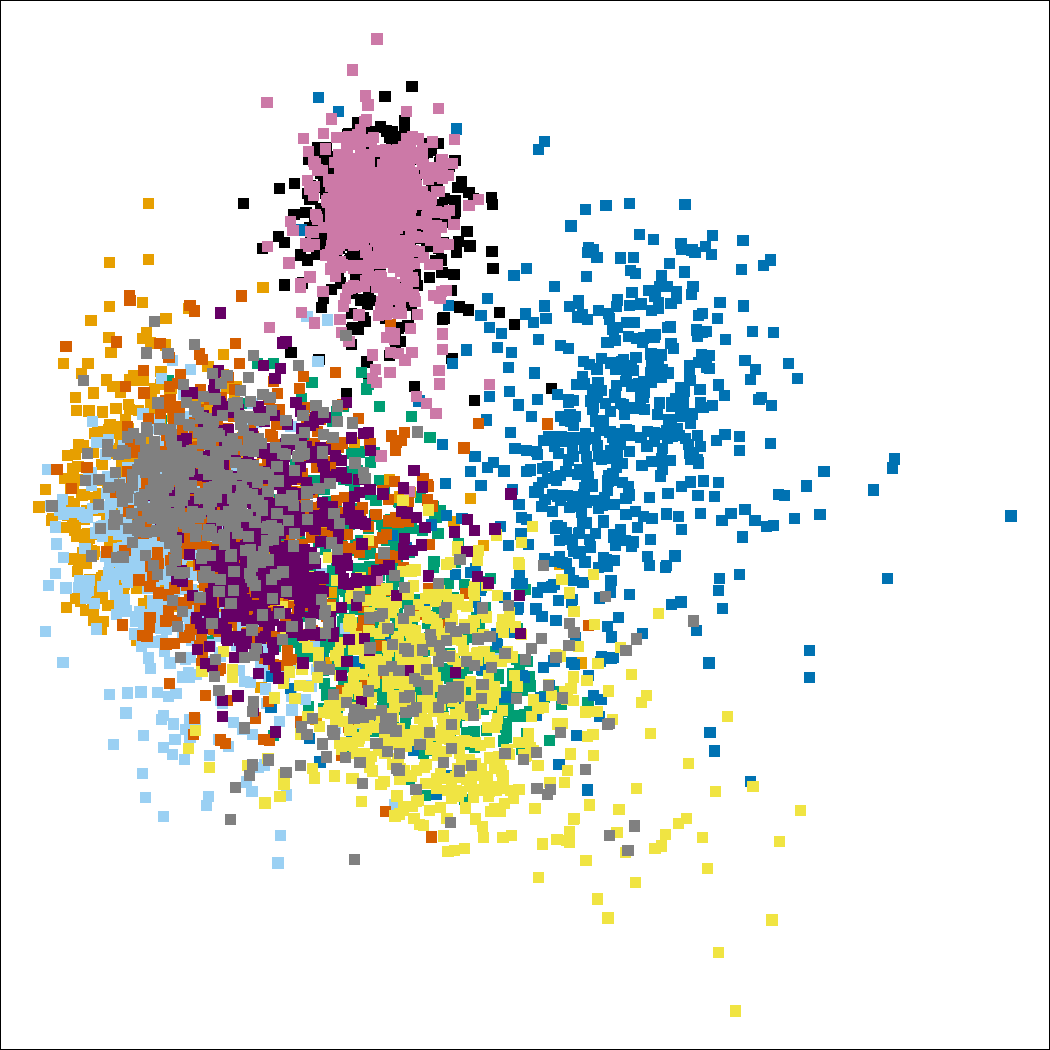}
\includegraphics[width = 3.8cm]{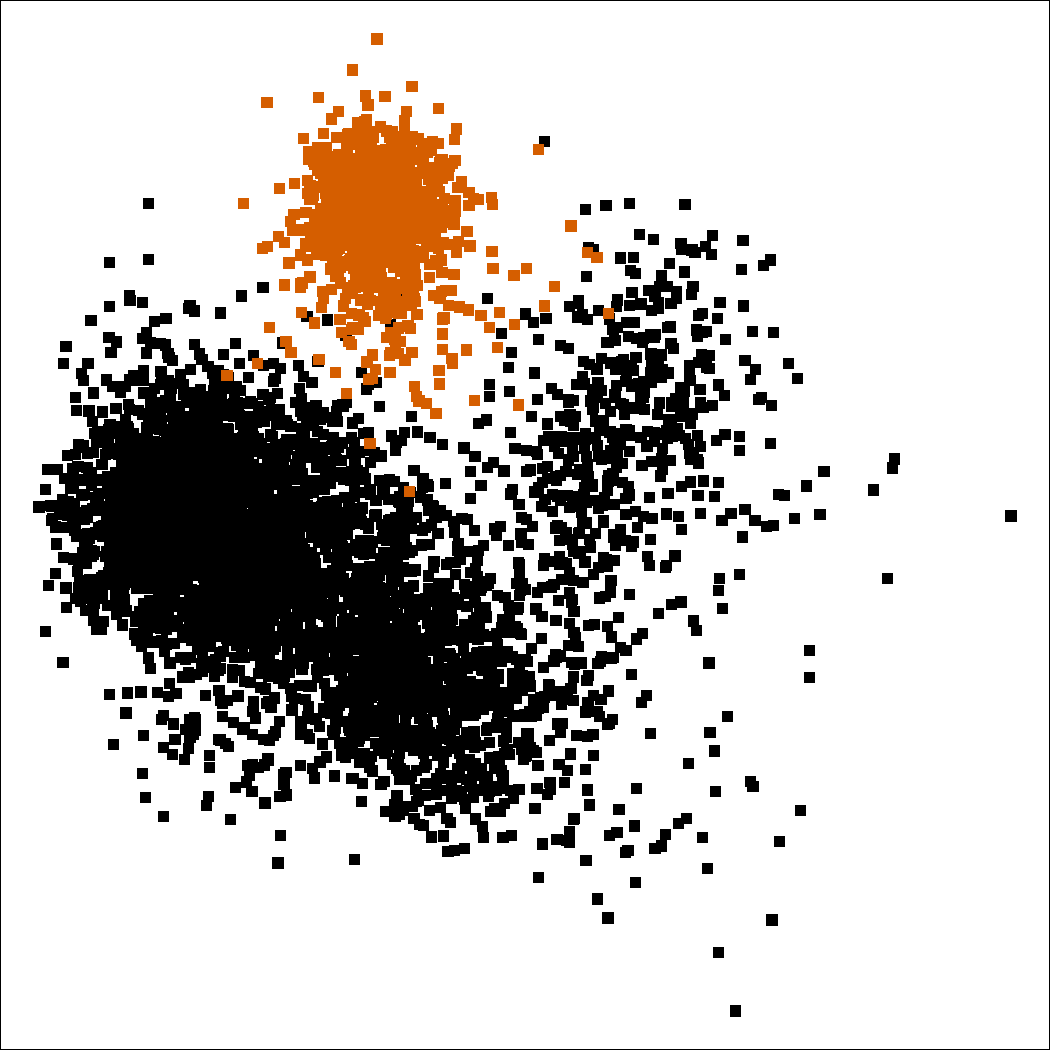} \label{fig:optidigitsBeta3}}
\end{figure}

\section{Derivatives}\label{sec:deriv}

\subsection{Evaluating $D_{P_i}\lambda_2(\cdot)$}

We first consider the standard Laplacian~$L$, and use $\lambda$ and $u$ to denote the second eigenvalue and corresponding
eigenvector.
By Eq.~(\ref{eigdev}) we have $d\lambda = u^\top d(L) u = u^\top d(D) u - u^\top d(A) u$. Now,
\begin{align*}
\frac{\partial D_{ii}}{\partial P_{mn}}  & = \sum_{j=1}^N \frac{\partial A_{ij}}{\partial P_{mn}} = \sum_{j=1}^N \frac{\partial s(P, i, j)}{\partial P_{mn}}, \\
	\frac{\partial A_{ij}}{\partial P_{mn}}  &= \frac{\partial s(P, i, j)}{\partial P_{mn}},
\end{align*}
and so,
\begin{equation*}\label{eq:derivstandard}
\frac{\partial{\lambda}}{\partial P_{mn}} = u^\top \frac{\partial L}{\partial P_{mn}} u 
= \frac{1}{2}\sum_{i, j}(u_i-u_j)^2\frac{\partial s(P, i, j)}{\partial P_{mn}}.
\end{equation*}

\noindent
For the normalised Laplacian, $L_{\mathrm{N}}$, consider first
\begin{align*}
d(L_{\mathrm{N}}) =& d(D^{-1/2}LD^{-1/2})\\
 =&  d(D^{-1/2})LD^{-1/2}+D^{-1/2}d(D)D^{-1/2} \\
&- D^{-1/2}d(A)D^{-1/2} + D^{-1/2}L d(D^{-1/2}).
\end{align*}
We again use $\lambda$ and $u$ to denote the second eigenvalue and corresponding eigenvector.
Using $LD^{-1/2}u = \lambda D^{1/2}u$,
\begin{align*}
d\lambda  = & u^\top d(D^{-1/2})LD^{-1/2}u + u^\top D^{-1/2}d(D)D^{-1/2}u\\ 
	&- u^\top D^{-1/2}d(A)D^{-1/2}u + u^\top D^{-1/2}L d(D^{-1/2})u\\
=& \lambda u^\top d(D^{-1/2})D^{1/2}u + u^\top D^{-1/2}d(D)D^{-1/2}u 
	\\
	&- u^\top D^{-1/2}d(A)D^{-1/2}u  + \lambda u^\top D^{1/2}d(D^{-1/2})u\\
=& (1-\lambda)u^\top D^{-1/2}d(D)D^{-1/2}u - u^\top D^{-1/2}d(A)D^{-1/2}u. \\
=& u^\top D^{-1/2}d(L)D^{-1/2}u  - \lambda u^\top D^{-1/2}d(D)D^{-1/2}u.
\end{align*}
Where in the third step we made use of the fact that\\
$d(D^{-1/2})DD^{-1/2} + D^{-1/2}d(D)D^{-1/2} + D^{-1/2}Dd(D^{-1/2}) = d(D^{-1/2}DD^{-1/2}) = d(I) = \mathbf{0}$.
Therefore,
\begin{align*} 
\frac{\partial \lambda}{\partial P_{mn}} = \frac{1}{2} \sum_{i, j} \left(\frac{u_i}{\sqrt{d_i}} - \frac{u_j}{\sqrt{d_j}}\right)^2 \frac{\partial s(P, i, j)}{\partial P_{mn}} 
- \lambda \sum_{i, j} \frac{u_i^2}{d_i}\frac{\partial s(P, i, j)}{\partial P_{mn}}.
\end{align*}

\subsection{Derivatives of the Approximate Eigenvalue Functions based on Microclusters}

In the general case we may consider a set of $m$ microclusters with centers
$c_1, \ldots, c_m$ and counts $n_1, \ldots, n_m$. The derivations we provide
are valid for $n_i = 1 \ \forall i \in \{1, \dots, m\}$, and so apply
to the exact formulation of the problem as well.  Let $\btheta \in \Theta$. We find it
practically convenient to associate the transformation in Eq.~(\ref{eq:Ttransform}),
which incorporates the set $\pmb{\Delta}(\btheta)$, with the projection of the microclusters rather than with the
computation of similarities. Specifically, we now let $\mathcal{T}$ be the transformed
projected microcluster centers, i.e.,
\begin{align*}
\mathcal{T} =& \{t_1, t_1, \dots, t_m, t_m \}\\
=& \{T_{\pmb{\Delta}(\btheta)}(V(\btheta)^\top c_1), T_{\pmb{\Delta}(\btheta)}(V(\btheta)^\top c_1),\\
& \hspace{5pt} \ldots, T_{\pmb{\Delta}(\btheta)}(V(\btheta)^\top c_m), T_{\pmb{\Delta}(\btheta)}(V(\btheta)^\top c_m) \},
\end{align*}
where each $t_i$ is repeated $n_i$ times. The reason for this is that with this formulation the majority
of terms in the above sums corresponding to $\partial \lambda$ (which are now partial derivatives w.r.t. the elements of $\mathcal{T}$, and not $\P$ as before) are zero. Specifically,
with this expression for $\mathcal{T}$, and letting $T$ be the matrix with columns corresponding to elements in $\mathcal{T}$, we have
\begin{align}\nonumber
\frac{\partial \lambda}{\partial T_{mn}} =& \frac{1}{2} \sum_{i, j}(u_i-u_j)^2\frac{\partial k(\|t_i - t_j\|/\sigma)}{\partial T_{mn}}\\
 =& \sum_{i \not = n } (u_i-u_n)^2\frac{\partial k(\|t_i - t_n\|/\sigma)}{\partial T_{mn}},
\end{align}
and similarly for the normalised Laplacian.

In
Section~\ref{sec:method} we expressed $D_{\btheta}\lambda$ via the chain rule
decomposition $D_P\lambda D_v PD_{\btheta} v$, which we can now simply restructure
as $D_T\lambda D_v TD_{\btheta} v$.
The compression of $\mathcal{T}$ to the
size $m$ non-repeated set, $\mathcal{T}^C = \{t_1, \ldots, t_m \}$, requires a
slight restructuring, as described in Section~\ref{sec:microclust}.
We begin with the standard Laplacian, letting $T^C$ be the matrix corresponding to $\mathcal{T}^C$, and define $N(\btheta)$ and $B(\btheta)$
as in Lemma~\ref{thm:approxbound1}. That is, $N(\btheta)$ is the diagonal
matrix with $i$-th diagonal element equal to $\sum_{j=1}^m n_j k(\|t_i - t_j\|/\sigma)$
and $B(\btheta)_{i,j} = \sqrt{n_i n_j} k(\|t_i - t_j\|/\sigma)$. The derivative of the
second eigenvalue of the Laplacian relies on the corresponding
eigenvector, $u$. However, this vector is not explicitly available as we only
solve the $m\times m$ eigen-problem of $N(\btheta) - B(\btheta)$. Let $u^C$ be
the second eigenvector of $N(\btheta) - B(\btheta)$. As in the proof of
Lemma~\ref{thm:approxbound1} if $i,j$ are such that the $i$-th element of $\mathcal{T}$
corresponds to the $j$-th microcluster, then $u^C_j = \sqrt{n_j}u_i$. The
derivative of $\lambda_2(N(\btheta)-B(\btheta))$ with respect to the $i$-th
column of $\btheta$, and thus equivalently of the second eigenvalue of the
Laplacian is therefore given by
\begin{align}\label{eq:deriv1}\nonumber
\Bigg(&\sum_{j \not = 1}\left(\frac{u^C_j}{\sqrt{n_j}}-\frac{u^C_1}{\sqrt{n_1}}\right)^2n_jn_1\frac{\partial k\left(\frac{\|t_j-t_i\|}{\sigma}\right)}{\partial T^C_{j1}} \ \dots\\
& \ \sum_{j \not = m}\left(\frac{u^C_j}{\sqrt{n_j}}-\frac{u^C_m}{\sqrt{n_m}}\right)^2n_jn_m\frac{\partial k\left(\frac{\|t_j-t_m\|}{\sigma}\right)}{\partial T^C_{jm}} \Bigg)
 D_{V_i} T^C_i D_{\btheta_i}V_i,
\end{align}
where $D_{\btheta_i}V_i$ is given in Eq.~(\ref{eq:difftheta}) and $D_{V_i} T^C_i$ is expressed below. We provide expressions for the case where $\Delta(\btheta) = \prod_{i=1}^l[- \beta\sigma_{\btheta_i}, \beta\sigma_{\btheta_i}]$, as in our implementation, where
we have again assumed that the data have been centered, i.e., have zero mean. 
Then $D_{V_i} T^C_i$ is the $m \times d$ matrix with $j$-th row equal to,
$$
\frac{\delta(1-\delta)}{(-\beta\sigma_{\btheta_i} - V_i^\top c_j + (\delta(1-\delta))^{1/\delta})^\delta} \left(\frac{\beta}{\sigma_{\btheta_i}}\Sigma V_i + c_j\right),
$$
if $V_i^\top c_j < -\beta \sigma_{\btheta_i}$,
$$
c_j,
$$
if $-\beta \sigma_{\btheta_i} \leq V_i^\top c_j \leq \beta \sigma_{\btheta_i}$, and
$$
\frac{\delta(1-\delta)}{(V_i^\top c_j - \beta\sigma_{\btheta_i} + (\delta(1-\delta))^{1/\delta})^\delta} \left(c_j - \frac{\beta}{\sigma_{\btheta_i}}\Sigma V_i\right) + 2\frac{\beta}{\sigma_{\btheta_i}} \Sigma V_i,$$
if $V_i^\top c_j>\beta\sigma_{\btheta_i}$.
Here $\Sigma$ is the covariance matrix of the data.

%Finally, for the function $k$ given in Eq.~(\ref{eq:kernel}) we have,
%\thinmuskip = .1mu \medmuskip = .1mu \thickmuskip = .1mu
%\begin{equation}
%\frac{\partial k\left(\frac{\| t_j - t_h \|}{\sigma}\right)}{\partial %T^C_{ih}} = \frac{T^C_{ij} - T^C_{ih}}{\sigma^2\alpha} \left(\frac{\| %t_j - t_h\|}{\sigma\alpha} + 1\right)^{\alpha-1}\exp\left(-\frac{\| t_j %- t_h\|}{\sigma}\right).
%\end{equation}
For the normalised Laplacian, the reduced $m\times m$ eigenproblem has precisely the same form as the original $N\times N$ problem, with the only difference being the introduction of the factors $n_j n_k$. Specifically, with the derivation in Section~\ref{sec:method} we can see that the corresponding derivative is as for the standard Laplacian above, except that the coefficients $(u_j^C/\sqrt{n_j} - u_k^C/\sqrt{n_k})^2n_j n_k$ in Eq.~(\ref{eq:deriv1}) are replaced with $(u_j^C/\sqrt{d_j} - u_k^C/\sqrt{d_k})^2 - \lambda((u_j^C)^2/d_j + (u_k^C)^2/d_k)$, where $\lambda$ is the second eigenvalue of the normalised Laplacian, $u^C$ is the corresponding eigenvector and $d_j$ is the degree of the $j$-th element of $\mathcal{T}^C$.

\section{Computational Complexity}\label{sec:complexity}

Here we give a very brief discussion of the computational complexity of the
proposed method. At each iteration in the gradient descent, computing the
projected data matrix, $P(\btheta)$, requires $\mathcal{O}(Nld)$ operations.
Computing all pairwise similarities from elements of the $l$-dimensional
$\P(\btheta)$ has computational complexity $\mathcal{O}(lN^2)$, and determining
both Laplacian matrices, and their associated eigenvalue/vector pairs adds a
further computational cost $\mathcal{O}(N^2)$. Each evaluation of the
objectives $\lambda_2(\Ltheta)$ or $\lambda_2(\LNtheta)$ therefore requires
$\mathcal{O}(lN(N+d))$ operations.  In order to compute the gradients of these
objectives, the partial derivatives with respect to each element of the
projected data matrix need to be calculated. As we discussed in relation
to the derivatives above, the majority of the terms in the sums in Eqs.~(\ref{eq:derivlam})
and~(\ref{eq:derivlamnorm}) are zero, and in fact each partial derivative can be computed in
$\mathcal{O}(N)$ time, and so all such partial derivatives can be computed in
$\mathcal{O}(lN^2)$ time. The matrix derivatives $D_{\btheta_i} V_i
,i=1,...,l$, in~(\ref{eq:difftheta}) can each be computed with $\mathcal{O}(d(d-1))$
operations. Finally, determining the gradients with respect to each column of
$\btheta$ involves computing the matrix product $D_{\btheta_i} \lambda= D_{P_i}
\lambda D_{V_i} P_i D_{\btheta_i} V_i$, where $D_{P_i} \lambda \in \R^{1\times
N}, D_{V_i} P_i \in \R^{N \times d}$ and $D_{\btheta_i} V_i \in \R^{d\times
(d-1)}$. This has complexity $\mathcal{O}(Nd(d-1))$. The complete gradient
calculation therefore requires $\mathcal{O}(lN(N+d(d-1)))$ operations. We have
found that the optimality conditions based on directional derivatives and gradient sampling steps are seldom, if ever required, and
moreover that these do not constitute the bottleneck in the running time of
the method in practice. The complexity of the optimality condition check may be computed along
similar lines, and be found to be $\mathcal{O}(t^2lN(N+d(d-1)))$, where $t$ is
the multiplicity of the eigenvalue $\lambda = \lambda_2(\Ltheta)$. The gradient sampling
is simply $\mathcal{O}(d)$ times the cost of computing a single gradient. The total
complexity of the projection pursuit optimisation depends on the number of
iterations in the gradient descent method, where in general this number is
bounded for a given accuracy level. For our experiments we use the BFGS
(Broyden-Fletcher-Goldfarb-Shanno) algorithm as this has been found to perform
well on non-smooth functions~\citep{LewisO2013}.

\section{Proofs} \label{sec:proofs}	

\subsection{Proof of Theorem~\ref{thm:convergence}} \label{sec:maxmargtheory}

Before proving Theorem~\ref{thm:convergence}, we require some supporting theory which we present below. 
We will use the notation $v^\top \mathcal{X} = \{v^\top x_1, ..., v^\top
x_N\}$, and for a set $\mathcal{P} \subset \R$ and $y \in \R$ we write, for example,
$\mathcal{P}_{>y}$ for $\mathcal{P} \cap (y, \infty)$. Recall that for scaling parameter $\sigma>0$ we define $\pmb{\theta}_{\sigma}: =
\mbox{argmin}_{\pmb{\theta} \in \Theta} \lambda_2(L(\pmb{\theta}, \sigma))$, where $L(\pmb{\theta}, \sigma)$ is as $L(\pmb{\theta})$ from
before, but with an explicit dependence on the scaling parameter. That is,
$\pmb{\theta}_{\sigma}$ defines the projection generating the minimal
spectral connectivity of $\X$ for a given value of~$\sigma$. We define $\btheta_{\sigma}^N$
similarly for the normalised Laplacian.

Recall that we are interested in those hyperplanes which intersect an arbitrary convex set $\pmb{\Delta}$. This is because very often the maximum marging hyperplane will separate only a few points from the remainder, as data tend to be more sparse in the tails of the underlying distribution. To account for the potential for hyperplanes with very large margins lying in the tails of the distribution, we make the additional assumption that the distance reducing parameter, $\delta$, tends to zero along with $\sigma$.

Lemmas~\ref{thm:maxdistbound} and~\ref{thm:maxdistboundnorm} provide lower bounds on the
second eigenvalue of the graph Laplacians of a one dimensional data set in terms
of the largest Euclidean separation of adjacent points which lie within the
interval $\Delta$, used to represent $\pmb{\Delta}(\btheta)$ in the context of a projection of $\X$. These lemmas also show how we construct the set
$\pmb{\Delta}^\prime$. Lemmas~\ref{thm:deltaMargin} and~\ref{thm:deltaMarginNorm} use these results to show
that a projection angle $\btheta \in \Theta$ leads to lower spectral
connectivity than all projections admitting smaller maximal margin hyperplanes
intersecting $\pmb{\Delta}^\prime$ for all pairs $\sigma, \delta$ sufficiently
close to zero.

\begin{lem} \label{thm:maxdistbound}
Let $k: \mathbb{R}_+ \to \mathbb{R}_+$ be a non-increasing, positive function and let $\sigma > 0, \delta \in (0, 0.5]$. Let $\P = \{p_1, ..., p_N\}$ be a univariate data set and let $\Delta =[a, b]$ for $a<b \in \R$. Suppose that $\vert \P \cap \Delta \vert \geq 2$ and $a\geq \min\{\P\}, b\leq\max\{\P\}$. Define $\Delta^\prime = [a^\prime, b^\prime]$, where $a^\prime = (a+\min\{\P\cap\Delta\})/2$, and $b^\prime = (b+\max\{\P\cap\Delta\})/2$. Let $M = \max_{x \in \Delta^\prime}\{\min_{i=1\dots N}\vert x-p_i \vert\}$.
Define $L(\P)$ to be the Laplacian of the graph with vertices $\P$ and similarities according to $s(P, i, j) = k(\vert T_{\Delta}(p_i) - T_{\Delta}(p_j)\vert/\sigma)$, where $P \in \R^{1 \times N}$ is the matrix with $i$-th column equal to $p_i$. Then $\lambda_2(L(\P)) \geq \frac{1}{\vert \P \vert ^3} k((2M+\delta C)/\sigma)$, where $C = \max\{D, D^{1-\delta}\},$ $D = \max\{a-\min \{\P\}, \max \{\P\} - b\}$.
\end{lem}

{\em {\bf Proof:}}
We can assume that $\P$ is sorted in increasing order, i.e. $p_i \leq p_{i+1}$, since this does not affect the eigenvalues of $L(\P)$. We first show that $s(P, i, i+1) \geq k((2M+\delta C)/\sigma)$ for all $i = 1, ..., N-1$. To this end observe that $\delta\left(x + \left(\delta\left(1-\delta\right)^{\frac{1}{\delta}}\right)\right)^{1-\delta}-\delta\left(\delta\left(1-\delta\right)\right)^{\frac{1-\delta}{\delta}} \leq \delta\max\{x, x^{1-\delta}\}$ for $x\geq 0$.
\begin{itemize}
\item If $p_i, p_{i+1} \leq a$ then $s(P, i, i+1) = k((T_{\Delta}(p_{i+1})- T_{\Delta}(p_i))/\sigma) \geq k((T_{\Delta}(a) - T_{\Delta}(p_i))/\sigma)$ $ \geq k((2M+\delta C)/\sigma)$ by the definition of $C$ and using the above inequality, since $k$ is non-increasing. The case $p_i, p_{i+1}\geq b$ is similar.
\item If $p_i, p_{i+1} \in \Delta$ then $p_i, p_{i+1} \in \Delta^\prime \Rightarrow \vert p_i - p_{i+1}\vert \leq 2M \Rightarrow s(P, i, i+1) \geq k(2M/\sigma) \geq k((2M+\delta C)/\sigma)$ since $M$ is the largest margin in $\Delta^\prime$.
\item If none the above hold, then we lose no generality in assuming $p_i < a$, $a<p_{i+1}<b$ since the case $a<p_i<b$, $p_{i+1}>b$ is analogous. We must have $p_{i+1} = \min\{\P\cap \Delta\}$ and so $a^\prime = (a+p_{i+1})/2$. If $p_{i+1}-a > 2M$ then $\min_{j=1 \dots N} \vert a^\prime - p_j \vert>M$, a contradiction since $a^\prime \in \Delta^\prime$ and $M$ is the largest margin in $\Delta^\prime$. Therefore $p_{i+1}-a \leq 2M$. In all 
\begin{align*}
T_{\Delta}(p_{i+1}) - T_{\Delta}(p_i) &= (p_{i+1}-a) + \delta(a-p_i+(\delta(1-\delta))^{\frac{1}{\delta}})^{1-\delta}\\
& \hspace{15pt}- \delta(\delta(1-\delta))^{\frac{1-\delta}{\delta}}\\
 &\leq 2M + \delta C\\
  \Rightarrow s(P, i, i+1)& \geq k((2M+\delta C)/\sigma).
\end{align*}
\end{itemize}
Now, let $u$ be the second eigenvector of $L(\P)$. Then $\|u\| = 1$ and $u\perp \mathbf{1}$ and therefore $\exists i, j$ s.t. $u_i - u_j \geq \frac{1}{\sqrt{N}}$. We thus know that there exists $m$ s.t. $\vert u_m - u_{m+1}\vert \geq \frac{1}{N^{3/2}}$. By~\cite[Proposition 1]{Luxburg2007}, we know that $u^\top L(\P)u = \frac{1}{2}\sum_{i, j}s(P, i, j)(u_i-u_j)^2 \geq s(P, m, m+1)(u_m-u_{m+1})^2 \geq  \frac{1}{N ^3} k((2M+\delta C)/\sigma)$ since all consecutive pairs $p_m,$ $p_{m+1}$ have similarity at least $k((2M+\delta C)/\sigma)$, by above. 
Therefore $\lambda_2(L(\P)) \geq \frac{1}{N^3}k((2M+\delta C)/\sigma)$ as required.
\hfill $\square$\\

\begin{lem} \label{thm:maxdistboundnorm}
Let the conditions of Lemma~\ref{thm:maxdistbound} hold and let $L_\mathrm{N}(\P)$ be the normalised Laplacian of the graph with vertices $\P$ and similarities $s(P, i, j) = k(\vert T_{\Delta}(p_i) - T_{\Delta}(p_j)\vert/\sigma)$. Then
$$
\lambda_2(L_{\mathrm{N}}(\P)) \geq \frac{1}{\vert \P \vert ^4} k((2M+\delta C)/\sigma).
$$
\end{lem}

\noindent {\em {\bf Proof:}}
The proof is similar to that of Lemma~\ref{thm:maxdistbound}, but requires a few simple modifications.
Let $u$ be the second eigenvector of $L_{\mathrm{N}}(\P)$. Since $\|u\| = 1, \exists i \in \{1, ..., N\}$ s.t. $\vert u_i \vert \geq \frac{1}{\sqrt{N}}$. Suppose w/o loss of generality that $u_i \leq -\frac{1}{\sqrt{N}}$. Now consider that for all $j, k \in \{1, ..., N\}$ we have $0 < s(P,j,k) \leq 1$ and $s(P,j,j) = 1$ and so $1 < \sqrt{d_j} \leq \sqrt{N}$ for all $j \in \{1, ..., N\}$. Therefore we have $u_i/\sqrt{d_i} \leq -\frac{1}{N}$. Furthermore, since $uD^{1/2} \perp \mathbf{1}$ we have $u_j > 0$ for some $j \in \{1, ..., N\} \Rightarrow u_j/\sqrt{d_j} > 0$. Therefore, $u_j/\sqrt{d_j} - u_i/\sqrt{d_i} > \frac{1}{N}$. We thus know that $\exists m \in \{1, ..., N\}$ s.t.
$
\left\vert u_m/\sqrt{d_m} - u_{m+1}/\sqrt{d_{m+1}}\right \vert > \frac{1}{N^2}.
$
By~\cite[Proposition 3]{Luxburg2007}, we know that 
\begin{align*}
u^\top L_{\mathrm{N}}(\P)u =& \frac{1}{2} \sum_{i \not = j} s(P,i,j)(u_i/\sqrt{d_i} - u_j/\sqrt{d_j})^2\\
 \geq& S(P,m,m+1)(u_m/\sqrt{d_m} - u_{m+1}/\sqrt{d_{m+1}})^2\\
  >& \frac{1}{N ^4} k((2M+\delta C)/\sigma),
\end{align*}
where the bound on $s(P, m, m+1)$ is taken from the proof of Lemma~5. Therefore $\lambda_2(L_{\mathrm{N}}(\P)) \geq \frac{1}{N^4}k((2M+\delta C)/\sigma)$ as required.
\hfill $\square$\\

In the above we have assumed that $\Delta$ is contained within the convex hull of the points $\P$, however the results of this section can easily be modified to allow for cases where this does not hold. In particular, if an unconstrained large margin hyperplane is sought, then setting $\pmb{\Delta}$ to be arbitrarily large allows for this. We have merely stated the results in the most convenient context for our practical implementation.

The set $\Delta^\prime$ in the above is defined in terms of the one dimensional interval $[a, b]$. We define the full dimensional set $\pmb{\Delta}^\prime$ along the same lines by,
\begin{align}
\nonumber
\pmb{\Delta}^\prime =& \{x \in \R^d \vert v(\btheta)^\top x \in \Delta(\btheta)^\prime \ \forall \btheta \in \Theta\},\\
\Delta(\btheta)^\prime :=& \Bigg[\frac{\min \Delta(\btheta) + \min\{v(\btheta)^\top \X \cap \Delta(\btheta)\}}{2},\\
& \hspace{15pt}\frac{\max \Delta(\btheta) + \max\{v(\btheta)^\top \X \cap\Delta(\btheta)\}}{2}\Bigg].
\end{align}
Here we assume that $\pmb{\Delta}$ is contained within the convex hull of the $d$-dimensional data set $X$.
Notice that since $\pmb{\Delta}$ is convex, we have $v(\btheta)^\top \pmb{\Delta}^\prime = \Delta(\btheta)^\prime$.
In what follows we show that as $\sigma$ is reduced to zero the optimal projection for spectral partitioning converges to the projection admitting the largest margin hyperplane intersecting $\pmb{\Delta}^\prime$. If it is the case that the largest margin hyperplane intersecting $\pmb{\Delta}$ also intersects $\pmb{\Delta}^\prime$, as is often the case, although this fact will not be known, then it is actually not necessary that $\delta$ tend towards zero. In such cases it only needs to satisfy $\delta \leq 2M/C$ for the corresponding values of $M$ and $C$ over all possible projections. In particular, choosing $\max\{\mbox{Diam}(\X), \mbox{Diam}(\X)^{1-\delta}\}$ instead of $C$ is appropriate for all projections.

\begin{lem}\label{thm:deltaMargin}
Let $\pmb{\theta} \in \Theta$ and let $k:\mathbb{R}_+ \to \mathbb{R}_+$ be non-increasing, positive, and satisfy
$$\lim_{x \to \infty} k(x(1+\epsilon))/k(x) = 0$$
for all $\epsilon > 0$. Then for any $0 < m < \max\limits_{b \in \Delta(\btheta)^\prime}\mbox{margin}(v(\pmb{\theta}), b)$ there exists $\sigma^\prime > 0$ s.t. if $0 < \sigma < \sigma^\prime$ and
$$\max\limits_{c \in \Delta(\btheta^\prime)^\prime}\mbox{margin}(v(\pmb{\theta}^\prime), c) < \max\limits_{b \in \Delta(\btheta)^\prime}\mbox{margin}(v(\pmb{\theta}), b) - m
$$ 
then $\lambda_2(L(\pmb{\theta}, \sigma)) < \lambda_2(L(\pmb{\theta}^\prime, \sigma))$.
\end{lem}

\noindent {\em {\bf Proof:}}
Let $B = \mbox{argmax}_{b \in \Delta(\btheta)^\prime}\mbox{margin}(v(\btheta), b)$ and $M = \mbox{margin}(v(\btheta), B)$. We assume that $M \not = 0$, since otherwise there is nothing to show. Now, since spectral clustering solves a relaxation of the minimum normalised cut problem we have,
\begin{align*}
\lambda_2(&L(\pmb{\theta}, \sigma)) \leq \frac{1}{\vert \X \vert}\min_{\C \subset \X} \sum_{\substack{i, j: x_i \in \C\\ x_j \not \in \C}} s(P(\btheta), i, j)\left(\frac{1}{\vert \C \vert} + \frac{1}{\vert \X \setminus \C \vert}\right)\\
&\leq  \frac{1}{\vert \X \vert}\sum_{\substack{i, j : v(\pmb{\theta})^\top x_i < B\\ v(\pmb{\theta})^\top x_j > B}} s(P(\btheta), i, j) \Bigg(\frac{1}{\vert (v(\pmb{\theta})^\top \X)_{< B}\vert}\\
& \hspace{120pt} + \frac{1}{\vert (v(\pmb{\theta})^\top \X)_{>B}\vert} \Bigg)\\
&=  \frac{1}{\vert \X \vert}\sum_{\substack{i, j : v(\pmb{\theta})^\top x_i < B\\ v(\pmb{\theta})^\top x_j > B}} k\left(\frac{T_{\Delta(\btheta)}(v(\btheta)^\top x_j) -  T_{\Delta(\btheta)}(v(\btheta)^\top x_i)}{\sigma}\right)\\
&\hspace{90pt}\times \left(\frac{\vert \X \vert}{\vert (v(\pmb{\theta})^\top \X)_{< B}\vert \vert (v(\pmb{\theta})^\top \X)_{>B}\vert}\right)\\
&\leq  \big\vert (v(\pmb{\theta})^\top \X)_{< B}\big\vert\big\vert (v(\pmb{\theta})^\top \X)_{> B}\big\vert k\left(\frac{2M}{\sigma}\right)\\
&\hspace{90pt} \times \left(\frac{1}{\vert (v(\pmb{\theta})^\top \X)_{< B}\vert \vert (v(\pmb{\theta})^\top \X)_{>B}\vert}\right)\\
&=  k(2M/\sigma).
\end{align*}
The final inequality holds since for any $i, j$ s.t. $v(\btheta)^\top x_i < B$ and $v(\btheta)^\top x_j >B$ we must have $T_{\Delta(\btheta)}(v(\btheta)^\top x_j) -  T_{\Delta(\btheta)}(v(\btheta)^\top x_i) \geq 2M$.
Now, for any $\pmb{\theta}^\prime \in \Theta$, let $M_{\pmb{\theta}^\prime} = \max_{c \in \Delta(\btheta^\prime)^\prime}\mbox{margin}(v(\pmb{\theta}^\prime), c)$. By Lemma~\ref{thm:maxdistbound} we know that $\lambda_2(L(\pmb{\theta}^\prime, \sigma)) \geq \frac{1}{\vert \X \vert^3} k((2M_{\pmb{\theta}^\prime}+\delta C/\sigma)$, where $C = \max\{\mbox{Diam}(X),$ $\mbox{Diam}(X)^{1-\delta}\}$. Therefore,
\begin{align*}
\lim_{\sigma \to 0^+}&\frac{\lambda_2(L(\pmb{\theta}, \sigma))}{\inf_{\pmb{\theta}^\prime \in \Theta}\{\lambda_2(L(\pmb{\theta}^\prime, \sigma)) \big \vert M_{\pmb{\theta}^\prime} < M - m\}}\\
& \leq  \lim_{\sigma \to 0^+}\frac{\vert \X \vert^3 k(2M/\sigma)}{ k((2(M-m)+\delta C)/\sigma)}\\
&=0.
\end{align*}
Since $\delta \to 0$ as $\sigma \to 0$, this gives the result.
\hfill $\square$\\

\begin{lem}\label{thm:deltaMarginNorm}
Let the conditions of Lemma~\ref{thm:deltaMargin} hold. For any $0 < m < \max_{b \in \Delta(\btheta)^\prime}\mbox{margin}(v(\pmb{\theta}), b)$ there exists $\sigma^\prime > 0$ s.t. if $0 < \sigma < \sigma^\prime$ and
$$
\max_{c \in \Delta(\btheta^\prime)^\prime}\mbox{margin}(v(\pmb{\theta}^\prime), c) < \max_{b \in \Delta(\btheta)^\prime}\mbox{margin}(v(\pmb{\theta}), b) - m
$$
then $\lambda_2(L_{\mathrm{N}}(\pmb{\theta}, \sigma)) < \lambda_2(L_{\mathrm{N}}(\pmb{\theta}^\prime, \sigma))$.
\end{lem}

\noindent {\em {\bf Proof:}}
Using a similar approach to that in the proof of Lemma~\ref{thm:deltaMargin}, we can arrive at the following.
\begin{align*}
\lambda_2&(L_{\mathrm{N}}(\btheta, \sigma))\leq \frac{\sum\limits_{\substack{i, j : v(\pmb{\theta})^\top x_i < B\\ v(\pmb{\theta})^\top x_j > B}} k\left(\frac{T_{\Delta(\btheta)}(v(\btheta)^\top x_j) -  T_{\Delta(\btheta)}(v(\btheta)^\top x_i)}{\sigma}\right)}{\mathrm{vol}((v(\pmb{\theta})^\top \X)_{< B}) \mathrm{vol}((v(\pmb{\theta})^\top \X)_{>B})}\\
&\leq k\left(\frac{2M}{\sigma}\right)\frac{\big\vert (v(\pmb{\theta})^\top \X)_{< B}\big\vert\big\vert (v(\pmb{\theta})^\top \X)_{> B}\big\vert}{\mathrm{vol}((v(\pmb{\theta})^\top \X)_{< B}) \mathrm{vol}((v(\pmb{\theta})^\top \X)_{>B})}\\
& \leq k(2M/\sigma)
\end{align*}
where the final inequality comes from the fact that $1 < d_i$ for all $i \in \{1, ..., N\}$, and hence vol$((v(\pmb{\theta})^\top \X)_{>B}) \geq \vert (v(\pmb{\theta})^\top \X)_{>B}\vert$, and similarly for $(v(\pmb{\theta})^\top \X)_{<B}$. The final step in the proof is equivalent to that of Lemma~\ref{thm:deltaMargin}, except that $\vert \X\vert ^3$ is replaced with $\vert \X\vert^4$.
\hfill $\square$\\

Lemmas~\ref{thm:deltaMargin} and~\ref{thm:deltaMarginNorm} show almost immediately that the margin admitted by the optimal projection for spectral bi-partitioning converges to the largest margin through $\pmb{\Delta}^\prime$ as $\sigma$ goes to zero. Theorem~\ref{thm:convergence}, which we are now in
a position to prove, shows the stronger result that the optimal projection itself converges to the projection admitting the largest margin.\\
\\
{\em Proof of Theorem~\ref{thm:convergence}:}
Take any $\epsilon > 0$. \citet{Pavlidis2015} have shown that $\exists m_\epsilon > 0$ s.t. for $w \in \R^d, c \in \R$, $\|(w, c)/\|w\| - (v(\pmb{\theta}^\star), b^\star) \| > \epsilon \Rightarrow $margin$(w/\|w\|, c/\|w\|) < $ margin$(v(\pmb{\theta}^\star), b^\star) - m_\epsilon$. By Lemma~\ref{thm:deltaMargin} we know $\exists \sigma^\prime > 0$, $\delta^\prime>0$ s.t. if $0 < \sigma < \sigma^\prime$ then $\exists c \in \Delta(\btheta)$ s.t. margin$(v(\pmb{\theta}_{\sigma}), c)$ $\geq $ margin$(v(\pmb{\theta}^\star), b^\star) - m_\epsilon$, since $\pmb{\theta}_{\sigma}$ is optimal for $\sigma$. Thus, by above, $\|(v(\pmb{\theta}_{\sigma}), c) - (v(\pmb{\theta}^\star), b^\star)\| \leq \epsilon$. But $\|(v(\pmb{\theta}_{\sigma}), c) - (v(\pmb{\theta}^\star), b^\star)\| \geq \|v(\pmb{\theta}_{\sigma}) - v(\pmb{\theta}^\star)\|$ for any $c \in \mathbb{R}$. Since $\epsilon > 0$ was arbitrary, we therefore have $v(\btheta_{\sigma}) \to v(\btheta^\star)$ as $\sigma \to 0^+$. The proof for $\btheta^N_{\sigma}$ is analogous. \hfill $\pmb{\square}$

\subsection{Proof of Lemma~\ref{thm:approxbound1}}

The proof of Lemma~\ref{thm:approxbound1} uses the following result from matrix
perturbation theory.
\begin{thm}[\citet{ye2009}]\label{thm:perturbbound}
Let $A = [a_{ij}]$ and $\tilde A = [\tilde a_{ij}]$ be two symmetric positive
semidefinite diagonally dominant matrices, and let $\lambda_1 \leq \lambda_2
\leq ... \leq \lambda_n$ and $\tilde \lambda_1 \leq \tilde \lambda_2 \leq ...
\leq \tilde \lambda_n$ be their respective eigenvalues. If, for some $0 \leq
\epsilon < 1$, $\vert a_{ij} - \tilde a_{ij} \vert \leq \epsilon \vert a_{ij}
\vert \ \forall i \not = j$, and $ \vert v_i - \tilde v_i \vert \leq \epsilon
v_i \ \forall i,$ where $v_i = a_{ii} - \sum_{j \not = i} \vert a_{ij} \vert$,
and similarly for $\tilde v_i$, then $$\vert \lambda_i - \tilde \lambda_i \vert
\leq \epsilon \lambda_i \ \forall i.$$
\end{thm}
\noindent
An inspection of the proof of Theorem~\ref{thm:perturbbound} reveals that
$\epsilon < 1$ is necessary only to ensure that the signs of $a_{ij}$ are the
same as those of $\tilde a_{ij}$. In the case of Laplacian matrices this
equivalence of signs holds by design, and so in this context the requirement
that $\epsilon < 1$ can be relaxed.\\

Now, for brevity we drop the notational dependence on $\btheta$. Let $\P^{c\prime} = \{V^\top c_1, V^\top c_1, ..., V^\top c_m, V^\top c_m\}$, where each $V^\top c_i$ is repeated $n_i$ times, and let $P^{c \prime}$ be the corresponding matrix of repeated projected centroids. Let $L^{c\prime}$ be the Laplacian of the graph with vertices $\P^{c\prime}$ and edges given by $s(P^{c\prime}, i, j)$. We begin by showing that $\lambda_2(L^{c\prime}) = \lambda_2(N-B)$. Take $v \in \R^m$, then,
\begin{align*}
v^\top(N-B)v &= \sum_{i,j}s(P^c, i, j)(v_i^2n_j - v_iv_j\sqrt{n_in_j})\\
&= \frac{1}{2}\sum_{i,j}s(P^c,i,j)(v_i^2n_j+v_j^2n_i-2v_iv_j\sqrt{n_in_j})\\
&\geq 0,
\end{align*}
and so $N-B$ is positive semi-definite. In addition, it is straightforward to verify that $(N-B)(\sqrt{n_1} \ \dots \ \sqrt{n_K}) = \mathbf{0}$, and hence $0$ is the smallest eigenvalue of $N-B$ with eigenvector $(\sqrt{n_1} \ \dots \ \sqrt{n_m})$. Now, let $u$ be the second eigenvector of $L^{c\prime}$. Then $u_j = u_k$ for pairs of indices $j,k$ aligned with the same $V^\top c_i$ in $P^{c\prime}$.  Define $u^c \in \R^m$ s.t. $u^c_i = \sqrt{n_i}u_j$ where index $j$ is aligned with $V^\top c_i$ in $P^{c\prime}_j$. Then $(u^c)^\top (\sqrt{n_1} \ \dots \ \sqrt{n_m}) = \sum_{i=1}^m u^c_i \sqrt{n_i} = \sum_{i=1}^m n_i u_{j_i}$ where index $j_i$ is aligned with $V^\top c_i$ in $P^{c\prime}_{j_i}$ for each $i$. Therefore $n_i u_{j_i} = \sum_{j:P^{c\prime} = V^\top c_i}u_j$ and hence $(u^c)^\top (\sqrt{n_1} \ \dots \ \sqrt{n_m}) = \sum_{i=1}^m\sum_{j: P^{c\prime}_j = V^\top c_i} u_j = \sum_{i=1}^N u_i = 0$ since $\mathbf{1}$ is the smallest eigenvector of $L^{c\prime}$ and so $u \perp \mathbf{1}$. Similarly $\|u^c\|^2 = \sum_{i=1}^m n_i u_{j_i}^2 = \sum_{i=1}^N u_i^2 = 1$. Thus $u^c \perp (\sqrt{n_1} \ \dots \ \sqrt{n_m})$ and $\|u^c\| = 1$ and so is a candidate for the second eigenvector of $N-B$. In addition it is straightforward to show that $(u^c)^\top (N-B)u^c = u\cdot L^{c\prime} u$. Now, suppose by way of contradiction that $\exists w \perp (\sqrt{n_1} \ \dots \ \sqrt{n_m})$ with $\|w\|=1$ s.t. $w^\top (N-B)w < (u^c)^\top (N-B)u^c$. Then let $w^\prime = (w_1/\sqrt{n_1} \ w_1/\sqrt{n_1} \ \dots \ w_m/\sqrt{n_m})$ where each $w_i/\sqrt{n_i}$ is repeated $n_i$ times. Then $\|w^\prime\| = 1$, $(w^\prime)^\top \mathbf{1} = w^\top (\sqrt{n_1} \ \dots \ \sqrt{n_m}) = 0$ and $w^\top L^{c\prime}w < u^\top L^{c\prime} u$, a contradiction since $u$ is the second eigenvector of $L^{c\prime}$.

Now, let $i,j,q,r$ be such that $x_q \in C_i$ and $x_r \in C_j$. We temporarily drop the notational dependence on $\Delta$. Then,
\begin{align*}
\| T (V^\top x_q) - T (V^\top x_r)\| =& \| T (V^\top x_q) - T (V^\top c_i)+T (V^\top c_i)\\
& \hspace{5pt} -T (V^\top c_j)
 +T (V^\top c_j)-T (V^\top x_r)\|\\
\leq& \| T (V^\top x_q) - T (V^\top c_i)\|\\
&+\| T (V^\top c_i)-T (V^\top c_j)\|\\
&+\| T (V^\top c_j)-T(V^\top x_r)\|\\
&\leq \rho_i + \rho_j + D_{ij},
\end{align*}
since $T$ contracts distances and $\rho_i$ and $\rho_j$ are the radii of $C_i$ and $C_j$. Since $k$ is non-increasing we therefore have,
\begin{align*}
&\frac{k(D_{ij}/\sigma)}{k((D_{ij}-\rho_i-\rho_j)^+/\sigma)}\leq\frac{k(D_{ij}/\sigma)}{k(\| T(V^\top x_q) - T(V^\top x_r)\|/\sigma)}\\
&\hspace{85pt}\leq\frac{k(D_{ij}/\sigma)}{k((D_{ij}+\rho_i+\rho_j)/\sigma)}\\
\Rightarrow& 1-\frac{k(D_{ij}/\sigma)}{k(\| T(V^\top x_q) - T(V^\top x_r)\|/\sigma)} \leq 1-\frac{k(D_{ij}/\sigma)}{k((D_{ij}-\rho_i-\rho_j)^+/\sigma)}\\
&\mbox{and}\\
& \frac{k(D_{ij}/\sigma)}{k(\| T(V^\top x_q) - T(V^\top x_r)\|/\sigma)}-1\leq \frac{k(D_{ij}/\sigma)}{k((D_{ij}+\rho_i+\rho_j)/\sigma)}-1.
\end{align*}
Therefore
\begin{align*}
&\left\vert \frac{ k(D_{ij}/\sigma)}{ k(\| T(V^\top x_q) - T(V^\top x_r) \|/\sigma)}-1\right\vert \leq\\
&\hspace{20pt} \max \left\{ 1- \frac{k(D_{ij}/\sigma)}{k((D_{ij}-\rho_i-\rho_j)^+/\sigma)},\frac{k(D_{ij}/\sigma)}{k((D_{ij}+\rho_i+\rho_j)/\sigma)} - 1\right\}.
\end{align*}
Now, we lose no generality by assume that $\X$ is ordered such that for each $i$ the elements of cluster $C_i$ are aligned with $V^\top c_i$ in $P^{c\prime}$, since this does not affect the eigenvalues of the Laplacian of $V^\top \X$, $L$. By the design of the Laplacian matrix the ``$v_i$" of Theorem~\ref{thm:perturbbound} are exactly zero. For off diagonal terms $q,r$ with corresponding $i,j$ as above, consider
\begin{align*}
\frac{\vert L_{qr} - L^{c\prime}_{qr} \vert}{\vert L_{qr}\vert}
 &= \frac{\vert k(D_{ij}/\sigma) - k(\| T(V^\top x_q) - T(V^\top x_r) \|/\sigma)\vert}{ k(\| T(V^\top x_q) - T(V^\top x_r) \|/\sigma)}\\
&= \left\vert \frac{ k(D_{ij}/\sigma)}{ k(\| T(V^\top x_q) - T(V^\top x_r) \|/\sigma)}-1\right\vert.
\end{align*}
Theorem~\ref{thm:perturbbound} thus gives the result. \hfill $\pmb{\square}$

\end{document}